\documentclass[twoside]{article}

\usepackage[accepted]{aistats2026}
\usepackage{eqnarray,amsmath,mathtools}
\usepackage[utf8]{inputenc} 
\usepackage[T1]{fontenc}    
\usepackage{booktabs}       
\usepackage{amsfonts}       
\usepackage{nicefrac}       
\usepackage{microtype}      

\usepackage{tikz}

\usepackage{caption}
\usepackage{subcaption}

\usepackage{epsf}
\usepackage{epsfig}
\usepackage{fancyhdr}
\usepackage{graphics}
\usepackage{graphicx}
\usepackage{psfrag}

\usepackage{url}
\usepackage[colorlinks,linkcolor=magenta,citecolor=blue, pagebackref=true]{hyperref}
\renewcommand*{\backrefalt}[4]{%
    \ifcase #1 \footnotesize{(Not cited.)}%
    \or        \footnotesize{(Cited on page~#2.)}%
    \else      \footnotesize{(Cited on pages~#2.)}%
    \fi}

\usepackage{color}

\usepackage{amsthm}
\usepackage{amsmath}
\usepackage{amssymb,bbm}
\usepackage{caption}
\usepackage{algorithm}
\usepackage{algorithmicx}
\usepackage{algpseudocode}
\usepackage{textcomp}
\usepackage{siunitx}
\usepackage{wrapfig}
\usepackage{multirow}
\usepackage{multicol}

\newcommand\blfootnote[1]{%
  \begingroup
  \renewcommand\thefootnote{}\footnote{#1}%
  \addtocounter{footnote}{-1}%
  \endgroup
}

\usepackage{comment}
\renewcommand{\arraystretch}{1.6}

\DeclareMathOperator*{\VD}{D_{V}}
\DeclareMathOperator*{\VDFRA}{D_{FRA}}
\DeclareMathOperator*{\VDO}{D_{O}}
\DeclareMathOperator*{\VDE}{D_{E}}

\newcommand{\divclus}{\mathsf{d}}
\newcommand{\height}{\mathsf{h}}
\DeclareMathOperator*{\TV}{D_{TV}}
\DeclareMathOperator*{\hels}{D_{h}^2}

\DeclareMathOperator*{\zero}{0} 

\let\inf\relax 
\DeclareMathOperator*\inf{\vphantom{p}inf}

\theoremstyle{plain}
\newtheorem{theorem}{Theorem}
\newtheorem{lemma}{Lemma}

\newtheorem{fact}{Fact}
\newtheorem{assumption}{Assumption}[section]
\theoremstyle{definition}

\usepackage{cleveref}
\crefname{lemma}{Lemma}{Lemmas}
\crefname{fact}{Fact}{Facts}
\crefname{section}{Section}{Sections}
\crefname{inequality}{Inequality}{Inequalities}
\crefname{appendix}{Appendix}{Appendices}
\crefname{table}{Table}{Tables}
\crefname{figure}{Figure}{Figures}
\crefname{algorithm}{Algorithm}{Algorithms}
\crefname{theorem}{Theorem}{Theorems}
\crefname{proposition}{Proposition}{Propositions}
\theoremstyle{plain}

\usepackage{amsmath,amsfonts,bm}

\usepackage{xcolor}
\definecolor{forestgreen}{rgb}{0.13, 0.55, 0.13}

\definecolor{frenchblue}{rgb}{0.0, 0.45, 0.73}

\definecolor{cherryblossompink}{rgb}{1.0, 0.72, 0.77}

\definecolor{bittersweet}{rgb}{1.0, 0.44, 0.37}

\definecolor{navyblue}{rgb}{0.0, 0.0, 0.5}

\def\eqref#1{equation~\ref{#1}}

\def\1{\bm{1}}

\def\ra{{\textnormal{a}}}

\def\rx{{\textnormal{x}}}

\def\rva{{\mathbf{a}}}

\def\rvx{{\mathbf{x}}}

\def\erva{{\textnormal{a}}}

\def\ervx{{\textnormal{x}}}

\def\vmu{{\bm{\mu}}}

\def\vtheta{{\bm{\theta}}}
\def\vTheta{{\bm{\Theta}}}
\def\veta{{\bm{\eta}}}

\def\valpha{{\bm{\alpha}}}
\def\vomega{{\bm{\omega}}}

\def\vell{{\bm{\ell}}}
\def\va{{\bm{a}}}
\def\vb{{\bm{b}}}

\def\ve{{\bm{e}}}

\def\vp{{\bm{p}}}
\def\vq{{\bm{q}}}

\def\vt{{\bm{t}}}

\def\vv{{\bm{v}}}

\def\vx{{\bm{x}}}

\def\eva{{a}}

\def\evv{{v}}

\def\mA{{\bm{A}}}

\def\mH{{\bm{H}}}
\def\mI{{\bm{I}}}
\def\mJ{{\bm{J}}}

\def\mX{{\bm{X}}}

\DeclareMathAlphabet{\mathsfit}{\encodingdefault}{\sfdefault}{m}{sl}
\SetMathAlphabet{\mathsfit}{bold}{\encodingdefault}{\sfdefault}{bx}{n}
\newcommand{\tens}[1]{\bm{\mathsfit{#1}}}
\def\tA{{\tens{A}}}

\def\tX{{\tens{X}}}

\def\gE{{\mathcal{E}}}

\def\gG{{\mathcal{G}}}
\def\gH{{\mathcal{H}}}

\def\gV{{\mathcal{V}}}

\def\sA{{\mathbb{A}}}
\def\sB{{\mathbb{B}}}

\def\sI{{\mathbb{I}}}

\def\sP{{\mathbb{P}}}

\def\sS{{\mathbb{S}}}

\def\emA{{A}}

\newcommand{\etens}[1]{\mathsfit{#1}}

\def\etA{{\etens{A}}}

\newcommand{\E}{\mathbb{E}}

\newcommand{\R}{\mathbb{R}}

\newcommand{\KL}{D_{\mathrm{KL}}}
\newcommand{\Var}{\mathrm{Var}}

\newcommand{\Cov}{\mathrm{Cov}}

\newcommand{\normltwo}{L^2}
\newcommand{\normlp}{L^p}

\newcommand{\parents}{Pa} 

\newcommand{\bbE}{\mathbb{E}}

\def\st{s.t.~}

\newcommand{\Ns}{\mathbb{N}} 

\newcommand{\cD}{\mathcal{D}}

\newcommand{\cE}{\mathcal{E}}

\newcommand{\cL}{\mathcal{L}}

\newcommand{\cN}{\mathcal{N}}
\newcommand{\cO}{\mathcal{O}}

\newcommand{\cT}{\mathcal{T}}

\newcommand{\cW}{\mathcal{W}}

\newcommand{\cX}{\mathcal{X}}

\newcommand{\norm}[1]{\|#1\|}

\DeclareMathOperator*{\argmax}{arg\,max}
\DeclareMathOperator*{\argmin}{arg\,min}
\usepackage[round,sort]{natbib}

\begin{document}

%
\runningtitle{SGMoE Dendrograms: Consistency Without Sweeps}

%

\runningauthor{Do, H., Mai, N., Nguyen, T., Ho, N., Nguyen, B., \& Drovandi, C.}

\twocolumn[


\aistatstitle{Dendrograms of Mixing Measures for Softmax-Gated Gaussian Mixture of Experts: Consistency Without Model Sweeps}

\aistatsauthor{TienHai Do$^{\star,1,2}$ \And Trung Nguyen Mai$^{\star,2,3}$ \And  TrungTin Nguyen$^{\star,\dagger,4,5}$}

\aistatsauthor{Nhat Ho$^{6}$ \And Binh T. Nguyen$^{1,2}$ \And Christopher Drovandi$^{4,5}$}

\aistatsaddress{$^{1}$Faculty of Mathematics and Computer Science, University of Science, Ho Chi Minh City, Vietnam.\\
$^{2}$Vietnam National University Ho Chi Minh City, Vietnam.\\
$^{3}$Faculty of Information Technology, University of Science, Ho Chi Minh City, Vietnam.\\
$^{4}$ARC Centre of Excellence for the Mathematical Analysis of Cellular Systems.\\
$^{5}$School of Mathematical Sciences, Queensland University of Technology, Brisbane City, Australia.\\
$^{6}$Department of Statistics and Data Science, University of Texas at Austin, Austin, USA.
} ]

\begin{abstract}
 We develop a unified statistical framework for softmax-gated Gaussian mixture of experts (SGMoE) that addresses three long-standing obstacles in parameter estimation and model selection: (i) non-identifiability of gating parameters up to common translations, (ii) intrinsic gate-expert interactions that induce coupled differential relations in the likelihood, and (iii) the tight numerator-denominator coupling in the softmax-induced conditional density. Our approach introduces Voronoi-type loss functions aligned with the gate-partition geometry and establishes finite-sample convergence rates for the maximum likelihood estimator (MLE). In over-specified models, we reveal a link between the MLE's convergence rate and the solvability of an associated system of polynomial equations characterizing near-nonidentifiable directions. For model selection, we adapt dendrograms of mixing measures to SGMoE, yielding a consistent, sweep-free selector of the number of experts that attains pointwise-optimal parameter rates under overfitting while avoiding multi-size training. Simulations on synthetic data corroborate the theory, accurately recovering the expert count and achieving the predicted rates for parameter estimation while closely approximating the regression function. Under model misspecification (e.g., $\epsilon$-contamination), the dendrogram selection criterion is robust, recovering the true number of mixture components, while the Akaike information criterion, the Bayesian information criterion, and the integrated completed likelihood tend to overselect as sample size grows. On a maize proteomics dataset of drought-responsive traits, our dendrogram-guided SGMoE selects two experts, exposes a clear mixing-measure hierarchy, stabilizes the likelihood early, and yields interpretable genotype-phenotype maps, outperforming standard criteria without multi-size training.
    \blfootnote{$^\star$Co-first author, $^{\dagger}$Corresponding author.}

\end{abstract}



\section{INTRODUCTION}
\label{sec_introduction}

{\bf Mixture of Experts: Scope and Appeal.}
Mixture of experts (MoE) were introduced as modular neural architectures in \cite{jacobs_adaptive_1991,jordan_hierarchical_1994}, where a gating network dispatches inputs to specialized experts. Beyond their practical versatility in speech, language, and vision \citep{Shazeer_JMLR,pham_competesmoeeffective_2024,do_hyperrouter_2023,Eigen_learning_2014,bao_vlmo_2022,dosovitskiy_image_2021,liang_m3vit_2022,You_Speech_MoE,You_Speech_MoE_2,peng_bayesian_1996}, MoE admit strong approximation guarantees and learning theory. Universal approximation results for conditional densities and regressors quantify how MoE improve upon unconditional mixtures by allowing both gates and experts to depend on covariates \citep{norets_approximation_2010,nguyen_universal_2016,nguyen_approximation_2019_MoE,nguyen_approximations_2021}. These developments complement classical approximation and risk bounds for unconditional mixtures \citep{genovese_rates_2000,rakhlin_risk_2005,nguyen_approximation_2025,chong_risk_2024,nguyen_convergence_2013,shen_adaptive_2013,ho_convergence_2016,ho_strong_2016,nguyen_approximation_2020,nguyen_approximation_2023} and are surveyed in \citet{yuksel_twenty_2012,nguyen_practical_2018,nguyen_model_2021,chen_towards_2022}.

{\bf Parameter Estimation: from Unconditional Mixtures to MoE.}
Over-specified finite mixtures can display slow, nonstandard parameter rates. In unconditional mixtures this is explained by singular Fisher information and merging components. Foundational results start with \citet{chen_optimal_1995} for univariate mixtures, and extend via Wasserstein tools to multivariate models and weaker identifiability \citep{nguyen_convergence_2013,ho_convergence_2016}, with minimax studies in \citet{heinrich2018,Manole_2020}. Algorithmic guarantees for Expectation-Maximization (EM) and Majorization-Minimization or Minimization-Maximization (MM) algorithms and moments have been analyzed under both exact-fit and over-fit regimes \citep{Siva_2017,Anandkumar_moment_method,Hardt_mixture,Raaz_Ho_Koulik_2020,Raaz_Ho_Koulik_2018_second,wu2020a,doss_optimal_2023,Wu_minimax_EM,tran_revisiting_2026}. For MoE with covariate-free gates, parameter rates depend on algebraic independence of experts and PDE-type couplings \citep{ho_convergence_2022,do_strong_2025}. In softmax-gated Gaussian mixture of experts (SGMoE), parameter estimation is harder due to translation invariance in softmax gates and intrinsic gate-expert couplings; recent progress includes identifiability, inverse bounds, and finite-sample guarantees for the maximum likelihood
estimator (MLE) with unified exact- and over-fit treatments in~\citet{nguyen_demystifying_2023,nguyen_general_2024,nguyen_towards_2024}.

{\bf Model Selection: Information Criteria, Penalties, and Bayes.}
Choosing the number of experts remains critical despite universal approximation theorems. Classical criteria balance fit and complexity, including AIC \citep{akaike_new_1974,fruhwirth_schnatter_analysing_2018}, BIC and its MoE adaptations \citep{schwarz_estimating_1978,khalili_estimation_2024,forbes_mixture_2022,berrettini_identifying_2024,forbes_summary_2022,nguyen_modifications_2025,ho_unified_2025}, ICL \citep{biernacki_assessing_2000,fruhwirth_schnatter_labor_2012}, eBIC for structured settings \citep{foygel_extended_2010,nguyen_joint_2024}, and SWIC for dependent data \citep{sin_information_1996,nguyen_large_sample_2025,westerhout_asymptotic_2024}. These methods are largely asymptotic and often require multi-size model sweeps. Non-asymptotic penalization brings risk guarantees via weak oracle bounds in high-dimensional MoE~\citep{nguyen_non_asymptotic_2021,nguyen_model_2022,nguyen_non_asymptotic_2022,nguyen_non_asymptotic_2023,montuelle_mixture_2014,nguyen_non_asymptotic_Lasso_2023}. Bayesian strategies avoid fixing the order but need careful marginal-likelihood evaluation or post-processing; the merge-truncate-merge approach ensures consistency in related mixture settings yet introduces sensitive tuning \citep{fruhwirth_schnatter_keeping_2019,zens_bayesian_2019,guha_posterior_2021,nguyen_bayesian_2024_JNPS}. A recent alternative leverages dendrograms of mixing measures for selection without exhaustive sweeps in~\citep{do_dendrogram_2024,thai_model_2025,tran_fast_2026}.

{\bf Gaps Specific to SGMoE.}
Softmax gating creates three intertwined obstacles. First, gate parameters are identifiable only up to common translations, so parameter losses must factor out these symmetries. Second, the softmax numerator-denominator coupling and the expert structure induce exact PDE relations between derivatives, which collapse naive Taylor decompositions and require algebra-aware inverse bounds. Third, when models are over-specified, the first nonvanishing terms in the expansions are ruled by solvability of polynomial systems; the resulting exponents govern slow parameter rates and depend on how many fitted atoms approximate each truth \citep{ho_convergence_2022,nguyen_demystifying_2023}. Existing selection criteria do not exploit this rate geometry for the MLE, and sweep-based procedures are computationally heavy for SGMoE.

{\bf Contributions.}
We introduce a fast-rate-aware Voronoi distance for SGMoE that augments the unified exact- and over-fit loss with merged-moment couplings inside multi-covered Voronoi cells (\cref{eq_def_new_voronoi_D}). This exposes slow directions created by redundant atoms, motivates a hierarchical merge operator, and yields an aggregation path (dendrogram) on mixing measures. Along this path we prove a monotone strengthening of the loss (\cref{lem:monotone_path}), obtain near-parametric finite-sample rates for the aggregated estimators together with height and likelihood control (\cref{thm:path_rates,thm:heights,thm:likelihood_path} and \cref{table_parameter_rates}), and derive a sweep-free dendrogram selection criterion (DSC) that is consistent and avoids multi-\(K\) training (\cref{thm_order_consistency,fig_merging_procedure,fig_DSC_well_specified}). Empirically, DSC is less prone to overfitting than AIC/BIC/ICL under \(\epsilon\)-contamination due to its structural penalty on small heights (\cref{fig_econtam_results}), and it restores fast parameter rates after aggregation in over-specified SGMoE (\cref{fig_voronoi_three}). To our knowledge this is the first method that couples finite-sample, fast-rate-aware merging with consistent model selection for SGMoE, avoiding multi-size training while preserving statistical efficiency.
\begin{table*}[!ht]
\caption{Summary of density and parameter rates for SGMoE. The Voronoi cells $\sA_j$ are defined in \cref{eq_voronoi_cell}. The function $\bar r(\cdot)$ is determined by solvability of the polynomial systems recalled in \cref{eq_system_of_polynomial_recall} (e.g., $\bar r(2)=4$, $\bar r(3)=6$). The merged row and the fast pathwise rates correspond to the aggregation path described in \cref{sec_rate_aware_sgmoe}.}
\centering
\begin{tabular}{|c|c|c|c|c|c|}
\hline
\textbf{Setting} & \textbf{Loss} & $p_{G_0}(y\mid \vx)$ & $\exp(\omega_{0k}^{0})$ & $\vomega_{1k}^{0},\, b_{k}^{0}$ & $\va_{k}^{0},\, \sigma_{k}^{0}$ \\
\hline
Exact-fit & $\VDE$ & $\cO\!\big((\log N/N)^{1/2}\big)$ & $\cO\!\big((\log N/N)^{1/2}\big)$ & $\cO\!\big((\log N/N)^{1/2}\big)$ & $\cO\!\big((\log N/N)^{1/2}\big)$ \\
\hline
Over-fit& $\VDO$ & $\cO\!\big((\log N/N)^{1/2}\big)$ & $\cO\!\big((\log N/N)^{1/2}\big)$ & $\cO\!\big((\log N/N)^{1/{2\bar r(|\sA_k|)}}\big)$ & $\cO\!\big((\log N/N)^{1/{\bar r(|\sA_k|)}}\big)$ \\
\hline
{\bf Merged} & $\VDFRA$ & $\cO\!\big((\log N/N)^{1/2}\big)$ & $\cO\!\big((\log N/N)^{1/2}\big)$ & $\cO\!\big((\log N/N)^{1/2}\big)$ & $\cO\!\big((\log N/N)^{1/2}\big)$ \\
\hline
\end{tabular}
\label{table_parameter_rates}
\end{table*}

{\bf SGMoE Setting.}
Let $(\rvx_{n}, y_{n})_{n=1}^{N}$ be i.i.d.\ samples with $\rvx_{n}\in\mathbb{R}^{D}$ and $y_{n}\in\mathbb{R}$.
Assume the data are generated by a SGMoE model of order $K_{0}$, whose conditional density is
\begin{align}
p_{G_{0}}(y\mid \vx)
  &:=
  \sum_{k=1}^{K_{0}}
  \frac{\exp\!\big((\vomega_{1k}^{0})^{\top}\vx + \omega_{0k}^{0}\big)}
     {\sum_{j=1}^{K_{0}}
      \exp\!\big((\vomega_{1j}^{0})^{\top}\vx + \omega_{0j}^{0}\big)}\nonumber\\
  &\hspace{1cm} \times \cN\!\big(y \,\big|\, \va_{k}^{0\top}\vx + b_{k}^{0},\, \sigma_{k}^{0}\big).
\label{eq_mixture_expert_1}
\end{align}

Each expert is Gaussian with mean $\va_{k}^{0\top}\vx+b_{k}^{0}$ and variance $\sigma_{k}^{0}>0$. We encode parameters via the (not-necessarily normalized) mixing measure
\begin{align*}
    G_{0} \equiv G_{0}(K_0)
\;:=\;
\sum_{k=1}^{K_{0}}
\exp(\omega_{0k}^{0})\,
\delta_{(\vomega_{1k}^{0},\, \va_{k}^{0},\, b_{k}^{0},\, \sigma_{k}^{0})},
\end{align*}
where $\veta_{k}^{0}:=(\omega_{0k}^{0}, \vomega_{1k}^{0}, \va_{k}^{0}, b_{k}^{0}, \sigma_{k}^{0})
\in \vTheta \subset \mathbb{R} \times \mathbb{R}^{D} \times \mathbb{R}^{D} \times \mathbb{R} \times \mathbb{R}_{>0}$.
Assume $\vTheta$ is compact and $\cX\subset\mathbb{R}^{D}$, the support of $\rvx$, is bounded. 
Assume $\rvx$ has a continuous distribution so that the model is identifiable under this convention, a standard mild assumption; see Proposition~1 of \cite{nguyen_demystifying_2023}.

\begin{figure}[ht]
    \centering
    \begin{subfigure}{0.32\linewidth}
        \centering
        \includegraphics[width=\linewidth]{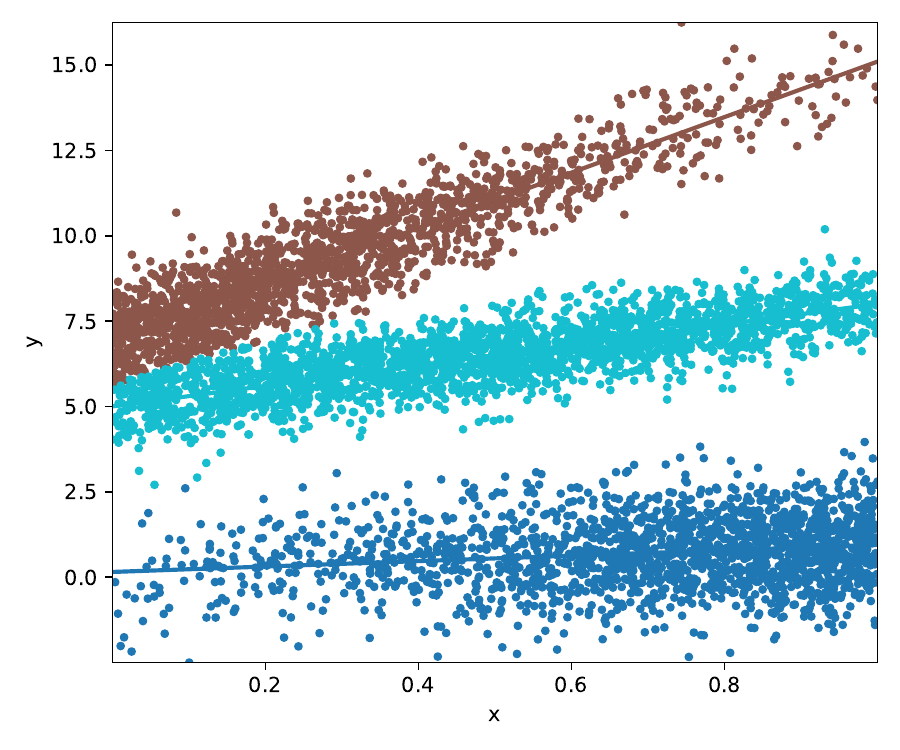}
        \caption{True regression}
        \label{fig_true_regression}
    \end{subfigure}
    \begin{subfigure}{0.32\linewidth}
        \centering
        \includegraphics[width=\linewidth]{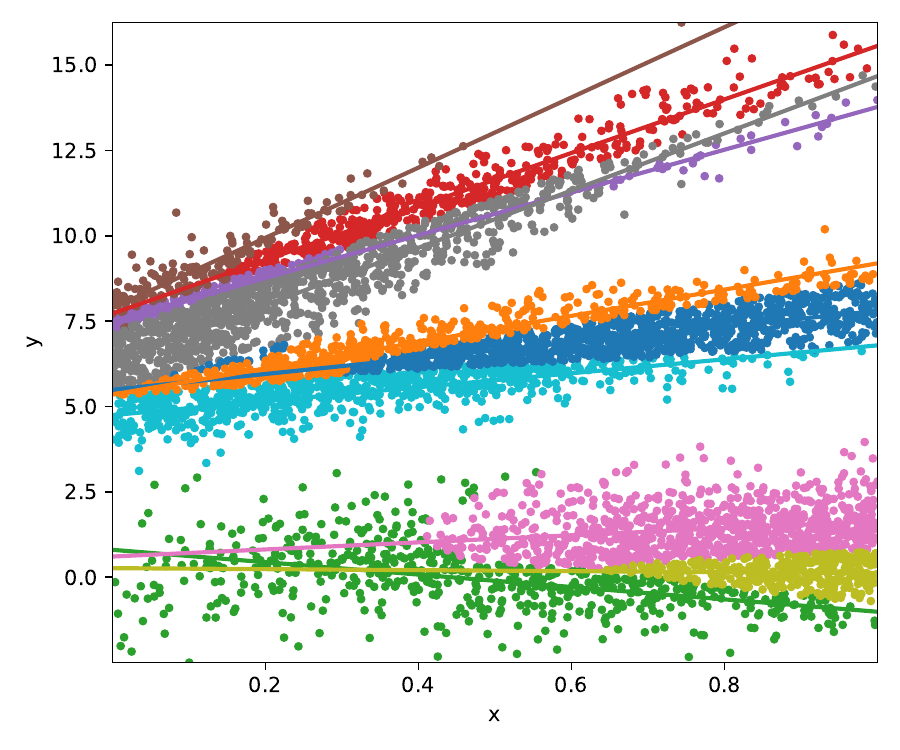}
        \caption{$K=10$}
        \label{fig_K10}
    \end{subfigure}
    \begin{subfigure}{0.32\linewidth}
        \centering
        \includegraphics[width=\linewidth]{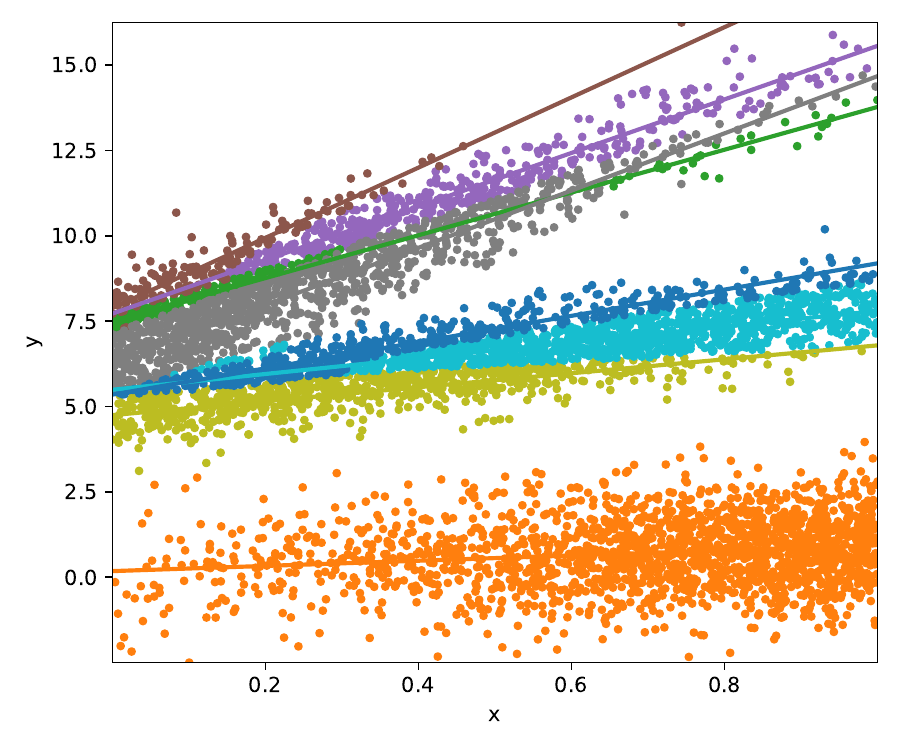}
        \caption{$K=8$}
        \label{fig_K8}
    \end{subfigure}

    \begin{subfigure}{0.32\linewidth}
        \centering
        \includegraphics[width=\linewidth]{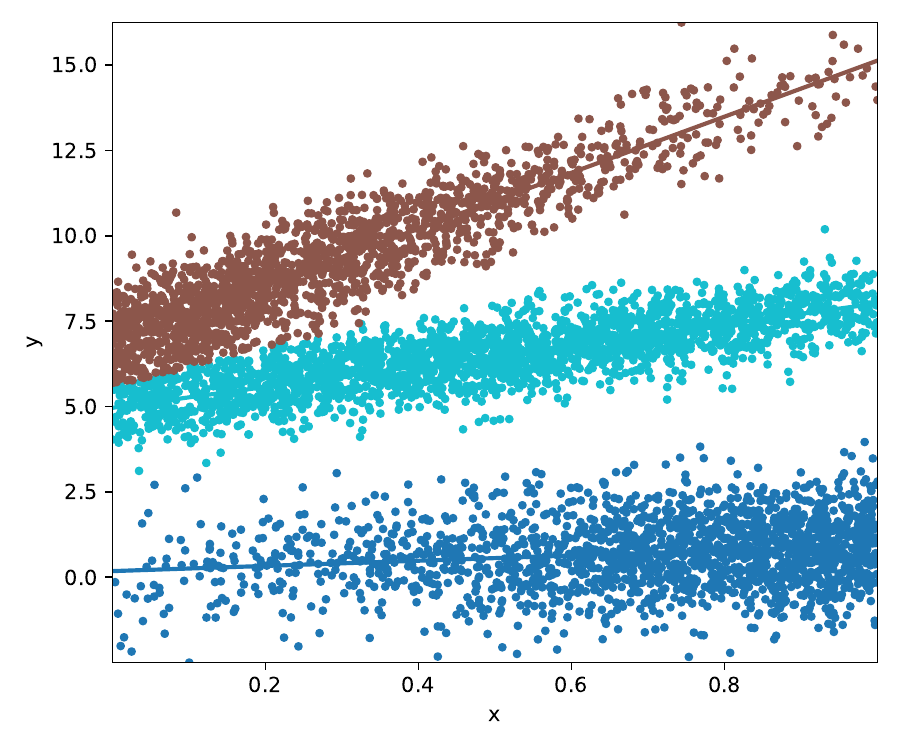}
        \caption{$K=3$}
        \label{fig_K3}
    \end{subfigure}
    \begin{subfigure}{0.32\linewidth}
        \centering
        \includegraphics[width=\linewidth]{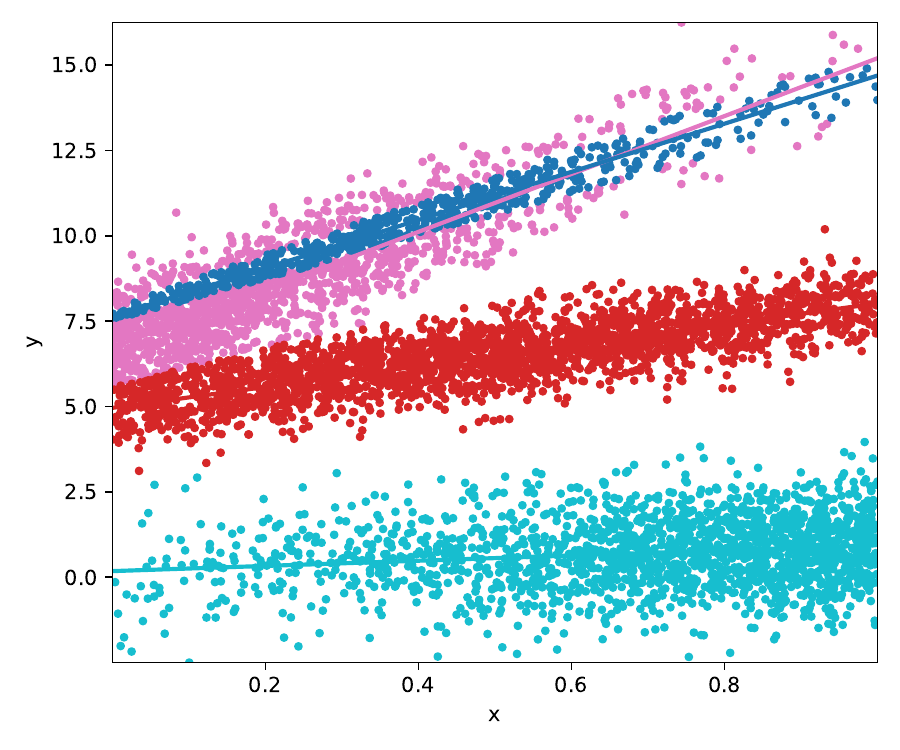}
        \caption{$K=4$}
        \label{fig_K4}
    \end{subfigure}
        \begin{subfigure}{0.32\linewidth}
        \centering
        \includegraphics[width=\linewidth]{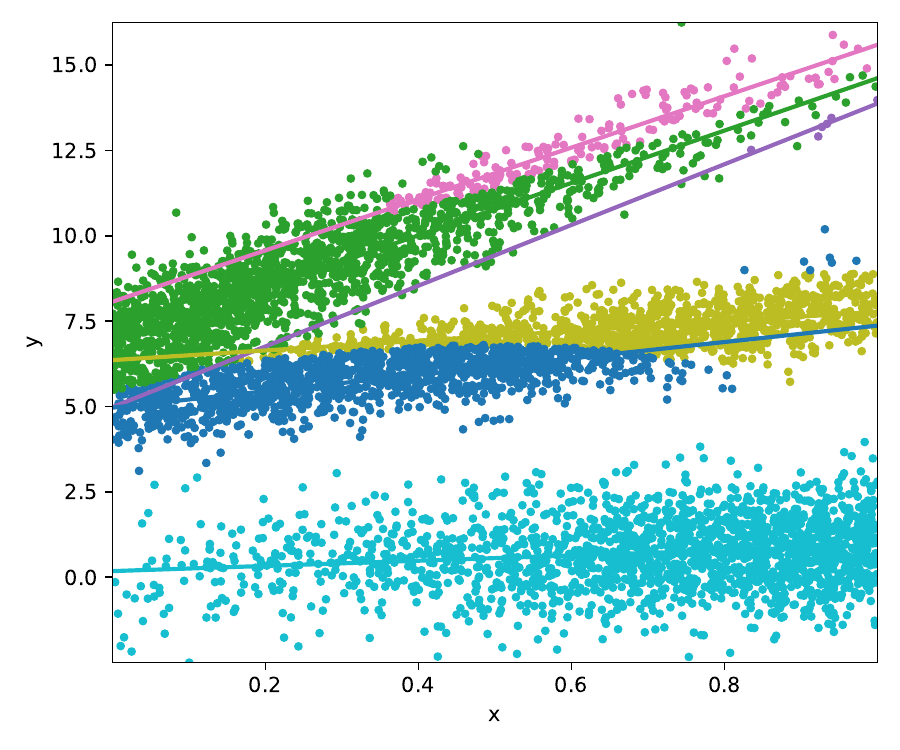}
        \caption{$K=6$}
        \label{fig_K6}
    \end{subfigure}

    \caption{Merging procedure from $K = 10$ to $K = 3$ of true mixing measure $G_0(3)$ with $K_0 = 3$ components, defined in \cref{eq_G_0_3}.}
    \label{fig_merging_procedure}
\end{figure}

\noindent\textbf{Maximum Likelihood Over At Most $K$ Experts.}
When the true order $K_{0}$ is unknown, we estimate within
$\cO_{K}(\vTheta)
:=\Big\{
G=\!\!\sum_{k=1}^{K'}\!\exp(\omega_{0k})\,\delta_{(\vomega_{1k},\,\va_{k},\,b_{k},\,\sigma_{k})}:$ $1\le K'\le K,\ (\omega_{0k},\vomega_{1k},\va_{k},b_{k},\sigma_{k})\in\vTheta
\Big\}.$
We analyze the exactly specified case $K = K_{0}$, the over-specified case $K > K_{0}$, and the merging scheme using the following maximum likelihood estimator (MLE): $\widehat{G}_{N}
\;\in\;
\argmax_{G \in \cO_{K}(\vTheta)}
\frac{1}{N}\sum_{n=1}^{N}\log\!\big(p_{G}(y_{n}\mid \rvx_{n})\big).$

{\bf Practical Implication.}
Practitioners can fit a single over-specified SGMoE with moderate $K\ge K_0$, compute its aggregation path, and select $\widehat K$ via DSC. This single-fit workflow avoids grid sweeps over $K$, merges near-duplicate atoms to collapse slow directions within Voronoi cells, accelerates parameter convergence, and often recovers the correct expert count even under mild contamination. Dendrogram heights provide a transparent structural summary.

{\bf Paper Organization.}
\cref{sec_preliminaries} states the unified parameter-rate result and the algebraic exponents \(\bar r(\cdot)\).
\cref{sec_rate_aware_sgmoe} introduces the fast-rate-aware distance, merge operator, aggregation path, fast pathwise rates, and DSC.
\cref{sec_simulation} illustrates parameter rates, path behaviour, model selection under clean and contaminated regimes, and a real-data application to maize drought-response traits in \cref{appendix_real_data}.
Then, we offer concluding remarks, limitations, and future work in \cref{sec_conclusion}.
Proof sketches appear at the end of \cref{sec_rate_aware_sgmoe}, with full proofs deferred to the appendix.
Additional biological background, preprocessing details for the maize dataset, and further geometric and technical discussion are provided in the supplementary material.

{\bf Notation.} Throughout the paper, for any natural number $N\in\Ns$ we abbreviate $\{1,2,\ldots,N\}$ by $[N]$.
Given two sequences of positive real numbers $\{a_N\}_{N=1}^\infty$ and $\{b_N\}_{N=1}^\infty$, we write $a_N=\cO(b_N)$ (equivalently, $a_N\lesssim b_N$) to mean that there exists a constant $C>0$ such that $a_N\le C\,b_N$ for all $N\in\Ns$.
For a vector $\vv\in\R^D$, set $|\vv|:=\evv_1+\cdots+\evv_D$, and let $\|\vv\|_{p}$ denote its $p$-norm; by default, $\|\vv\|$ refers to the $2$-norm unless otherwise stated.
We also use $\|\mA\|$ for the Frobenius norm of a matrix $\mA\in\R^{D\times D}$.
For any set $\sS$, $|\sS|$ denotes its cardinality.
Finally, for two probability density functions $p$ and $q$ with respect to the Lebesgue measure $\mu$, define $\TV(p,q) := \frac 1 2 \int |p-q| d\vmu$ as their Total Variation distance, while $\hels(p,q) := \frac 1 2 \int (\sqrt p - \sqrt q)^2 d\vmu$ denotes the squared Hellinger distance between them. Let $\vTheta$ be the parameter space.
Write $\cE_K(\vTheta)$ for the collection of discrete probability measures on $\vTheta$ with exactly $K$ atoms, and $\cO_K(\vTheta):=\bigcup_{K'\le K}\cE_{K'}(\vTheta)$ for those with at most $K$ atoms.
For a mixing measure $G=\sum_{k=1}^{K}\pi_k\,\delta_{\vtheta_k}$, we (slightly abusively) refer to each component $\pi_k\,\delta_{\vtheta_k}$ as an “atom,” comprising both its weight $\pi_k$ and parameter $\vtheta_k$.
When clear from context, we drop $\vTheta$ and simply write $\cE_K$ and $\cO_K$.

\section{PRELIMINARIES}
\label{sec_preliminaries}

We present a unified result for the parameter estimation rate of the MLE in the SGMoE that simultaneously covers the exact-specified case ($K=K_0$) and the over-specified case ($K>K_0$), building on \cite{nguyen_demystifying_2023}.

{\bf Voronoi Cells.}
For a candidate mixing measure $G=\sum_{k=1}^{K}\exp(\omega_{0k})\,\delta_{(\vomega_{1k},\va_k,b_k,\sigma_k)}$ and the true $G_0=\sum_{k=1}^{K_0}\exp(\omega^0_{0k})\,\delta_{(\vomega^0_{1k},\va^0_k,b^0_k,\sigma^0_k)}$, define for $k\in [K_0]$:
\begin{equation}
    \sA_k (G) := \{ \ell \in [K] : \norm{\vtheta_{\ell} - \vtheta^0_{k}} \leq \norm{\vtheta_{\ell} - \vtheta^0_{j}},\; \forall j \ne k \},\label{eq_voronoi_cell}
\end{equation}
where we denote $\vtheta_{\ell} := (\vomega_{1\ell}, \va_{\ell}, b_{\ell}, \sigma_{\ell}) $.
We use the softmax-translation $(t_0,\vt_1)$ from identifiability (cf. Proposition~1 of \citealp{nguyen_demystifying_2023}) and the shorthand $\Delta_{\vt_{1}} \vomega_{1\ell k} : = \vomega_{1\ell} - \vomega_{1k}^{0} - \vt_{1}$, $\Delta \va_{\ell k} : = \va_{\ell} - \va_{k}^{0}$, $\Delta b_{\ell k} : = b_{\ell} - b_{k}^{0}$, $\Delta \sigma_{\ell k} : = \sigma_{\ell} - \sigma_{k}^{0}$. For notational simplicity, we write $\sA_k$ instead of $\sA_k(G)$.

{\bf Algebraic Obstruction and Exponents.}
For $M\ge2$, let $\bar r(M)$ be the smallest integer $r$ determined by the polynomial system as follows: given $0\le|\vell_1|\le r,\;0\le \ell_2\le r-|\vell_1|,\;|\vell_1|+\ell_2\ge1,$ the polynomial system
\begin{equation}
\sum_{j=1}^{M}\;
\sum_{(\valpha_1,\valpha_2,\alpha_3,\alpha_4)\in\sI_{\vell_1,\ell_2}}
\frac{p_{5j}^2\,p_{1j}^{\valpha_1} p_{2j}^{\valpha_2} p_{3j}^{\alpha_3} p_{4j}^{\alpha_4}}
{\valpha_1!\,\valpha_2!\,\alpha_3!\,\alpha_4!}=0,
\label{eq_system_of_polynomial_recall}
\end{equation}
admits no non-trivial solution (all $p_{5j}\neq0$ and at least one $p_{3j}\neq0$). The ranges of $\valpha_{1}, \valpha_{2}, \alpha_{3}, \alpha_{4}$ in the above sum satisfy $\sI_{\vell_1, \ell_{2}} = \{\valpha = (\valpha_{1}, \valpha_{2}, \alpha_{3}, \alpha_{4}) \in \Ns^{D} \times \Ns^{D} \times \Ns \times \Ns: \ \valpha_{1} + \valpha_{2} = \vell_1, \ |\valpha_{2}| + \alpha_{3} + 2 \alpha_{4} = \ell_{2}\}$. For general dimension $D$ and parameter $M \geq 2$, finding the exact value of $\bar{r}(M)$ is a non-trivial central problem in algebraic geometry~\citep{sturmfels_solving_2002}. Known values:
\begin{fact}[{\citealp[Lemma 1]{nguyen_demystifying_2023}}]\label{lemma:value_r_unified}
For any $D\ge1$: $\bar r(2)=4$, $\bar r(3)=6$, and $\bar r(M)\ge7$ for $M\ge4$.
\end{fact}

{\bf Classical Overfit-Aware Voronoi Distance.}
Define a single loss that reduces to the exact-fit metric when each cell has one atom, and adds over-fit penalties otherwise:
\begin{align}
    &\VDO(G, G_{0}) : = \VDE(G, G_{0}) \nonumber\\
    &+ \inf_{t_{0}, \vt_{1}}\sum_{k: |\sA_{k}| > 1} \sum_{\ell \in \sA_{k}} \exp(\omega_{0\ell}) \Big(\|(\Delta_{\vt_{1}} \vomega_{1\ell k},  \Delta b_{\ell k})\|^{\bar{r}(|\sA_{k}|)} \nonumber \\
    &\hspace{3cm}+ \|(\Delta \va_{\ell k},\Delta \sigma_{\ell k})\|^{\bar{r}(|\sA_{k}|)/2}\Big),\label{eq_overfitted_loss}\\
    &\VDE(G, G_{0}) : = \inf_{t_{0}, \vt_{1}} \sum_{k = 1}^{K_{0}} \Big|\sum_{\ell \in \sA_{k}} \exp(\omega_{0\ell}) - \exp(\omega_{0k}^{0} + t_{0})\Big| \nonumber \\
    &+ \sum_{k: |\sA_{k}| = 1} \sum_{\ell \in \sA_{k}} \exp(\omega_{0\ell}) \|(\Delta_{\vt_{1}} \vomega_{1\ell k}, \Delta \va_{\ell k},  \Delta b_{\ell k}, \Delta \sigma_{\ell k})\|.\nonumber
\end{align}
When $|\sA_k|=1$ for all $k$ (i.e., $K=K_0$), \cref{eq_overfitted_loss} equals the exact-fit metric $\VDE$; if some $|\sA_k|>1$ (i.e., $K>K_0$), \cref{eq_overfitted_loss} adds the higher-order penalties determined by $\bar r(\cdot)$.

\begin{fact}[{\citealp[Theorems 1 and 2]{nguyen_demystifying_2023}}]\label{thm:unified_rate}There exist universal constants $C,c>0$ (depending only on $G_0$ and $\vTheta$) \st the MLE $\widehat G_N$ of order $K\ge K_0$ satisfies
\begin{equation}\label{eq_unified_rate_bound}
\mathbb{P}\!\Big(\VDO(\widehat G_N,G_0)>C\,(\log N/N)^{1/2}\Big)\;\lesssim\;e^{-c\log N}.
\end{equation}
\end{fact}

\emph{Remarks.} 
(i) If $K=K_0$ (all $|\sA_k|=1$), then $\VDO=\VDE$ and \cref{eq_unified_rate_bound} yields the exact-specified rate
\(
\mathbb{P}\!\big(\VDE(\widehat G_N,G_0)>C(\log N/N)^{1/2}\big)\lesssim e^{-c\log N},
\)
implying parametric ($N^{-1/2}$ up to logs) estimation of
$\exp(\omega^0_{0k})$, $\vomega^0_{1k}$ (up to translation), $\va^0_k$, $b^0_k$, $\sigma^0_k$ for all $k\in[K_0]$.
(ii) If $K>K_0$ (some $|\sA_k|>1$), the same bound holds for $\VDO$, while the exponents $\bar r(|\sA_k|)$ inside \cref{eq_overfitted_loss} encode the slower algebraic behavior of over-covered parameters within each Voronoi cell.

\section{FAST-RATE-AWARE EXPERT AGGREGATION IN SGMOE}
\label{sec_rate_aware_sgmoe}

\subsection{Why Merge Experts? The Rate Gap}
\label{subsec_motivation_rate_gap}
Building on \cref{sec_preliminaries}, identifiability and the unified parameter-rate bound (\cref{thm:unified_rate}) imply that converting density accuracy into parameter accuracy hinges on a suitable inverse (loss) inequality. When the model is over-specified ($K>K_0$), several fitted atoms may fall into the same Voronoi cell $\sA_k$ (defined in \cref{eq_voronoi_cell}), which induces a \emph{rate gap}: single-covered truths achieve (near) parametric rates, whereas multi-covered truths converge more slowly with exponents governed by $\bar r(|\sA_k|)$ from \cref{sec_preliminaries}. To exploit this, we (i) refine the loss to expose mergeable structure, and (ii) aggregate (merge) near-duplicate atoms to recover fast rates and guide model order selection.

\subsection{A Fast-Rate-Aware Voronoi Distance}
\label{subsec_rate_aware_loss}

{\bf Our Proposal.}
Let $\VDO(G,G_0)$ denote the over-fit Voronoi loss from \cref{eq_overfitted_loss} and $\sA_k$ be as in \cref{eq_voronoi_cell}. We augment it with first-order “merged-moment” couplings inside multi-covered cells to obtain
\begin{align}
    &\VDFRA (G, G_0) :=  \VDO (G, G_0) \nonumber\\
    & \quad+\inf_{t_{0}, \vt_{1}} \sum_{k: |\sA_{k}| > 1} \bigg( \norm{\sum_{\ell \in \sA_k} \exp{(\omega_{0\ell})} (\Delta b_{\ell k})} \nonumber \\
    &\qquad+ \norm{\sum_{\ell \in \sA_k} \exp{(\omega_{0\ell})} (\Delta_{\vt_1} \vomega_{1\ell k})} \nonumber \\
    &\qquad+ \norm{\sum_{\ell \in \sA_k} \exp{(\omega_{0\ell})}[(\Delta b_{\ell k})^2 + (\Delta \sigma_{\ell k})]} \nonumber \\
    &\qquad+ \norm{\sum_{\ell \in \sA_k} \exp{(\omega_{0\ell})}[(\Delta_{\vt_1} \vomega_{1\ell k})(\Delta b_{\ell k}) + (\Delta \va_{\ell k})]} \nonumber\\
    &\qquad+  \norm{\sum_{\ell \in \sA_k} \exp{(\omega_{0\ell})}(\Delta_{\vt_1} \vomega_{1\ell k})(\Delta_{\vt_1} \vomega_{1\ell k})^{\top}} \bigg).\label{eq_def_new_voronoi_D}
\end{align}

{\bf Link to \texorpdfstring{\cref{sec_preliminaries}}{Preliminaries}.}
The penalties inside \cref{eq_def_new_voronoi_D} are consistent with the exponents $\bar r(|\sA_k|)$ that appear in the unified loss \cref{eq_overfitted_loss}: when $|\sA_k|=1$, \(\VDFRA\) reduces to the exact-fit metric \(\VDE\); when $|\sA_k|>1$, the added block-sums control the slow directions and quantify how well the cell behaves \emph{as if merged}.

{\bf Motivation for Merging.} Because the slow rates originate from multiple atoms sharing a cell, replacing these atoms by their softmax-weighted aggregate collapses the problematic directions and restores first-order (parametric) behavior for the merged parameters. Thus, \(\VDFRA\) both (i) certifies where merging is beneficial (large intra-cell terms) and (ii) predicts the rate improvement obtained by aggregation, which we leverage next for hierarchical merging and model selection.

\subsection{A Merge Operator Tailored to SGMoE}
\label{subsec_merge_operator}

{\bf Connection to \cref{sec_preliminaries} and Novelty.}
The unified rate result in \cref{sec_preliminaries} shows that parameter convergence hinges on how fitted atoms distribute across Voronoi cells; multi-covered cells induce slower algebraic behavior governed by $\bar r(\cdot)$. The merge operator below is the first ingredient of our contribution: it \emph{operationalizes} that insight by collapsing near-duplicate atoms within a cell using softmax-weighted updates. This turns slow, multi-component directions into a single, first-order direction, setting up our fast pathwise rates (\cref{thm:path_rates}) and height/likelihood controls (\cref{thm:heights,thm:likelihood_path}).

{\bf Rate-Weighted Dissimilarity.}
For $G^{(K)}=\sum_{k=1}^K \exp(\omega_{0k})\delta_{\vtheta_k}$ with $\vtheta_k=(\vomega_{1k},\va_k,b_k,\sigma_k)$, define
\begin{align}
&\divclus\!\left(\exp(\omega_{0\ell_1})\delta_{\vtheta_{\ell_1}},
\exp(\omega_{0\ell_2})\delta_{\vtheta_{\ell_2}}\right)\nonumber\\
&:=\frac{\exp(\omega_{0\ell_1}+\omega_{0\ell_2})}{\exp(\omega_{0\ell_1})+\exp(\omega_{0\ell_2})}
\|(\vomega_{1\ell_1},b_{\ell_1})-(\vomega_{1\ell_2},b_{\ell_2})\|^2
 \nonumber \\
 & + \frac{\exp(\omega_{0\ell_1}+\omega_{0\ell_2})}{\exp(\omega_{0\ell_1})+\exp(\omega_{0\ell_2})}\|(\va_{\ell_1},\sigma_{\ell_1})-(\va_{\ell_2},\sigma_{\ell_2})\|.
\label{eq_dissim_sgmoe}
\end{align}
Pick $(i,j)=\argmin_{\ell_1\neq\ell_2\in[K]}\divclus(\cdot,\cdot)$ and replace the pair by the \emph{softmax-weighted} aggregate
\begin{align}
    \omega_{0*} &= \log \left( \exp{\omega_{0i}} + \exp{\omega_{0j}} \right), \nonumber \\
    \vomega_{1*} &= \exp{\left(\omega_{0i} - \omega_{0*} \right)}\vomega_{1i} + \exp{\left(\omega_{0j} - \omega_{0*}\right)} \vomega_{1j}, \nonumber\\
    b_* &= \exp{\left(\omega_{0i} - \omega_{0*} \right)} b_i + \exp{\left(\omega_{0j} - \omega_{0*}\right)} b_j, \nonumber\\
    \va_* &= \frac{\exp(\omega_{0i})}{\exp(\omega_{0*})} [(\vomega_{1i} - \vomega_{1*}) (b_i - b_*) + \va_i] \nonumber\\
    &\quad+ \frac{\exp(\omega_{0j})}{\exp(\omega_{0*})} [(\vomega_{1j} - \vomega_{1*}) (b_j - b_*) + \va_j], \nonumber \\
    \sigma_* &= \frac{\exp(\omega_{0i})}{\exp(\omega_{0*})} [(b_i - b_*)^2 + \sigma_i] \nonumber\\
    &\quad + \frac{\exp(\omega_{0j})}{\exp(\omega_{0*})} [ (b_j - b_*)^2 + \sigma_j].\label{eq_merging_SGMoE_models}
\end{align}

Then we define $G^{(K - 1)} = \exp(\omega_{0*}^{}) \delta_{(\vomega_{1*}^{}, \va_*, b_*, \sigma_*)} + \sum_{k \ne i, j} \exp(\omega_{0k}^{}) \delta_{(\vomega_{1k}^{}, \va_k, b_k, \sigma_k)}$. A description of the whole procedure can be seen in \cref{alg:sgmoe_merge}. The choice of merging atoms and deriving the new atom (\cref{eq_dissim_sgmoe,eq_merging_SGMoE_models}) are in particular faithful to hierarchical clustering and $K$-means algorithms. 

\subsection{The Hierarchical View of Aggregation Path}
\label{subsec_agg_path}

{\bf Transition and Main Idea.} The merge step converts local redundancy into a single effective atom. Repeating it induces a \emph{global} hierarchy, the aggregation path, along which our new analysis proves a \emph{monotone strengthening} of the loss and, crucially, \emph{fast} convergence at every level. This bridges \cref{sec_preliminaries} (unified loss but slow rates) with a constructive, data-driven path that achieves the same near-parametric behavior after aggregation.

Having presented the algorithm to choose and merge a mixing measure with $K$ atoms to $K-1$ atoms, we now describe the dendrogram (hierarchical aggregation) of $G$ that emerges by repeatedly applying the merging procedure.

{\bf Dendrogram (Hierarchical Aggregation).}
Iterate the merge in \cref{eq_dissim_sgmoe,eq_merging_SGMoE_models} from $\kappa=K$ down to $2$, generating $\{G^{(\kappa)}\}_{\kappa=2}^K$.
Define the dendrogram $\cT(G)=(\gV,\gE,\gH)$ with $\gV$ containing $K$ levels, the $\kappa$-th level holding the atoms of $G^{(\kappa)}$, $\gE$ storing the links between merged pairs across adjacent levels, and $\gH=(\height^{(K)},\dots,\height^{(2)})$ with
\(
\height^{(\kappa)}:=\min\{\divclus(\cdot,\cdot)\text{ over pairs in }G^{(\kappa)}\}.
\)
The quantity $\height^{(\kappa)}$ is the height between levels $\kappa$ and $\kappa-1$.

When we represent $\cT(G)$ on a graph, $\height^{(\kappa)}$ is the height between $\kappa$-th level and $(\kappa-1)$-th level. The procedure to construct the dendrogram of $G$ is given by \cref{alg:sgmoe_path}.

\begin{algorithm}[!ht]
\caption{SGMoE Merge Step (Fast-Rate-Aware)}
\label{alg:sgmoe_merge}
\begin{algorithmic}[1]
\Require $G^{(\kappa)}=\sum_{k=1}^{\kappa}\exp(\omega_{0k})\delta_{(\vomega_{1k},\va_k,b_k,\sigma_k)}$
\State $(i,j)\gets\argmin\{\divclus\big(\exp(\omega_{0\ell_1})\delta_{\vtheta_{\ell_1}},\exp(\omega_{0\ell_2})\delta_{\vtheta_{\ell_2}}\big): \ell_1\neq\ell_2\in[\kappa]\}$
\State Compute $(\omega_{0*},\vomega_{1*},\va_*,b_*,\sigma_*)$ by \cref{eq_merging_SGMoE_models}
\State \Return $G^{(\kappa-1)}=\exp(\omega_{0*})\delta_{(\vomega_{1*},\va_*,b_*,\sigma_*)}+\sum_{k\neq i,j}\exp(\omega_{0k})\delta_{\vtheta_k}$
\end{algorithmic}
\end{algorithm}

\begin{algorithm}[!ht]
\caption{SGMoE Hierarchical Aggregation Path}
\label{alg:sgmoe_path}
\begin{algorithmic}[1]
\Require $G^{(K)}=\sum_{k=1}^{K}\exp(\omega_{0k})\delta_{(\vomega_{1k},\va_k,b_k,\sigma_k)}$
\State Initialize $\cT(G)=(\gV,\gE,\gH)$ with $\gV_K=\{\text{atoms of }G^{(K)}\}$, $\gE=\varnothing$
\For{$\kappa=K,\dots,2$}
  \State $G^{(\kappa-1)}\gets \text{Algorithm \ref{alg:sgmoe_merge}}(G^{(\kappa)})$
  \State Append atoms of $G^{(\kappa-1)}$ to level $\gV_{\kappa-1}$, link merged pair in $\gE$
  \State $\height^{(\kappa)}\gets \min\divclus(\cdot,\cdot)$ over pairs in $G^{(\kappa)}$; append to $\gH$
\EndFor
\State \Return $\cT(G)=(\gV,\gE,\gH)$ and $\{G^{(\kappa)}\}_{\kappa=1}^{K}$
\end{algorithmic}
\end{algorithm}

{\bf Monotone Strengthening of the Loss (Bridge to Fast Rates).}
The following lemma formalizes that each merge step cannot increase our fast-rate-aware distance to $G_0$, making the path progressively \emph{easier} to estimate:
\begin{lemma}
\label{lem:monotone_path}
As $\VDFRA (G^{(K)},G_0)\to0$,
\(
\VDFRA(G^{(K)},G_0)\;\gtrsim\;\VDFRA(G^{(K-1)},G_0)\;\gtrsim\;\cdots\;\gtrsim\;\VDFRA(G^{(K_0)},G_0),
\)
with constants depending only on $G_0$, $\vTheta$, and $K$.
\end{lemma}

{\bf Behavior of the Path for the MLE (Main Fast-Rate Theorem).}  
Leveraging the monotonicity above together with the unified inverse bound from \cref{sec_preliminaries}, we obtain fast rates \emph{at every level} of the path, including the exact-fit and under-fit levels where aggregation recovers  optimal parametric rate behavior:
\begin{theorem}[Fast convergence rates along the path]
\label{thm:path_rates}
There exist universal constants $C'_1,c_1,C'_2,c_2>0$ such that for all $\kappa\in[K_0+1,K]$ and $\kappa'\in[K_0]$, we have
\begin{align}
    &\sP\!\big(\VDFRA(\widehat G_N^{(\kappa)},G_0)>C'_2(\log N/N)^{1/2}\big)\lesssim e^{-c_2\log N}, \\
    &\sP\!\big(\VDE(\widehat G_N^{(\kappa')},G_0^{(\kappa')})>C'_1(\log N/N)^{1/2}\big)\lesssim e^{-c_1\log N}.\nonumber
\end{align}
\end{theorem}

\subsection{Heights and Likelihood Along the Path}
\label{subsec_heights_lik}

{\bf Transition from Structure to Statistics.}
Heights summarize structural redundancy; likelihood captures statistical fit. Our second set of novel guarantees shows (i) heights shrink at a rate dictated by $\bar r(\widehat G_N) := \max_{k \in [K_0]} \bar{r}(|\sA_k (\widehat G_N)|)$, and (ii) the empirical likelihood concentrates to its population counterpart along the path.

{\bf Height Definitions.}
For all $\kappa\in[K_0+1,K]$ and $\kappa'\in[K_0]$, let
\begin{align}
    \height_N^{(\kappa)}:=\min&\Big\{\divclus\!\big(\exp(\widehat\omega_{0k_1})\delta_{\widehat\vtheta_{k_1}},
\exp(\widehat\omega_{0k_2})\delta_{\widehat\vtheta_{k_2}}\big) \nonumber \\
    &\qquad: k_1\neq k_2,\ \text{atoms of }\widehat G_N^{(\kappa)}\Big\},
\end{align}
and let $\height_0^{(\kappa')}$ be the analogous height on the true path. Then:
\begin{theorem}[Height control]
\label{thm:heights}
For all $\kappa\in[K_0+1,K]$ and $\kappa'\in[K_0]$, $\height_N^{(\kappa)}\;\lesssim\;(\log N/N)^{1/\bar r(\widehat G_N)},$ and
\[
\big|\height_N^{(\kappa')}-\height_0^{(\kappa')}\big|\;\lesssim\;(\log N/N)^{1/2},
\]
with constants depending only on $G_0$, $\vTheta$, and $\kappa$.
\end{theorem}

{\bf Likelihood.}
We define empirical average log-likelihood and population average log-likelihood as follow: $\bar\ell_N(p_G):=N^{-1}\sum_{n=1}^N\log p_G(y_n| \rvx_n)$ and $\cL(p_G):=\E_{(\rvx,y)\sim P_{G_0}}[\log p_G(y| \vx)]$.

\textbf{Condition K.} There exist positive constants $c_{\alpha}$ and $c_{\beta}$ such that for all sufficiently small $\epsilon$ and $\vtheta_0, \vtheta \in \vTheta$ such that $\norm{\vtheta - \vtheta_0} \leq \epsilon$, we have $\log f(\vx, y|\vtheta) \geq (1 + c_{\beta} \epsilon) \log f(\vx, y|\vtheta_0) - c_{\alpha} \epsilon$.

\begin{theorem}[Likelihood concentration on the path]
\label{thm:likelihood_path}
Assume \textbf{Condition K} hold. Then, for any $\kappa\in[K_0+1,K]$,
\(
\big|\bar\ell_N(p_{\widehat G_N^{(\kappa)}})-\cL(p_{G_0})\big|\lesssim(\log N/N)^{1/(2\bar r(\widehat G_N))}.
\)
Moreover, for $\kappa'\in[K_0]$, $\bar\ell_N(p_{\widehat G_N^{(\kappa')}})\to \cL(p_{G_0^{(\kappa')}})$ in $\sP_{G_0}$-probability as $N\to\infty$.
\end{theorem}

\subsection{Choosing the Number of Experts via a Height-Likelihood Rule}
\label{subsec_order_selection}

{\bf Novel Model Selection Principle.}
By combining structural signal (heights) and statistical fit (likelihood), our DSC favors models that are both well-separated and well-supported by the data, unlike AIC/BIC/ICL, which ignore the geometry of the fitted atoms.

{\bf DSC Definition.}
For each level $\kappa$, define
\[
\mathrm{DSC}_N^{(\kappa)}:=-\Big(\height_N^{(\kappa)}+\epsilon_N\,\bar\ell_N\!\big(p_{\widehat G_N^{(\kappa)}}\big)\Big),
\]
where the weight $\epsilon_N$ satifies $1\ll \epsilon_N \ll (N/\log N)^{1/(2\bar r(\widehat G_N))}$.
 A practical choice is $\epsilon_N:=\log N$. Select
\[
\widehat K_N:=\argmin_{\kappa\in[2,K]}\mathrm{DSC}_N^{(\kappa)}.
\]

\begin{theorem}[Consistency of model selection]
\label{thm_order_consistency}
Assume that data are generated by a softmax-gated Gaussian MoE, the parameter space $\vTheta$ is compact, the covariate support $\cX \subset \R^D$ is bounded, the DSC uses a penalty $\epsilon_N$ satisfying as above, and the true component $K_0\ge2$. Then $\widehat K_N\to K_0$ in $\sP_{G_0}$-probability as $N\to\infty$.
\end{theorem}

{\bf Interpretation.}
Unlike pure likelihood criteria (AIC/BIC/ICL), $\mathrm{DSC}_N^{(\kappa)}$ also penalizes \emph{structural closeness} through $\height_N^{(\kappa)}$. Small heights indicate either redundant atoms (near-duplicates) or atoms with tiny softmax weights; both are symptomatic of over-specification. The joint use of heights and likelihood therefore yields a more robust selection rule in SGMoE.

\subsection{Proof Sketches}

We sketch the proofs of \cref{lem:monotone_path,thm:path_rates,thm:heights,thm:likelihood_path,thm_order_consistency}, which together establish monotonicity along the dendrogram path and consistency of the dendrogram-based model selection. We first motivate the fast-rate-aware Voronoi distance in \cref{eq_def_new_voronoi_D}. When $\widehat{G}_N \to G_0$, over-specification yields Voronoi cells with $|\sA_k^N|>1$. Repeatedly merging such atoms eventually makes every cell singleton, which motivates our construction. Using the density decomposition
\begin{align*}
    Q_{N}&=\Big[\sum_{k=1}^{K_{0}} \exp \big((\vomega_{1 k}^{0}+\vt_1)^{\top} x+\omega_{0 k}^{0}+t_{0}\big)\Big]\\
    &\quad \times\big[p_{G_{N}}(y|\vx)-p_{G_{0}}(y|\vx)\big],
\end{align*}
we analyze the sums over indices with $|\sA_k^N|>1$ under $1 \le |\vell_1|+\ell_2 \le 2\bar r(|\sA_k^N|)$. For clarity, we also consider $(\vell_1,\ell_2)$ with $1 \le |\vell_1|+\ell_2 \le 2$, which corresponds to $|\sA_k^N|=1$. This reasoning leads to the merging algorithm.

{\bf Proof Sketch of \cref{lem:monotone_path}.}
Proceed by induction on $\kappa\in[K_0,K]$ and justify $\VDFRA(G^{(K)},G_0)\gtrsim \VDFRA(G^{(K-1)},G_0)$. As $\VDFRA(G^{(K)},G_0)\to0$, extract a sequence that satisfies $(\va_{\ell}^N,b_{\ell}^N,\sigma_{\ell}^N)\to(\va_k^0,b_k^0,\sigma_k^0)$ and there exist $t_0\in\R$, $\vt_1\in\R^{D}$ with $\sum_{\ell\in\sA_k^N}\exp(\omega_{0\ell}^N)\to \exp(\omega_{0k}^0+t_0)$ and $\vomega_{1\ell}^N\to \vomega_{1k}^0+\vt_1$ for all $\ell\in\sA_k^N$. The minimizing pair $(\ell_1,\ell_2)$ must belong to a common $\sA_k^N$. Using \cref{eq_merging_SGMoE_models} and Jensen’s inequality for the convex maps $z\mapsto\|z\|^{\bar r_k}$ and $z\mapsto\|z\|^{\bar r_k/2}$, it suffices to show
\begin{align*}
&(\exp\omega_{0\ell_1}^N+\exp\omega_{0\ell_2}^N)\, \|(\Delta_{\vt_1}\vomega_{1 * k}^N,\Delta b_{*k}^N)\|^{\bar r_k}\\
&\hspace{2.5cm} \lesssim \sum_{j\in\{\ell_1,\ell_2\}}\exp\omega_{0j}^N\,\|(\Delta_{\vt_1}\vomega_{1 j k}^N,\Delta b_{j k}^N)\|^{\bar r_k},\\
&(\exp\omega_{0\ell_1}^N+\exp\omega_{0\ell_2}^N)\, \|(\Delta \va_{* k}^N,\Delta \sigma_{* k}^N)\|^{\bar r_k/2}\\
&\hspace{2.5cm}\lesssim \sum_{j\in\{\ell_1,\ell_2\}}\exp\omega_{0j}^N\,\|(\Delta \va_{j k}^N,\Delta \sigma_{j k}^N)\|^{\bar r_k/2},
\end{align*}
which yields the desired monotonicity.

{\bf Proof Sketch of \cref{thm:path_rates}.}
Combine \cref{lem:monotone_path} with an inverse bound for $\VDFRA(\widehat G_N,G_0)$. Following \citealp{nguyen_demystifying_2023}, establish
\[
\bbE_\rvx\big[\TV(p_G(\cdot|\vx),p_{G_0}(\cdot|\vx))\big]\;\gtrsim\;\VDFRA(G,G_0),
\]
and use Proposition 2 in \citealp{nguyen_demystifying_2023},  $$\bbE_\rvx[\hels(p_{\widehat G_N}(\cdot|\vx),p_{G_0}(\cdot|\vx))]=\cO_\sP((\log N/N)^{1/2}),$$ to derive the rate for $\widehat G_N$. Apply \cref{lem:monotone_path} to obtain the bounds for $\kappa\in[K_0+1,K]$. For $\kappa'\in[K_0]$, combine the previous rate with the merging formula to conclude.

{\bf Proof Sketch of \cref{thm:heights}.}
Use \cref{thm:path_rates} and the fact that any merged pair lies in the same Voronoi cell. Inequalities analogous to those in \cref{lem:monotone_path,thm:path_rates} translate parameter rates into height bounds.

{\bf Proof Sketch of \cref{thm:likelihood_path}.}
Consider three cases. If $\kappa\ge K_0$, invoke empirical process tools \citep{van_de_geer_empirical_2000} and comparisons between Hellinger and Wasserstein distances \citep{chen_optimal_1995,villani_topics_2003,villani_optimal_2009}. If $\kappa=K_0$, combine \cref{thm:path_rates} with verification that $u(y|\vx;\vomega_1,\va,b,\sigma):=\exp(\vomega_1^\top \vx)\cN(y|\va^\top \vx + b,\sigma)$ satisfies \textbf{Condition K}. If $\kappa<K_0$, conclude via standard convergence arguments.

Finally, \cref{thm_order_consistency} follows from \cref{thm:heights,thm:likelihood_path}.

\section{SIMULATION STUDIES}\label{sec_simulation}

\begin{figure*}[ht]
    \centering
    \begin{subfigure}[b]{0.32\textwidth}
        \centering
        \includegraphics[width=\linewidth]{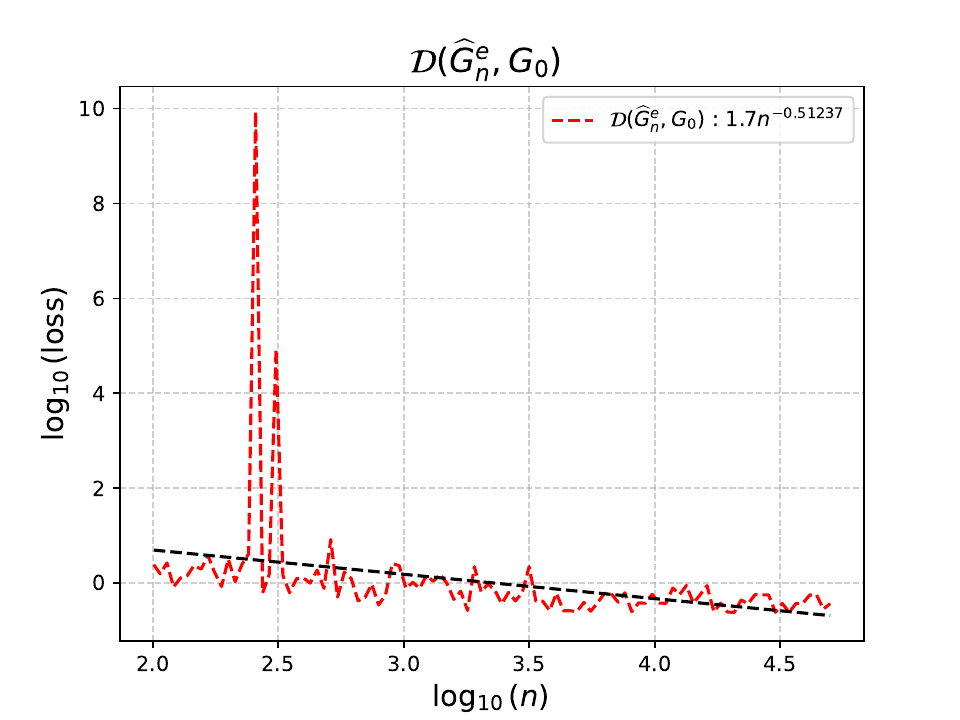}
        \caption{Exact-fitted $\hat{G}^e_n$.}
        \label{fig_voronoi_exact}
    \end{subfigure}
    \hfill
    \begin{subfigure}[b]{0.32\textwidth}
        \centering
        \includegraphics[width=\linewidth]{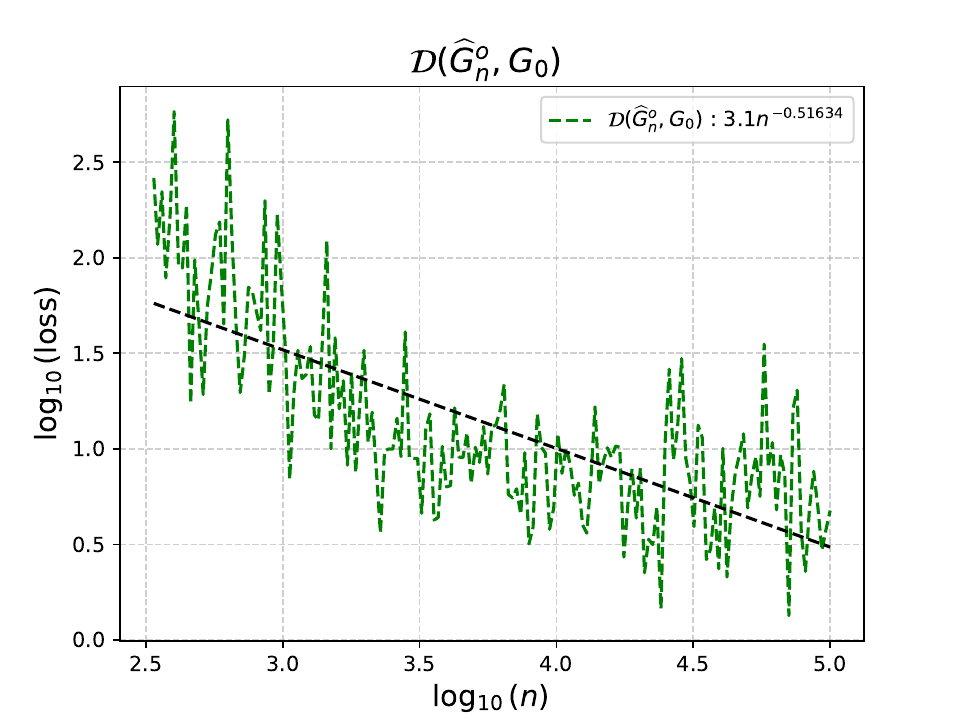}
        \caption{Over-fitted $\hat{G}^o_n$.}
        \label{fig_voronoi_over}
    \end{subfigure}
    \hfill
    \begin{subfigure}[b]{0.32\textwidth}
        \centering
        \includegraphics[width=\linewidth]{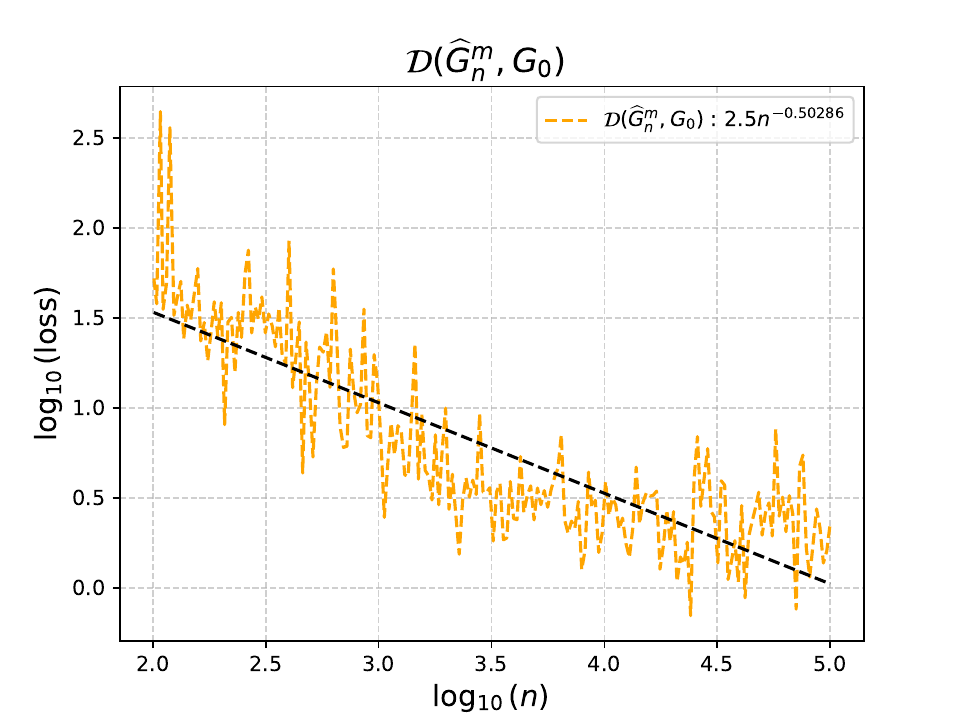}
        \caption{Merged $\hat{G}^m_n$.}
        \label{fig_voronoi_merge}
    \end{subfigure}
    \caption{Convergence under three settings: (a) exact-fitted, (b) over-fitted, and (c) merged mixing measures.}
    \label{fig_voronoi_three}
\end{figure*}

We first show that the dendrogram-based merge yields fast convergence of the mixing measure: starting from an over-fitted estimator that converges slowly, the merged estimator approaches the truth quickly. We then assess model selection via DSC against AIC, BIC, and ICL. Unlike these single-shot selectors, the dendrogram offers a hierarchical view of the fitted atoms, clarifying redundancy and structure. All simulations were run in Python~3.12 on a standard Unix-based system.

{\bf Numerical Schemes.}
The ground-truth mixing measure is
{\small
\begin{align*}
    G_0 &\equiv G_0(2) \;:=\; \sum_{k=1}^{2} \exp(\omega_{0k}^{0})\,\delta_{(\vomega_{1k}^{0}, \va_{k}, \vb_{k}, \sigma_{k})}\\
    &\;=\;
    \exp(-8)\,\delta_{(25,\,-20,\,15,\,0.3)}
    +
    \exp(0)\,\delta_{(0,\,20,\,-5,\,0.4)}.
\end{align*}
}
For each experiment, $N$ varies on a logarithmic grid from $\log_{10}(N_{\min})$ to $\log_{10}(N_{\max})$, yielding $N_{\text{num}}$ sizes in $[N_{\min},N_{\max}]$. At each $N$, we generate $N_{\text{rep}}$ datasets from $G_0$ and compute the exact-fitted MLE $\widehat{G}^{e}_{N}\in\cE_{2}$ and the over-fitted MLE $\widehat{G}^{o}_{N}\in\cO_{4}$ ($K=4$) using an EM variant of \cite{Chamroukhi_2009}. EM stops at tolerance $\epsilon=10^{-6}$ or 2000 iterations. Because the softmax gate in \cref{eq_mixture_expert_1} is translation-invariant, we fix a baseline by setting
$\omega_{0K_{0}}^{0}=0$ and $\vomega_{1K_{0}}^{0}=\zero$.

To stabilize estimation and highlight asymptotics, EM is favorably initialized. For each replication and $(K,K_0)$, split $[K]$ into $K_0$ disjoint sets $\mathbb{S}_1,\ldots,\mathbb{S}_{K_0}$, each nonempty. For $k\in\mathbb{S}_t$, draw
$\veta^0_k=(\omega_{0k}^{0},\vomega_{1k}^{0},\va^0_k,b^0_k,\sigma^0_k)$
from a Gaussian centered at
$\veta^0_t=(\omega_{0t}^{0},\vomega_{1t}^{0},\va^0_t,b^0_t,\sigma^0_t)$
with small covariance. After estimating $\widehat{G}^{o}_{N}$, apply the merging procedure in \cref{alg:sgmoe_path} to obtain $\widehat{G}^{m}_{N}\in\cE_{2}$.

{\bf Fast Parameter Estimation via the Dendrogram.}
We measure accuracy with the Voronoi distance in \cref{eq_def_new_voronoi_D}. For the exact-fitted setting, we use $30$ replicates over $100$ sample sizes with $N\in[10^{2},5\!\times\!10^{4}]$; for the over-fitted setting, $40$ replicates over $165$ sizes with $N\in[338,10^{5}]$; for the merged estimator, $40$ replicates over $200$ sizes with $N\in[10^{2},10^{5}]$. The average loss and a reference slope $N^{-1/2}$ are shown in \cref{fig_voronoi_three}. Results match \cref{thm:path_rates}: the exact-fitted and merged estimators attain the optimal $N^{-1/2}$ rate toward $G_0$, while merging drives the over-fitted estimator to the exact-fit level. For illustration of \cref{alg:sgmoe_path}, \cref{fig_merging_procedure} considers $ G_0(3)$ as follows:
{\small
\begin{align}
   e^{-2}\,\delta_{(3,\,1,\,0,\,1)}
       + e^{1}\,\delta_{(-3.5,\,8,\,7,\,0.8)}
       + e^{0}\,\delta_{(0,\,3,\,5,\,0.6)}. \label{eq_G_0_3}
\end{align}
}

\begin{figure*}[!ht]
    \centering
    \begin{subfigure}[b]{0.32\textwidth}
        \centering
        \includegraphics[width=\linewidth]{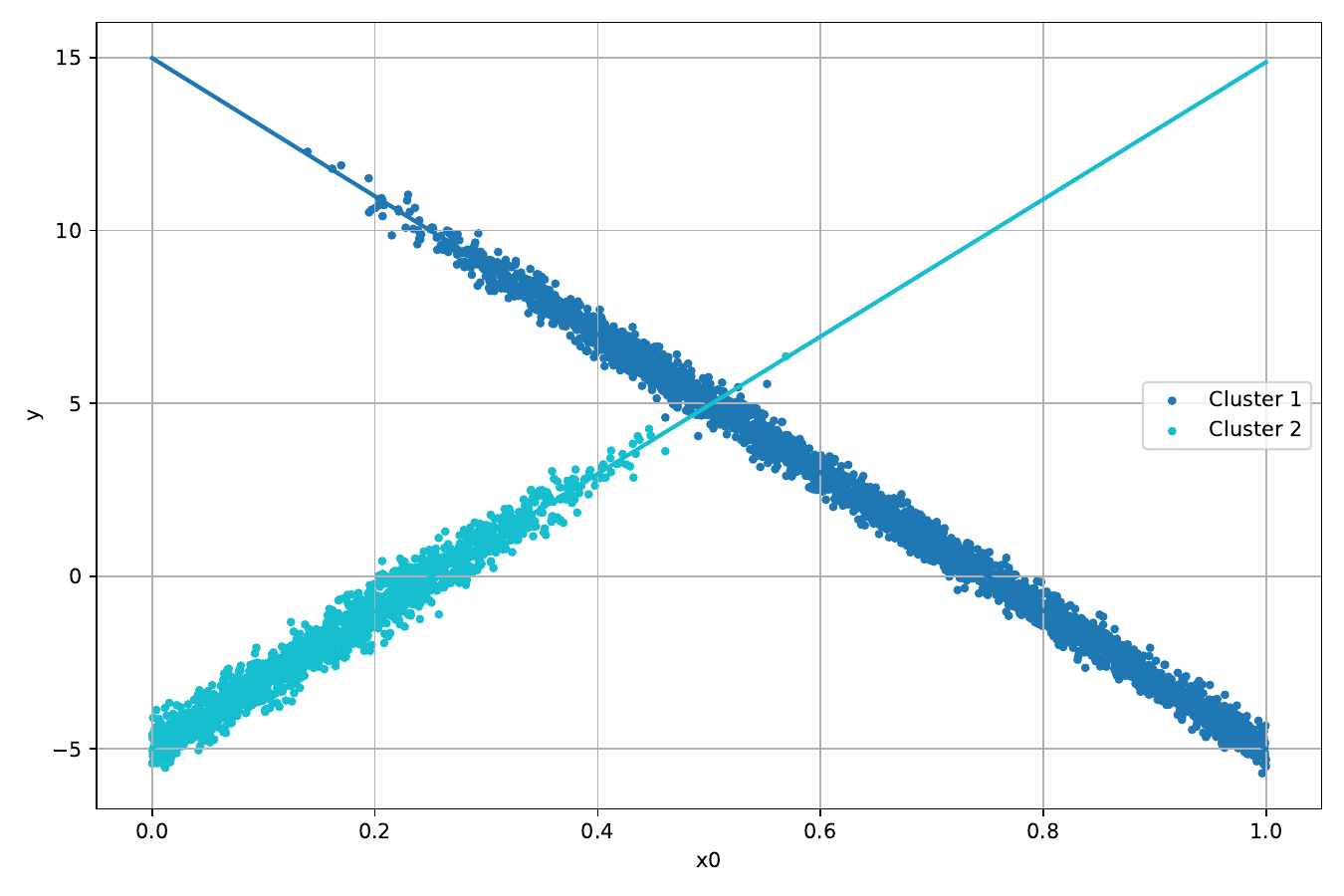}
        \caption{Data clusters for $G_0$ ($n=5000$).}
        \label{fig_2_exp5_5000}
    \end{subfigure}
    \hfill
    \begin{subfigure}[b]{0.32\textwidth}
        \centering
        \includegraphics[width=\linewidth]{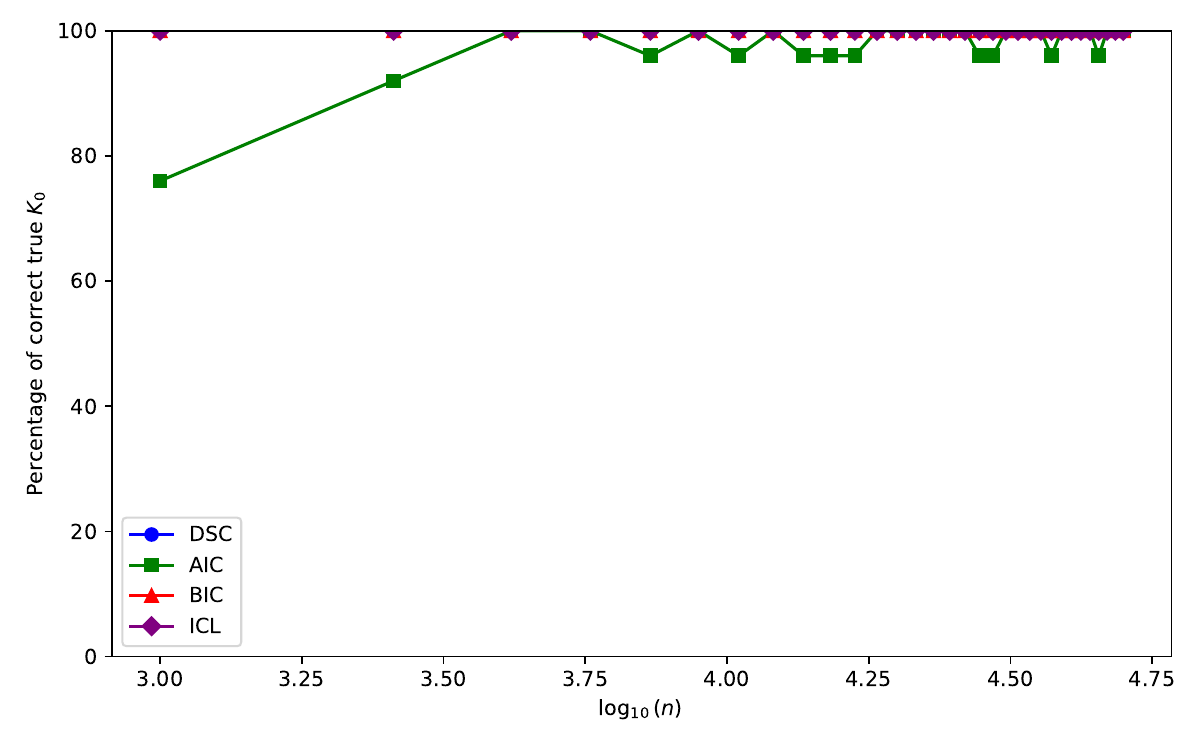}
        \caption{Proportion of correct selections.}
        \label{fig_50000_25tries_percentage}
    \end{subfigure}
    \hfill
    \begin{subfigure}[b]{0.32\textwidth}
        \centering
        \includegraphics[width=\linewidth]{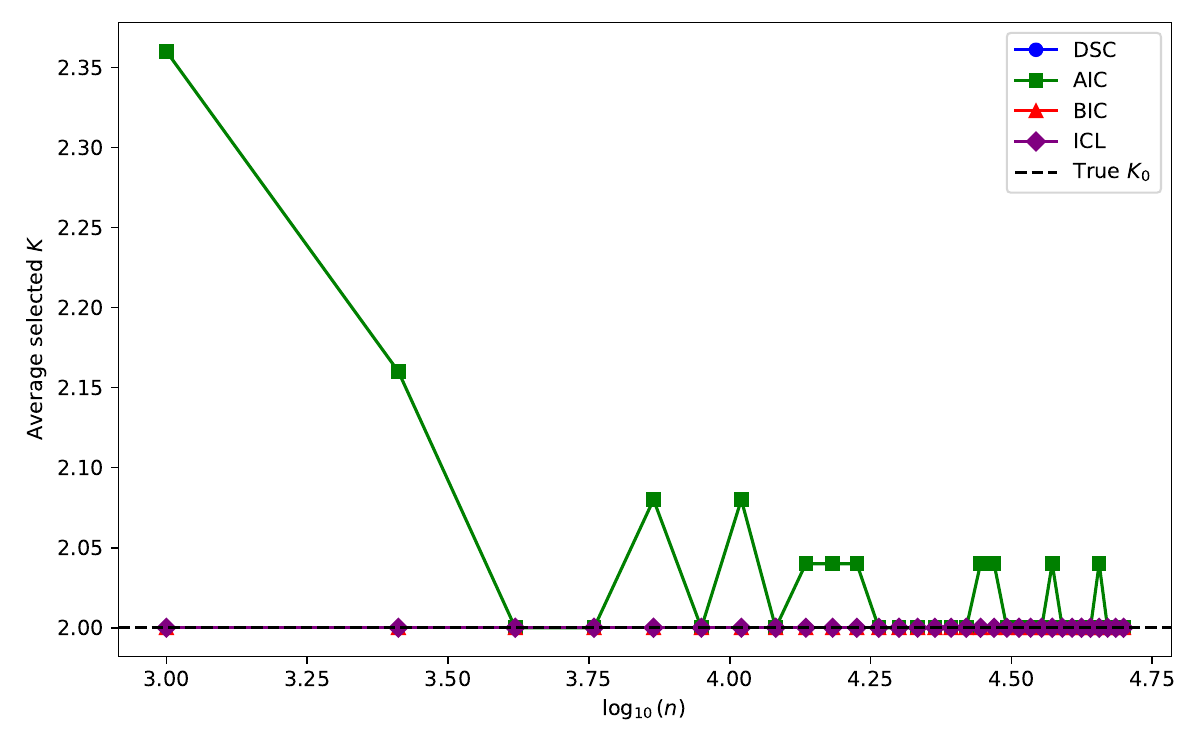}
        \caption{Average selected components.}
        \label{fig_K8_50000_25tries_average}
    \end{subfigure}
    \caption{DSC vs.\ AIC, BIC, and ICL for selecting $K_0 = 2$ of $G_0$.}
    \label{fig_DSC_well_specified}
\end{figure*}

{\bf Model Selection with DSC.}
We compare DSC to AIC, BIC, and ICL over $32$ sample sizes with $N\in[10^{3},5\times 10^{4}]$ and $N_{\text{rep}}=25$. For each method, we report the selection frequency of $K_0$ and the average selected size (see \cref{fig_DSC_well_specified}). AIC/BIC/ICL fit a model for each $\kappa\in[K]$ via EM and pick the best by the corresponding criterion. DSC fits a single SGMoE with $K=4$, builds its dendrogram, and evaluates the criterion with $\omega_N=\log N$ (\cref{subsec_order_selection}). AIC tends to overestimate at small $N$, while all methods recover $K_0$ for large $N$.

{\bf Misspecified Regime.}
We study $\epsilon$-contamination with
$p_0=(1-\epsilon)p_{G_0}+\epsilon q$,
where $q$ is Laplace$(0,1)$. \cref{fig_e-conta_plot_5000} shows the contaminated sample ($n=5000$). \cref{fig_proportion_e-conta,fig_average_e-conta} report the proportion of correct selections and the average selected size. AIC/BIC/ICL behave similarly: they may find $K_0=2$ at small $N$, but tend to overselect as $N$ grows, indicating sensitivity to contamination. DSC, leveraging dendrogram structure, is more robust and continues to select $K_0$ with non-negligible frequency even at large $N$. 

\begin{figure*}[!h]
    \centering
    \begin{subfigure}[b]{0.32\textwidth}
        \centering
        \includegraphics[width=\linewidth]{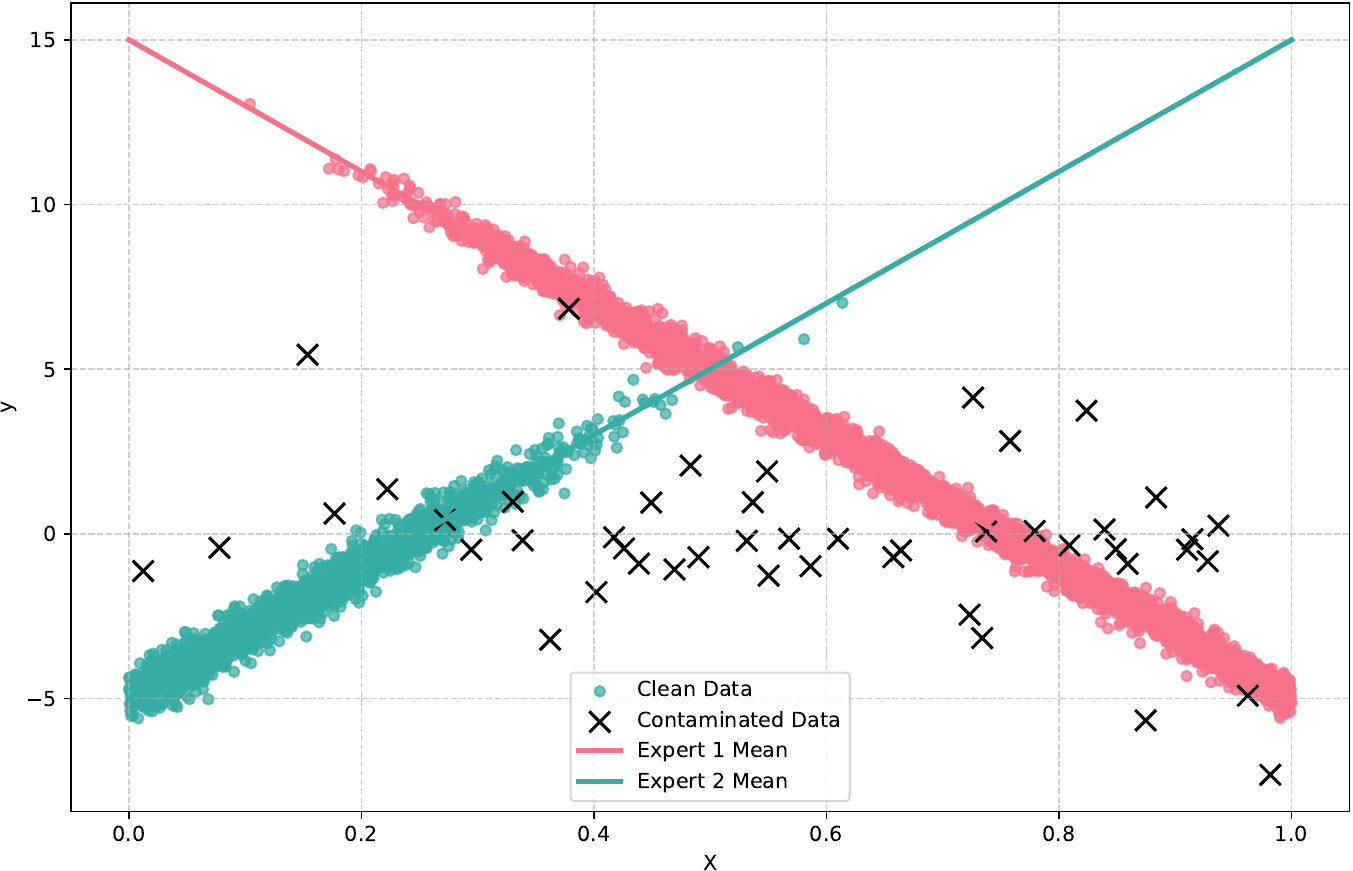}
        \caption{Contaminated sample ($n=5000$).}
        \label{fig_e-conta_plot_5000}
    \end{subfigure}
    \hfill
    \begin{subfigure}[b]{0.32\textwidth}
        \centering
        \includegraphics[width=\linewidth]{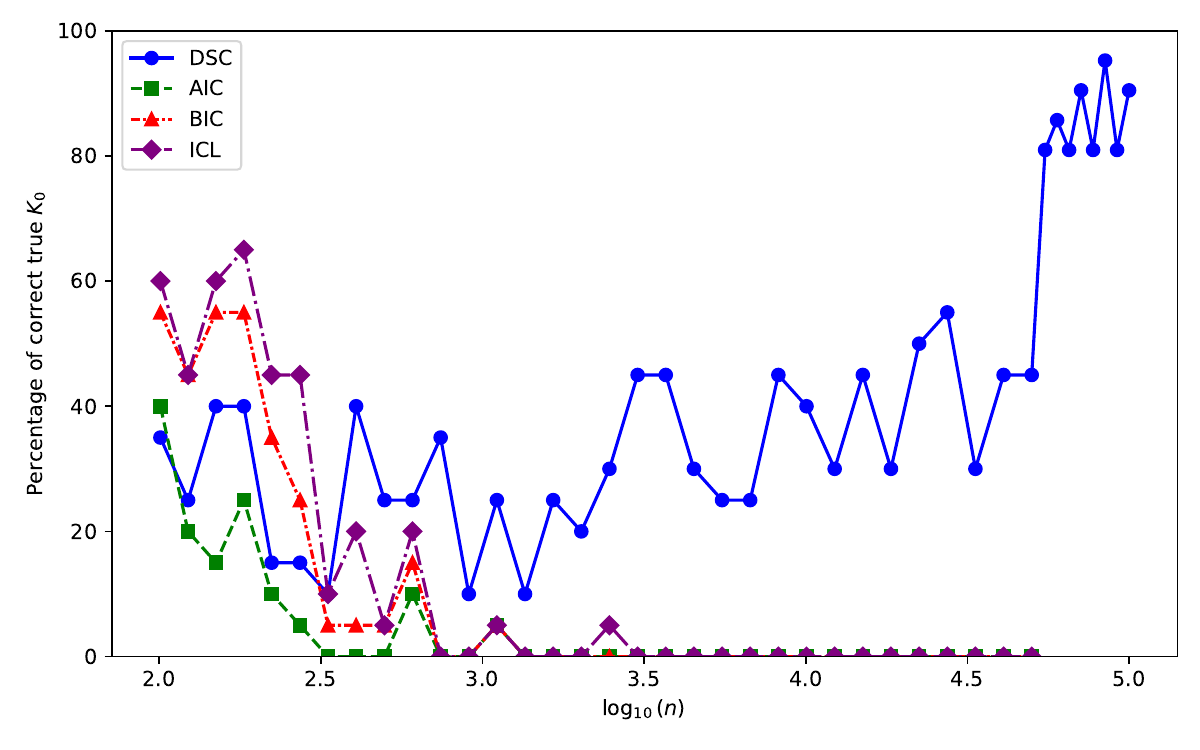}
        \caption{Proportion of correct selections.}
        \label{fig_proportion_e-conta}
    \end{subfigure}
    \hfill
    \begin{subfigure}[b]{0.32\textwidth}
        \centering
        \includegraphics[width=\linewidth]{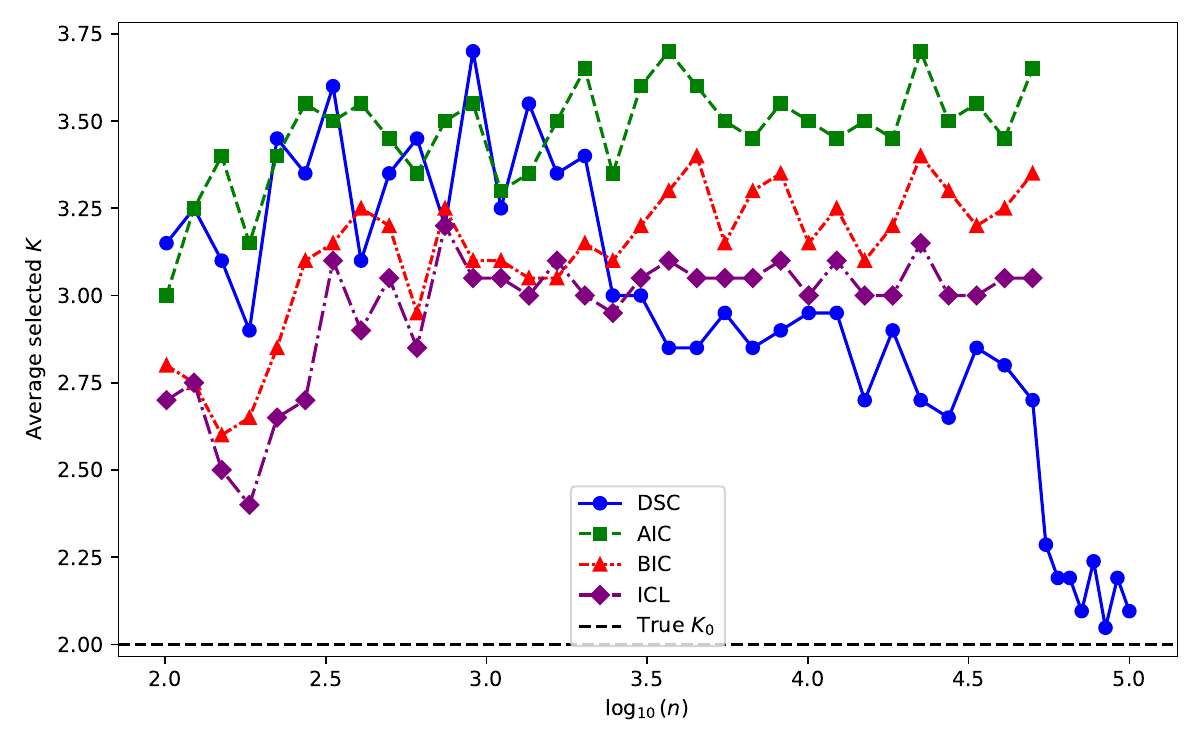}
        \caption{Average selected components.}
        \label{fig_average_e-conta}
    \end{subfigure}
    \caption{Model selection under $\epsilon$-contamination with $K_0=2$. After AIC, BIC, and ICL fail to recover $K_0$, we further test DSC on 8 sample sizes between $5.5\times 10^{4}$ and $10^{5}$, where it still recovers $K_0$ with high frequency.}
    \vspace{-.5cm}
    \label{fig_econtam_results}
\end{figure*}

\section{REAL DATA APPLICATION}\label{appendix_real_data}

We illustrate the dendrogram of mixing measures obtained from our SGMoE model using a real dataset from the study in \cite{10.1093/jrsssc/qlae012}. The data originate from a large-scale experiment on maize aimed at understanding the genetic and molecular bases of drought-responsive traits from proteins expressed in the leaf~\citep{prado2018phenomics, blein2020systems}, where 254 genotypes representing the genetic diversity of dent maize were grown under two watering conditions and phenotyped for seven ecophysiological traits. 

After preprocessing and removing missing data as described in \cite{10.1093/jrsssc/qlae012}, the final dataset consists of 233 maize genotypes ($N = 233$), two ecophysiological traits (outputs), which are \emph{water use} (WU) and the proteins quantified under the \emph{water deficit} (WD) condition, and 973 protein variables (inputs, $D = 973$). To reduce dimensionality and remove irrelevant features, we apply a Lasso procedure to select $D = 10$ protein variables most associated with the target trait and primarily focus on the ecophysiological trait WU. 

We then fit the SGMoE model with $K = 20$ clusters. To ensure a more robust initialization, we first cluster the data into 20 groups using the K-Means algorithm. The resulting cluster assignments are then used to initialize the gating and expert parameters of the SGMoE model, providing a stable starting point for the subsequent steps of the EM algorithm.
\begin{figure*}[!h]
    \centering
    \begin{subfigure}[b]{0.32\textwidth}
        \centering
        \includegraphics[width=\linewidth]{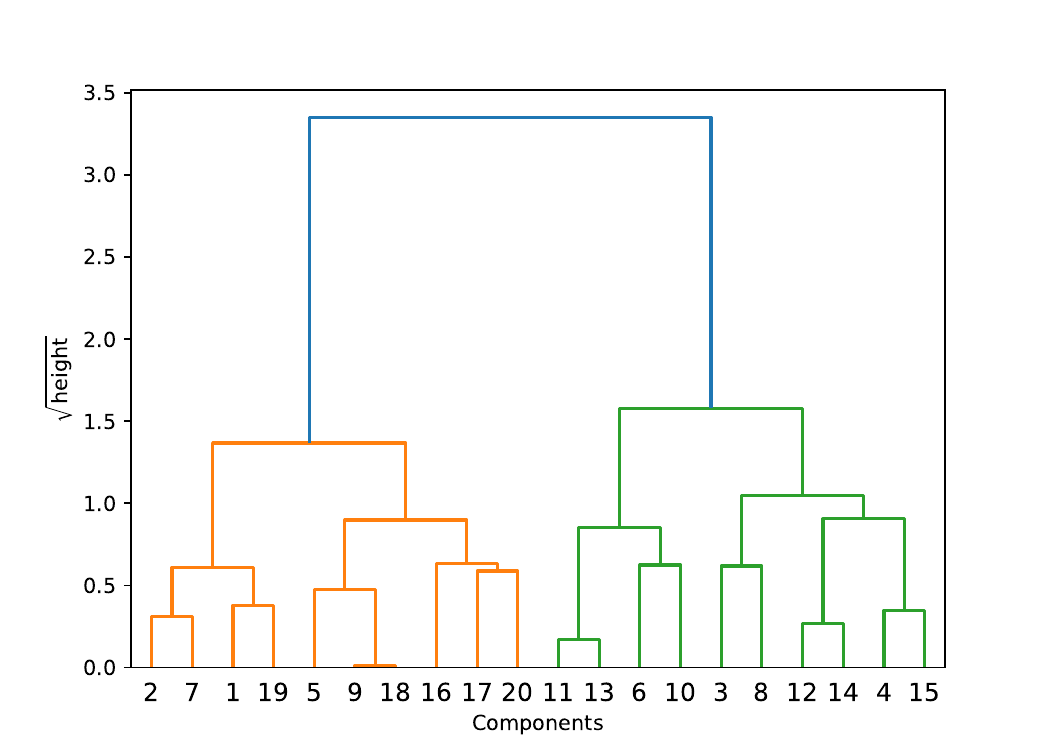}
        \caption{Dendrogram of mixing measure.}
        \label{ddg_mixing_measures}
    \end{subfigure}
    \hfill
    \begin{subfigure}[b]{0.32\textwidth}
        \centering
        \includegraphics[width=\linewidth]{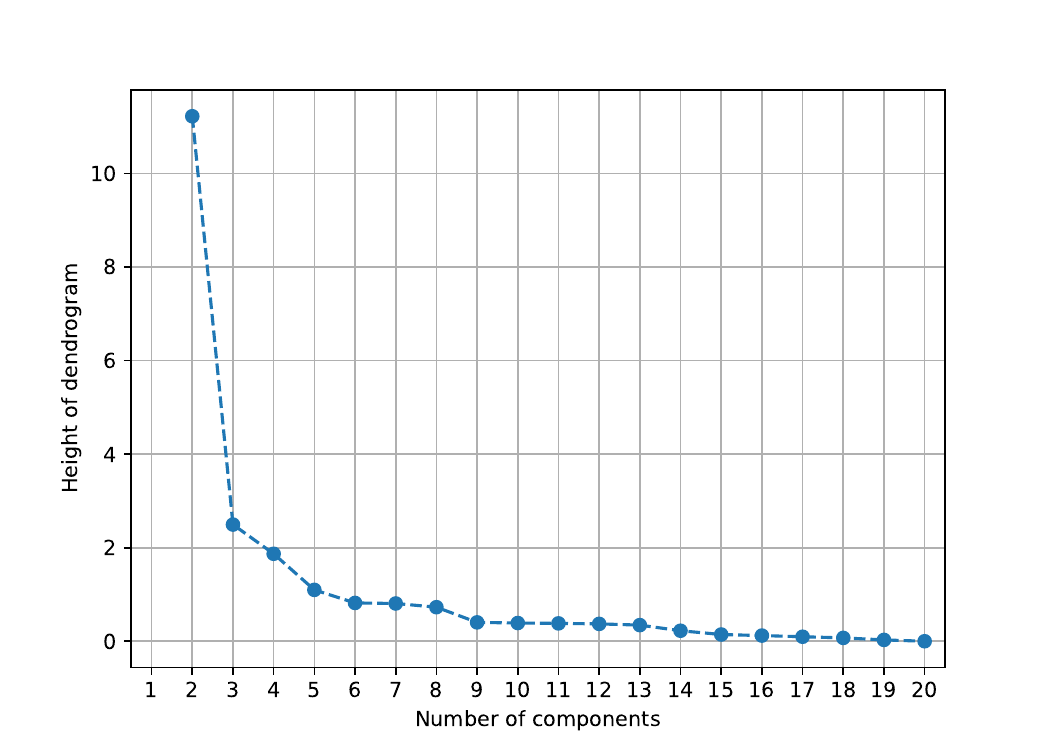}
        \caption{Heights between levels.}
        \label{ddg_height}
    \end{subfigure}
    \hfill
    \begin{subfigure}[b]{0.32\textwidth}
        \centering
        \includegraphics[width=\linewidth]{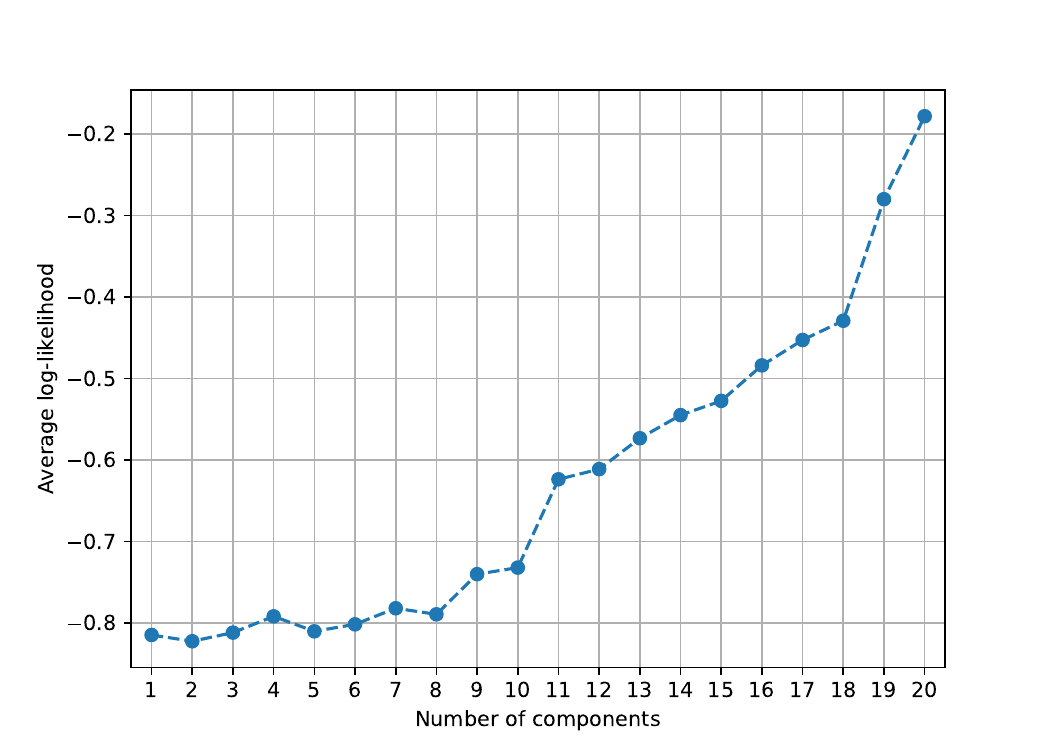}
        \caption{Average log-likelihood across levels.}
        \label{ddg_avg_lllh}
    \end{subfigure}
    \caption{Dendrogram of mixing measure inferred from maize drought-responsive traits dataset.}
    \label{fig_voronoi_three_real_data_set}
\end{figure*}

\cref{ddg_mixing_measures} displays the dendrogram of the fitted mixing measure obtained by \cref{alg:sgmoe_path}, which reveals the hierarchical structure underlying the data. In this experiment, both BIC and ICL select a single component, while DSC selects 2 components, and AIC overestimates with 18 components. The corresponding heights and average log-likelihoods across levels are shown in \cref{ddg_height} and \cref{ddg_avg_lllh}, respectively. We observe that the merging heights generally decrease and approach zero, while the average log-likelihood stabilizes in a few initial levels. Notably, the height at level~2 is much larger than those at subsequent levels, suggesting that there should be two clusters in the data.  

The dendrogram not only facilitates effective model selection but also unveils the hierarchical relationships among mixture components, thereby enhancing the interpretability of the estimated parameters in complex biological data settings.

\section{CONCLUSION}
\label{sec_conclusion}

This work shows that rate-aware geometry, realized through a Voronoi distance together with merging and dendrograms of mixing measures, delivers both fast parameter estimation and consistent, sweep-free model selection in SGMoE. We hope these ideas spur further advances in structured mixture models and expert architectures. Our analysis assumes linear softmax gates, Gaussian experts, compact \(\vTheta\), and bounded covariate support. Extending the theory beyond these settings will require additional regularity and tail controls. Exact values of \(\bar r(M)\) are known for \(M\le3\); for \(M\ge4\) only lower bounds are available. While our guarantees use these bounds, sharper algebraic results would further tighten rates. 

\newpage
\subsubsection*{Acknowledgments}
This project was funded primarily by the Australian Research Council Centre of Excellence for the Mathematical Analysis of Cellular Systems (CE230100001), which supported TrungTin Nguyen and Christopher Drovandi. Christopher Drovandi was also supported by an Australian Research Council Future Fellowship (FT210100260). Additional support was provided by Vietnam National University Ho Chi Minh City (VNU-HCM) under grant number A2025-18-02. The authors also acknowledge Dr.~Dat Do (University of Chicago) for helpful discussions about the dendrogram of mixing measures for mixture models~\citep{do_dendrogram_2024}.

\bibliographystyle{apalike2}
\bibliography{ref}

\section*{Checklist}

\begin{enumerate}

  \item For all models and algorithms presented, check if you include:
  \begin{enumerate}
    \item A clear description of the mathematical setting, assumptions, algorithm, and/or model. [Yes]
    \item An analysis of the properties and complexity (time, space, sample size) of any algorithm. [Yes]
    \item (Optional) Anonymized source code, with specification of all dependencies, including external libraries. [No]
  \end{enumerate}

  \item For any theoretical claim, check if you include:
  \begin{enumerate}
    \item Statements of the full set of assumptions of all theoretical results. [Yes]
    \item Complete proofs of all theoretical results. [Yes]
    \item Clear explanations of any assumptions. [Yes]     
  \end{enumerate}

  \item For all figures and tables that present empirical results, check if you include:
  \begin{enumerate}
    \item The code, data, and instructions needed to reproduce the main experimental results (either in the supplemental material or as a URL). [Yes]
    \item All the training details (e.g., data splits, hyperparameters, how they were chosen). [Yes]
    \item A clear definition of the specific measure or statistics and error bars (e.g., with respect to the random seed after running experiments multiple times). [Yes]
    \item A description of the computing infrastructure used. (e.g., type of GPUs, internal cluster, or cloud provider). [Yes]
  \end{enumerate}

  \item If you are using existing assets (e.g., code, data, models) or curating/releasing new assets, check if you include:
  \begin{enumerate}
    \item Citations of the creator If your work uses existing assets. [Not Applicable]
    \item The license information of the assets, if applicable. [Not Applicable]
    \item New assets either in the supplemental material or as a URL, if applicable. [Not Applicable]
    \item Information about consent from data providers/curators. [Not Applicable]
    \item Discussion of sensible content if applicable, e.g., personally identifiable information or offensive content. [Not Applicable]
  \end{enumerate}

  \item If you used crowdsourcing or conducted research with human subjects, check if you include:
  \begin{enumerate}
    \item The full text of instructions given to participants and screenshots. [Not Applicable]
    \item Descriptions of potential participant risks, with links to Institutional Review Board (IRB) approvals if applicable. [Not Applicable]
    \item The estimated hourly wage paid to participants and the total amount spent on participant compensation. [Not Applicable]
  \end{enumerate}

\end{enumerate}

\clearpage
\appendix
\thispagestyle{empty}

\onecolumn
\aistatstitle{Supplementary Materials for\\ ``Dendrograms of Mixing Measures for Softmax-Gated Gaussian Mixture of Experts: Consistency without Model Sweeps"}

\paragraph{Supplementary Organization.}
This supplement has six parts. \emph{First}, \cref{appendix_real_data_details} provides additional biological background and preprocessing details for the maize drought-response dataset used in the main-paper illustration. \emph{Second}, \cref{app:overview_moe} gives a unified overview of unconditional mixtures, MoE, and SGMoE, highlighting their geometric and statistical differences and clarifying the motivation for the dendrogram framework. \emph{Third}, \cref{appendix_Voronoi} illustrates the Voronoi cells and the merge step underlying the SGMoE aggregation path. \emph{Fourth}, \cref{sec:appendix:challenges} details the main technical challenges, namely softmax translation invariance, gate--expert PDE couplings, and algebraic cancellations; it also discusses in more detail the connection between the polynomial equations in \cref{eq_system_of_polynomial_recall} and the over-specified SGMoE setting. \emph{Fifth}, \cref{sec:appendix:proof-sketches} expands the proof sketches for \cref{lem:monotone_path,thm:path_rates,thm:heights,thm:likelihood_path}, highlighting how the Voronoi geometry drives the analysis. \emph{Finally}, \cref{sec:appendix:proofs-main} presents the full proofs, together with the main proof ingredients and notational conventions used throughout the appendix.

\section{ADDITIONAL DETAILS ON THE MAIZE DROUGHT-RESPONSE DATA}
\label{appendix_real_data_details}

This appendix provides additional biological context and a more explicit description of the data-processing pipeline used for the real-data illustration in the main paper. The goal is to clarify the provenance of the dataset, the meaning of the response and predictor variables, and the rationale for the preprocessing choices, while avoiding repetition of the main-text discussion.

\paragraph{Biological motivation and data provenance.}
The dataset comes from a broader systems-genetics effort on dent maize aimed at linking molecular variation in the leaf proteome to drought-related ecophysiological traits. In that line of work, a genetically diverse maize panel was grown under contrasting watering conditions and characterised using both high-throughput phenotyping and proteomics, with the broader objective of understanding how genotype-dependent molecular responses are related to drought adaptation and plant water-use behaviour \citep{prado2018phenomics,blein2020systems}. The statistical prediction study of \citet{10.1093/jrsssc/qlae012} used these biological measurements as a benchmark for multivariate trait prediction from high-dimensional proteomic covariates, specifically considering drought-related traits measured on a panel of 233 maize genotypes with 973 protein predictors \citep{10.1093/jrsssc/qlae012}. Systems-genetics reports associated with the same experimental programme also emphasise that the maize data were designed to study drought-related traits by integrating proteomic and genomic information \citep{blein2020systems}. 

\paragraph{Why this dataset is relevant here.}
This dataset is well suited to our SGMoE framework for three reasons. First, the sample is biologically heterogeneous: the maize panel spans substantial genetic diversity, so one should not expect all genotypes to follow a single homogeneous regression relationship. Second, drought response is known to be multi-mechanistic, with different molecular programmes potentially associated with different water-use strategies or stress-response profiles. Third, the predictor space is high-dimensional relative to the sample size, which makes model structure and interpretability especially important. These features make the dataset a natural test bed for a method that combines flexible conditional modelling with hierarchical aggregation and model selection. The resulting fitted components can then be interpreted as latent subgroups of genotypes sharing similar proteomic-to-phenotypic relationships rather than as merely algorithmic clusters.

\paragraph{Raw variables used in the illustration.}
Following \citet{10.1093/jrsssc/qlae012}, we work with a cleaned subset of the original experiment after removing observations with missing values. The final analysis set contains $N=233$ maize genotypes. The biological study recorded two drought-related ecophysiological outputs together with quantitative protein abundances measured under water-deficit conditions, yielding a predictor matrix with 973 protein variables before dimension reduction. In the present illustration, we focus primarily on the ecophysiological trait \emph{water use} (WU), while the predictors are the leaf protein abundances measured under the water-deficit regime. This choice is scientifically meaningful because WU is directly linked to drought adaptation and integrates the cumulative effect of genotype-specific physiological regulation under stress.

\paragraph{Preprocessing strategy.}
The preprocessing follows the protocol used for the statistical benchmark in \citet{10.1093/jrsssc/qlae012}, with the same starting point of a cleaned matrix after exclusion of incomplete observations. Since the original proteomic representation is very high-dimensional compared with the number of genotypes, we apply a supervised screening step before fitting the SGMoE. Concretely, we use a Lasso-based variable-selection procedure to extract a smaller subset of proteins that are most strongly associated with the target trait, and we retain $D=10$ proteins for the analysis shown in the main paper. This reduction serves two purposes. Statistically, it improves stability in the small-$N$, large-$D$ regime and reduces the risk that the fitted experts are driven by noise dimensions. Biologically, it yields a more interpretable model by restricting attention to a compact set of drought-informative protein signals. We stress that this Lasso step is used only as a preprocessing device; the clustering, aggregation path, and model selection are all performed by the SGMoE methodology thereafter.

\paragraph{Model fitting for the SGMoE path.}
After preprocessing, we fit an over-specified SGMoE with $K=20$ initial components. Because mixture models can be sensitive to starting values, we initialise the fit using a preliminary $K$-means partition of the genotypes. These initial groups are then used to seed the gating and expert parameters before running the estimation procedure. The purpose of this intentionally over-specified fit is not to interpret all 20 initial components literally, but rather to create a rich starting representation from which the dendrogram path can merge redundant atoms and reveal a more stable low-dimensional structure. In this sense, the over-specified fit plays the same exploratory role as in the synthetic studies: it allows the subsequent aggregation path to separate persistent large-scale structure from small within-cell duplications. 

\paragraph{Interpretation of the fitted path.}
In the main-text illustration, the fitted dendrogram suggests a pronounced split at level $2$, while the average log-likelihood stabilises quickly along the path. From a biological viewpoint, this pattern is consistent with the idea that the maize panel contains a small number of broad genotype groups with distinct proteomic-response profiles under drought, rather than many sharply separated subpopulations. Thus, the selected two-expert solution should be read as a parsimonious summary of two dominant genotype--phenotype response regimes. The value of the dendrogram is therefore twofold: it provides a data-driven model-selection tool, and it offers a hierarchical view of how more complex over-specified representations collapse into a small number of biologically interpretable regimes. 

\paragraph{Why the real-data example is informative for our methodology.}
Unlike the synthetic experiments, this dataset does not come with a known ground-truth number of experts. Its role is instead to illustrate the practical behaviour of the pathwise procedure on a genuinely heterogeneous biological problem. In particular, it shows that the dendrogram can remain informative even when standard information criteria disagree strongly, and that the selected solution can still be interpreted in domain terms through genotype--phenotype structure and early likelihood stabilisation. This complements the theory by demonstrating that the SGMoE aggregation path is not only a technical device for proving rates, but also a practically useful summary of heterogeneity in complex omics-assisted prediction problems.

\section{OVERVIEW OF MIXTURE AND MOE GEOMETRY}
\label{app:overview_moe}

This section provides a unified overview of unconditional mixtures, MoE, and SGMoE, clarifying the geometric and statistical differences that motivate our dendrogram framework.

First, we recall the definitions of unconditional mixtures, MoE, and covariate-free gates.

\begin{itemize}
        \item \textit{Unconditional mixture:}
        \begin{equation*}
            p(y) = \sum_{k=1}^{K} \pi_{k} \times f(y; \veta_{k}).
        \end{equation*}
        \item \textit{MoE:}
        \begin{equation*}
            p(y \mid \vx) = \sum_{k=1}^{K} \pi_{k}(\vx) \times f(y; \veta_{k}(\vx)),
        \end{equation*}
        where both the weights and the experts depend on $\vx$.
        \item \textit{Covariate-free gates:}
        \begin{equation*}
            \pi_{k}(\vx) \equiv \pi_{k},
        \end{equation*}
        which reduces to a \textbf{mixture of regressions} when $\veta_{k}(x)$ is a regression map.
    \end{itemize}

Next, we analyze the difference from existing dendrogram approaches and compare them to Gaussian-gated Gaussian MoE (GGMoE) \citep{thai_model_2025}.

\paragraph{Difference from existing dendrogram approaches.}
Our framework differs from classical dendrogram methods in three key aspects. 
First, we introduce Voronoi-type losses in the gate space that respect softmax symmetry (common translations). 
Second, our method is tailored to conditional SGMoE geometry and provides finite-sample predictive parameter rates, building on \cite{nguyen_demystifying_2023}. Third, earlier dendrogram approaches were developed for \textbf{unconditional} finite mixtures \citep{do_dendrogram_2024} and rely on standard Wasserstein-type losses; they are not tailored to the \textbf{conditional} geometry and softmax-induced couplings of SGMoE. We introduce a \textbf{fast-rate-aware Voronoi loss} $\VDFRA$ that (i) reduces to the exact-fit loss when cells are singletons and (ii) adds merged-moment block sums precisely in the slow directions created by Voronoi multi-coverage. 
This is motivated by the insufficiency of Wasserstein for SGMoE parameter geometry (and even the limitations of early Voronoi losses in \cite{nguyen_demystifying_2023}) and is spelled out in our appendix overview and the formal $\VDFRA$/merge analysis.

\paragraph{Compare to GGMoE.}
GGMoE is a \textbf{generative} MoE that models covariates and gates via Bayes’ rule, enabling closed-form EM M-steps but \textbf{not matching modern deep MoE practice}. Our framework aligns with contemporary \textbf{discriminative softmax/top-$k$} gating \textbf{learned end-to-end} over features: we \textbf{directly} optimize the conditional density $p(y \mid \vx)$ \textbf{without} a generative model for $\vx$ \citep{dai-etal-2024-deepseekmoe,nguyen2024statisticalperspectivetopksparse,Shazeer_JMLR,pham_competesmoeeffective_2024,do_hyperrouter_2023}. This conditional focus aligns with predictive use but is \textbf{analytically harder}: softmax gating introduces a tight numerator–denominator coupling and \textbf{nontrivial gate–expert interactions} that do not arise in GGMoE’s EM updates. Beyond this objective mismatch, we contribute \textbf{Voronoi-type losses} aligned with the \textbf{gate-induced partition} and establish \textbf{finite-sample MLE convergence rates} for SGMoE in both \textbf{exact-fit} and \textbf{over-specified} regimes, addressing the conditional SGMoE geometry directly rather than relying on a generative model for $x$. Empirically, beyond synthetic studies, we analyze a \textbf{maize proteomics} dataset of drought-responsive traits: the \textbf{dendrogram-guided SGMoE path} selects \textbf{two experts}, stabilizes the likelihood early, reveals a clear \textbf{hierarchical structure} in the mixing measure, and yields interpretable \textbf{genotype–phenotype} mappings, \textbf{complementing} GGMoE-centric work whose experiments are primarily synthetic \citep{thai_model_2025}.

\section{ILLUSTRATION OF VORONOI CELLS AND MERGE STEPS FOR SGMoE}\label{appendix_Voronoi}

For a candidate mixing measure
\(G=\sum_{k=1}^{K}\exp(\omega_{0k})\,\delta_{(\vomega_{1k},\va_k,b_k,\sigma_k)}\)
and the true
\(G_0=\sum_{k=1}^{K_0}\exp(\omega^0_{0k})\,\delta_{(\vomega^0_{1k},\va^0_k,b^0_k,\sigma^0_k)}\),
define, for \(k\in[K_0]\), the (parameter-space) Voronoi cell
\begin{equation}
\sA_k(G):=\big\{\ell\in[K]:\ \|\vtheta_\ell-\vtheta^0_k\|
\le \|\vtheta_\ell-\vtheta^0_j\|,\ \forall\, j\neq k\big\},
\label{eq_voronoi_cell_app}
\end{equation}
where \(\vtheta_\ell:=(\vomega_{1\ell},\va_\ell,b_\ell,\sigma_\ell)\).
We use the softmax translation \((t_0,\vt_1)\) from identifiability
(cf. Proposition~1 of \citealp{nguyen_demystifying_2023}) and the shorthand
\(\Delta_{\vt_1}\vomega_{1\ell k}:=\vomega_{1\ell}-\vomega^0_{1k}-\vt_1\),
\(\Delta\va_{\ell k}:=\va_\ell-\va^0_k\),
\(\Delta b_{\ell k}:=b_\ell-b^0_k\),
\(\Delta\sigma_{\ell k}:=\sigma_\ell-\sigma^0_k\).
For brevity we write \(\sA_k\) for \(\sA_k(G)\).
(We restate \cref{eq_voronoi_cell_app} only for completeness; throughout we reference the main-paper definition \cref{eq_voronoi_cell}.)

\textbf{Explanation.}
\cref{fig_voronoi_merge_combo} summarizes the geometry and the merge step used by our method for an example with \(K_0=6\) and \(K=10\): red squares denote true atoms of \(G_0\), blue circles denote fitted atoms of \(G\).
Each Voronoi cell is generated by one true atom, and its cardinality \(|\sA_k|\) equals the number of fitted atoms assigned to that true atom (e.g., two circles in a cell imply \(|\sA_k|=2\)).
Panel~\cref{fig_voronoi_cells} shows the Voronoi partition \(\{\sA_k\}_{k\in[K_0]}\) induced by \(G_0\) as in \cref{eq_voronoi_cell}.
Cells with \(|\sA_k|>1\) reveal redundancy: multiple fitted atoms approximate the same truth and create slow directions.
Panel~\cref{fig_merge_schematic_only} zooms into one such multi-covered cell and depicts the merge step at a visual level:
the closest pair (w.r.t.\ our rate-weighted dissimilarity) is merged into a single aggregate; iterating this operation produces the aggregation path.
Panel~\cref{fig_merge_equations_box} links the visuals to the mathematics: labels “fitted i,” “fitted j,” and “merged *” correspond to
\(\exp(\omega_{0i})\delta_{\vtheta_i}\), \(\exp(\omega_{0j})\delta_{\vtheta_j}\), and \(\exp(\omega_{0*})\delta_{\vtheta_*}\).
Pair selection uses \(\divclus\) from \cref{eq_dissim_sgmoe}, and the softmax-weighted update rules are given in \cref{eq_merging_SGMoE_models}.
Together, these steps collapse slow directions within a cell, strengthen the loss along the path \(\VDFRA\) (\cref{eq_def_new_voronoi_D}), and enable our fast pathwise guarantees and sweep-free model selection via DSC.

\begin{figure*}[!ht]
  \centering
  \begin{subfigure}[b]{0.48\textwidth}
    \centering
    \includegraphics[width=\linewidth]{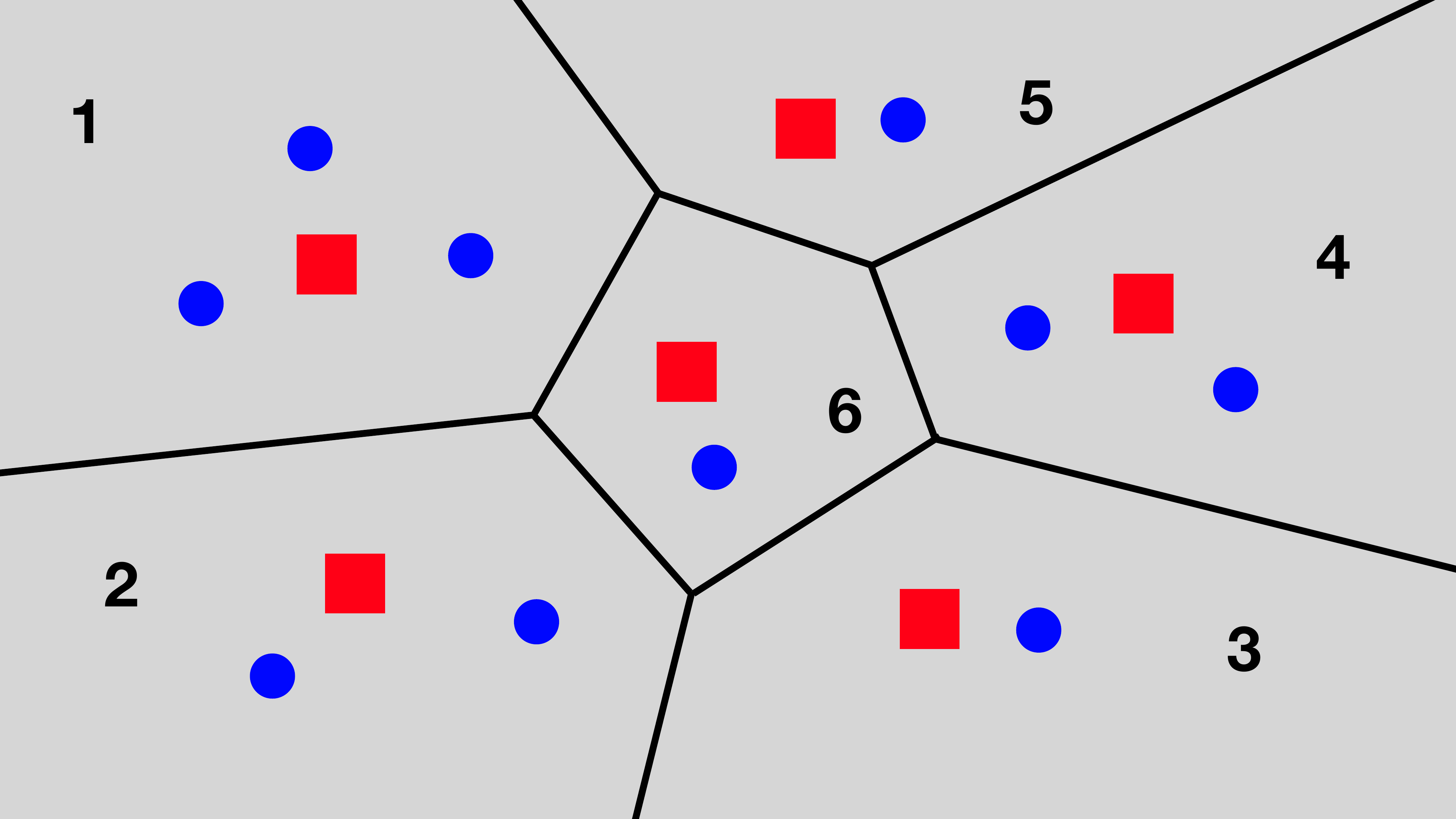}
    \caption{Voronoi cells \(\{\sA_k\}_{k\in[K_0]}, K_0 = 6, K = 10\) induced by \(G_0\) as in \cref{eq_voronoi_cell}. 
    Red squares are true atoms \(\{\vtheta_k^0\}\). Blue circles are fitted atoms \(\{\vtheta_\ell\}\). 
    The cardinality \(|\sA_k|\) equals the number of fitted atoms approximating the true atom in that cell.}
    \label{fig_voronoi_cells}
  \end{subfigure}
  \hfill
  \begin{subfigure}[b]{0.48\textwidth}
  \centering
  \tikzset{
    cell/.style       ={rounded corners, thick, draw=gray!45},
    trueatom/.style   ={rectangle, draw=red!70!black, fill=red!70, minimum size=6pt, inner sep=0pt},
    fittedatom/.style ={circle, draw=blue!65!black, fill=blue!60, minimum size=7pt, inner sep=0pt},
    mergedatom/.style ={circle, draw=purple!60!black, fill=purple!60, minimum size=8pt, inner sep=0pt},
    flow/.style       ={->, thick},
    weak/.style       ={dash pattern=on 2pt off 2pt, draw=gray!70},
    sel/.style        ={<->, very thick, draw=orange!80!black},
    lbl/.style        ={font=\scriptsize, inner sep=1pt}
  }
  \begin{tikzpicture}[>=latex,scale=1]
    \draw[cell] (-4.5,-2) rectangle (4.5,2);

    \node[trueatom, label={[lbl]above:{true}}] (true) at (0,0.95) {};

    \node[fittedatom, label={[lbl]below:{fitted i}}] (i) at (-1.30,-0.10) {};
    \node[fittedatom, label={[lbl]below:{fitted j}}] (j) at ( 1.30,-0.20) {};
    \node[fittedatom, label={[lbl]right:{fitted p}}] (p) at (-2.10, 0.55) {};
    \node[fittedatom, label={[lbl]left :{fitted q}}] (q) at ( 2.10, 0.55) {};

    \draw[sel] (i) -- node[below, lbl]{most similar} (j);

    \draw[weak] (i) -- (p);
    \draw[weak] (j) -- (q);

    \node[mergedatom, label={[lbl]right:{merged *}}] (star) at (0,0.35) {};
    \draw[flow] (i) -- (star);
    \draw[flow] (j) -- (star);
    \draw[flow] (star) -- (true);

    \begin{scope}[shift={(0,-1.35)}]
      \draw[rounded corners, draw=gray!60, fill=gray!5] (-2.9,-0.55) rectangle (2.9,0.65);
      \node[font=\scriptsize\bfseries] at (0,0.48) {Legend};

      \node[trueatom]   (Ltrue) at (-2.3,0.15) {};
      \node[lbl, anchor=west] at (-2.05,0.15) {true atom};

      \node[fittedatom] (Lfitt) at (-0.3,0.15) {};
      \node[lbl, anchor=west]  at (-0.05,0.15) {fitted atom};

      \node[mergedatom] (Lmerg) at (1.7,0.15) {};
      \node[lbl, anchor=west]  at (1.95,0.15) {merged};

      \draw[sel] (-1.2,-0.25) -- (-0.2,-0.25);
      \node[lbl, anchor=west] at (-0.05,-0.28) {pair selected to merge};
    \end{scope}
  \end{tikzpicture}
  \caption{Visual merge in a \textbf{multi-covered cell} \(|\sA_k|>1\). Among four fitted atoms, the closest pair (i, j) by a dissimilarity is merged first; repeating yields the aggregation path.}
  \label{fig_merge_schematic_only}
  \end{subfigure}

  \begin{subfigure}[b]{0.8\textwidth}
  \centering
  \fbox{%
    \parbox{0.92\linewidth}{\small
    \vspace{2pt}
    \textbf{Math key and merge equations.}  
    Visual labels \emph{i}, \emph{j}, and * correspond to
    \[
      \text{fitted i: } \exp(\omega_{0i})\delta_{\vtheta_i},\qquad
      \text{fitted j: } \exp(\omega_{0j})\delta_{\vtheta_j},\qquad
      \text{merged *: } \exp(\omega_{0*})\delta_{\vtheta_*}.
    \]
    Pair selection uses the rate-weighted dissimilarity \(\divclus\) in \cref{eq_dissim_sgmoe}.  
    The softmax-weighted merge (\cref{eq_merging_SGMoE_models}) is
    \[
      \omega_{0*}=\log\!\big(e^{\omega_{0i}}+e^{\omega_{0j}}\big),\quad
      \alpha_i=\frac{e^{\omega_{0i}}}{e^{\omega_{0i}}+e^{\omega_{0j}}},\;
      \alpha_j=\frac{e^{\omega_{0j}}}{e^{\omega_{0i}}+e^{\omega_{0j}}},
    \]
    \[
      \vomega_{1*}=\alpha_i\,\vomega_{1i}+\alpha_j\,\vomega_{1j},\qquad
      b_*=\alpha_i\,b_i+\alpha_j\,b_j,
    \]
    \[
      \va_*=\alpha_i\!\big[(\vomega_{1i}-\vomega_{1*})(b_i-b_*)+\va_i\big]
           +\alpha_j\!\big[(\vomega_{1j}-\vomega_{1*})(b_j-b_*)+\va_j\big],
    \]
    \[
      \sigma_*=\alpha_i\!\big[(b_i-b_*)^2+\sigma_i\big]
              +\alpha_j\!\big[(b_j-b_*)^2+\sigma_j\big].
    \]
    \vspace{1pt}
    }%
  }
  \caption{Mathematical notation and closed-form merge in \cref{subsec_merge_operator}.}
  \label{fig_merge_equations_box}
\end{subfigure}

\caption{Voronoi geometry and merge step for SGMoE. Multi-covered cells \(|\sA_k|>1\) signal redundant fitted atoms. The merge operator collapses them to a single aggregate that aligns with the true atom and improves the rate as formalized by our pathwise guarantees.}
\label{fig_voronoi_merge_combo}
\end{figure*}

\section{THEORETICAL CHALLENGES: MORE DETAILS}\label{sec:appendix:challenges}
The geometric picture above motivates the analytic tools below. We now detail three fundamental challenges in the statistical analysis of SGMoE that create substantial obstacles for parameter estimation and model selection:

(i) \emph{Softmax translation invariance.} Gating parameters are identifiable only up to common translations. Unlike covariate-independent gating functions, the softmax gate is invariant under simultaneous shifts of intercepts and slopes, which makes the parameterization non-unique. As a result, standard identifiability arguments break down, and it becomes necessary to design translation-invariant loss functions. We address this by introducing the Voronoi partition and loss (see \cref{eq_def_new_voronoi_D}), which takes an infimum over translations and thereby aligns the loss with the geometry of gating partitions.

(ii) \emph{Gate-expert PDE couplings.} The likelihood function exhibits intrinsic gate-expert interactions that induce coupled differential relations among parameters. These relations lead to numerous linear dependencies among derivative terms in Taylor expansions, which prevents a direct decomposition of density discrepancies \(p_{\widehat{G}_N} (y|\vx) - p_{G_0} (y|\vx)\) into independent components. Moreover, the parameters of the softmax gating numerators and the Gaussian experts are intrinsically linked through explicit PDEs,
\begin{equation}
    \frac{\partial^2 u}{\partial\vomega_1 \partial b} = \frac{\partial u}{\partial \va}, \qquad \frac{\partial^2 u}{\partial b^2} = 2 \frac{\partial u}{\partial \sigma}, \label{eq_pde}
\end{equation}
where \(u (y | \vx; \vomega, \va, b, \sigma) := \exp (\vomega_1^{\top} \vx)\,\cN (y| \va^{\top} \vx + b, \sigma)\).
Our analysis requires a systematic reorganization of these dependent terms to recover a meaningful set of independent directions.

(iii) \emph{Algebraic cancellations.} Due to the tight coupling between numerators and denominators in the softmax-induced conditional density, higher-order cancellations in the expansions give rise to systems of polynomial equations introduced in \cref{eq_system_of_polynomial_recall}. The solvability of these systems determines the order of the first non-vanishing terms and directly controls the convergence rates of the MLE in over-specified models. This algebraic obstruction is a key source of non-standard, slower rates unique to SGMoE.

These challenges indicate that previously used loss functions, such as the Wasserstein distance, are insufficient for analyzing parameter quantities in either standard mixture models or mixtures with covariate-free gating functions. Moreover, the convergence rates of parameter estimates, as reported in \cite{nguyen_demystifying_2023}, remain relatively slow due to the influence of the associated polynomial systems. Therefore, developing a dedicated method or algorithm, such as our DSC approach in \cref{sec_rate_aware_sgmoe}, for models of this type is well motivated.

In addition, we also clarify the relationship between polynomial equations in \cref{eq_system_of_polynomial_recall} and SGMoE in the over - specified case. Following \cref{thm:path_rates}, at $\kappa = K$, we can see that the convergence rate of Fast-Rate-Aware Voronoi distance $\VDFRA$ is $$\VDFRA (\widehat{G}_N, G_0) \lesssim \left( \frac{\log N}{N} \right)^{1/2},$$ this is an "optimal" rate for a mixing measure. However, the convergence rate of some parameters such as $\vomega_1, \va, b, \sigma$ are not $\left( \frac{\log N}{N} \right)^{1/2}$ (following the definition of $\VDFRA$). In particular, in the over-specified case, respectively $|\sA_k| \geq 2$, then the associated parameters suffer slower rates of the order $N^{-1/(2\bar r(|\sA_k|))}$ or $N^{-1/\bar r(|\sA_k|)}$ (see \cref{table_parameter_rates}). 

To explain this connection, we revisit our proof for over-specified case. Firstly, we want to show that $\E_{\rvx}\left[V\left(p_{G}(\cdot \mid \vx), p_{G_{0}}(\cdot \mid \vx)\right)\right] \gtrsim \VDFRA\left(G, G_{0}\right)$ because we can see that if we obtain this argument, we will get the "optimal" convergence rate of $\VDFRA$. We can rewrite the quantity $Q_N$ as follows:
\begin{align*}
Q_{N} & =\sum_{k=1}^{K_0} \sum_{\ell \in \sA_k} \exp \left(\omega_{0 \ell}^N\right) \bigg[u\left(y|\vx ; \vomega_{1 \ell}^N, \va_{\ell}^N, b_{\ell}^N, \sigma_{\ell}^N\right)-u\left(y|\vx ; \vomega_{1k}^{0}, \va_{k}^{0}, b_{k}^{0}, \sigma_{k}^{0}\right)-v\left(y|\vx ; \vomega_{1 \ell}^N\right) \\
&\quad +v\left(y|\vx ; \vomega_{1k}^{0}\right)\bigg]+\sum_{k=1}^{K_0}\left(\sum_{\ell \in \sA_k} \exp \left(\omega_{0 \ell}^N\right)-\exp \left(\omega_{0k}^{0}\right)\right)\left[u\left(y|\vx ; \omega_{0k}^{0}, \va_{k}^{0}, b_{k}^{0}, \sigma_{k}^{0}\right)-v\left(y|\vx ; \vomega_{1k}^{0}\right)\right],
\end{align*}

where we define $u\left(y|\vx ; \vomega_{1}, \va, b, \sigma\right):=\exp \left(\vomega_{1}^{\top} \vx\right) \cN\left(y| \va^{\top} \vx+b, \sigma\right)$ and $v\left(y|\vx ; \vomega_{1}\right):= \exp \left(\vomega_{1}^{\top} \vx\right) p_{G_{N}}(y|\vx)$. Next, for each $k \in\left[K_0\right]$ and $\ell \in \sA_k$, we denote $h_{1}\left(\vx, \va_{k}^{0}, b_{k}^{0}\right):=\left(\va_{k}^{0}\right)^{\top} \vx+b_{k}^{0}$ and then apply the Taylor expansions to the functions $u\left(y|\vx ; \vomega_{1 \ell}^N, \va_{\ell}^N, b_{\ell}^N, \sigma_{\ell}^N\right)$ and $v\left(y|\vx ; \vomega_{1 \ell}^N\right)$ up to orders $r_{1 k}$ and $r_{2 k}$ (which we will choose later), respectively, as follows:
\begin{align*}
& u\left(y|\vx ; \vomega_{1 \ell}^N, \va_{\ell}^N, b_{\ell}^N, \sigma_{\ell}^N\right)-u\left(y|\vx ; \vomega_{1k}^{0}, \va_{k}^{0}, b_{k}^{0}, \sigma_{k}^{0}\right) \\
& =\sum_{\left|\ell_{1}\right|+\ell_{2}=1}^{2 r_{1 k}} T_{\ell_{1}, \ell_{2}}^N(k) \vx^{\ell_{1}} \exp \left((\vomega_{1k}^{0})^{\top} \vx\right) \frac{\partial^{\ell_{2}} \cN}{\partial h_{1}^{\ell_{2}}}\left(y|\left(\va_{k}^{0}\right)^{\top} \vx+b_{k}^{0}, \sigma_{k}^{0}\right)+R_{1 \ell k}(\vx, y), \\
& v\left(y|\vx ; \vomega_{1 \ell}^N\right)-v\left(y|\vx ; \vomega_{1k}^{0}\right)=\sum_{|\gamma|=1}^{r_{2 k}} S_{\gamma}^N(k) \vx^{\gamma} \exp \left(\left(\vomega_{1k}^{0}\right)^{\top} \vx\right) p_{G_{N}}(y|\vx)+R_{2 \ell k}(\vx, y),
\end{align*}

where $R_{1 \ell k}(\vx, y)$ and $R_{2 \ell k}(\vx, y)$ are Taylor remainders such that $R_{\rho \ell k}(\vx, y) / \VDFRA\left(G_{N}, G_{0}\right)$ vanishes as $N \rightarrow \infty$ for $\rho \in\{1,2\}$. As a result, the limit of $Q_{N} / \VDFRA\left(G_{N}, G_{0}\right)$ when $n$ goes to infinity can be seen as a linear combination of elements of the following set:

$$
\begin{aligned}
\mathcal{W} & :=\left\{\vx^{\ell_{1}} \exp \left(\left(\vomega_{1k}^{0}\right)^{\top} \vx\right) \frac{\partial^{\ell_{2}} \cN}{\partial h_{1}^{\ell_{2}}}\left(y|\left(\va_{k}^{0}\right)^{\top} \vx+b_{k}^{0}, \sigma_{k}^{0}\right): k \in\left[K_0\right], 0 \leq 2\left|\ell_{1}\right|+\ell_{2} \leq 2 r_{1 k}\right\} \\
& \cup\left\{\vx^{\gamma} \exp \left((\vomega_{1k}^{0})^{\top} \vx\right) p_{G_{0}}(y|\vx): k \in\left[K_0\right], 0 \leq|\gamma| \leq r_{2 k}\right\}
\end{aligned}
$$

which is shown to be linearly independent. By the Fatou's lemma, we demonstrate that $Q_{N} / \VDFRA\left(G_{N}, G_{0}\right)$ goes to zero as $N \rightarrow \infty$, implying that all the coefficients in the representation of $Q_{N} / \VDFRA\left(G_{N}, G_{0}\right)$, denoted by $T_{\ell_{1}, \ell_{2}}^N(k) / \VDFRA\left(G_{N}, G_{0}\right)$ and $S_{\gamma}^N(k) / \VDFRA\left(G_{N}, G_{0}\right)$, vanish when $N \rightarrow \infty$. Given that result, we aim to select the Taylor orders $r_{1 k}$ and $r_{2 k}$ such that at least one among the limits of $T_{\ell_{1}, \ell_{2}}^N(k) / \VDFRA\left(G_{N}, G_{0}\right)$ and $S_{\gamma}^N(k) / \VDFRA\left(G_{N}, G_{0}\right)$ is different from zero, which leads to a contradiction. In the over-specified case, we assume that all the limits of $T_{\ell_{1}, \ell_{2}}^N(k) / \VDFRA\left(G_{N}, G_{0}\right)$ and $S_{\gamma}^N(k) / \VDFRA\left(G_{N}, G_{0}\right)$ equal zero. After some steps of considering typical limits as in the previous setting which requires $r_{2 k}=2$ for all $k \in\left[K_{0}\right]$, we encounter the following system of polynomial equations:

$$
\sum_{\ell \in \sA_{k}} \sum_{\left(\valpha_{1}, \valpha_{2}, \alpha_{3}, \alpha_{4}\right) \in \sI_{\ell_{1}, \ell_{2}}} \frac{p_{5 \ell}^{2} p_{1 \ell}^{\valpha_{1}} p_{2 \ell}^{\valpha_{2}} p_{3 \ell}^{\alpha_{3}} p_{4 \ell}^{\alpha_{4}}}{\valpha_{1}!\valpha_{2}!\alpha_{3}!\alpha_{4}!}=0
$$

for all $\left(\ell_{1}, \ell_{2}\right) \in \mathbb{N}^{D} \times \mathbb{N}$ such that $0 \leq\left|\ell_{1}\right| \leq r_{1 k}, 0 \leq \ell_{2} \leq r_{1 k}-\left|\ell_{1}\right|$ and $\left|\ell_{1}\right|+\ell_{2} \geq 1$ for some $k \in\left[K_{0}\right]$. Due to the construction of this system, it must have at least one non-trivial solution. Therefore, we choose $r_{1 k} = \bar{r} (|\sA_k|)$ for all $k \in [K_0]$.

To discuss about the value of $\bar{r} (M)$ with $M \ge 2$ in general, by \cref{lemma:value_r_unified}, we obtain $\bar{r} (M) = 2M$ for $M = 2, 3$ and as $M$ increases, so does $\bar{r} (M)$. Hence, we predict that $\bar{r} (M) = 2M$. With this conjecture, we can see that the slow convergence rate of parameter estimation of SGMoE before we apply the merging atoms process.

\section{PROOF SKETCHES}\label{sec:appendix:proof-sketches}

In this section we expand the sketches for
\cref{lem:monotone_path}, \cref{thm:path_rates,thm:heights,thm:likelihood_path}.

\paragraph{Why the \texorpdfstring{\(\VDFRA\)}{VDFRA} loss in \cref{eq_def_new_voronoi_D}?}
When \(\widehat G_N \to G_0\) with \(K > K_0\), some Voronoi cells \(\sA_k\) are multi-covered. The slow directions in \(\VDO\) (with exponents \(\bar r(|\sA_k|)\)) arise from these cells. \(\VDFRA\) augments \(\VDO\) with first-order \emph{merged-moment} block-sums that vanish when a cell behaves as a single aggregate. Thus \(\VDFRA\) is simultaneously (i) exact-fit consistent, it reduces to \(\VDE\) when \(|\sA_k|=1\), and (ii) overfit-aware, penalizing precisely the slow directions that merging removes. In the over-specified case, cells with \(|\sA_k|>1\) may persist; repeatedly merging atoms within such cells yields singletons and restores first-order behavior. Formally, using the density decomposition
\[
Q_{N}=\Big[\sum_{k=1}^{K_{0}} \exp \big((\vomega_{1 k}^{0}+\vt_1)^{\top} \vx+\omega_{0 k}^{0}+t_{0}\big)\Big]\cdot\big[p_{G_{N}}(y|\vx)-p_{G_{0}}(y|\vx)\big],
\]
we analyze the sums over indices with \(|\sA_k|>1\) under \(1 \le |\vell_1|+\ell_2 \le 2\bar r(|\sA_k|)\); for clarity, we also isolate the case \(1 \le |\vell_1|+\ell_2 \le 2\), corresponding to \(|\sA_k|=1\). This leads directly to the merge operator and the aggregation path.

\subsection{Proof Sketch of Lemma~\ref{lem:monotone_path}}

We argue for the first merge \(G^{(K)}\to G^{(K-1)}\); the rest follows by induction.
Assume \(\VDFRA(G^{(K)},G_0)\to0\).
Then, for the Voronoi partition \(\{\sA_k\}\), there exist \((t_0,\vt_1)\) such that, for every \(k\),
\[
\sum_{\ell\in\sA_k}\exp(\omega_{0\ell})\to \exp(\omega_{0k}^0+t_0),\qquad
(\vomega_{1\ell},\va_\ell,b_\ell,\sigma_\ell)\to (\vomega_{1k}^0+\vt_1,\va_k^0,b_k^0,\sigma_k^0).
\]
The minimizing pair \((i,j)\) of \(\divclus\) must lie in the same cell \(\sA_k\).
Let the merged atom be \(\exp(\omega_{0*})\delta_{(\vomega_{1*},\va_*,b_*,\sigma_*)}\) as in \cref{eq_merging_SGMoE_models}.
Using the convexity of \(z\mapsto\|z\|^{m}\) for \(m\in\{\bar r(|\sA_k|),\bar r(|\sA_k|)/2\}\) and the identities implicit in \cref{eq_merging_SGMoE_models}, we obtain the two key comparisons
\begin{align*}
(\exp\omega_{0i}+\exp\omega_{0j})\|(\Delta_{\vt_1}\vomega_{1*k},\Delta b_{*k})\|^{\bar r(|\sA_k|)}
&\;\lesssim\; \sum_{t\in\{i,j\}}\exp\omega_{0t}\|(\Delta_{\vt_1}\vomega_{1tk},\Delta b_{tk})\|^{\bar r(|\sA_k|)},\\
(\exp\omega_{0i}+\exp\omega_{0j})\|(\Delta \va_{*k},\Delta \sigma_{*k})\|^{\bar r(|\sA_k|)/2}
&\;\lesssim\; \sum_{t\in\{i,j\}}\exp\omega_{0t}\|(\Delta \va_{tk},\Delta \sigma_{tk})\|^{\bar r(|\sA_k|)/2}.
\end{align*}
The block-sum terms in \(\VDFRA\) also decrease since the merged parameters are softmax-weighted averages.
Collecting terms yields
\(\VDFRA(G^{(K)},G_0)\gtrsim \VDFRA(G^{(K-1)},G_0)\), proving monotonicity.

\subsection{Proof Sketch of Theorem~\ref{thm:path_rates}}

\paragraph{(A) Inverse bound.}
We first prove an inverse inequality:
there exists \(C>0\) depending only on \(G_0\) and \(\vTheta\) such that, for any \(G\in\cO_K(\vTheta)\),
\begin{equation}\label{eq:inverse_bound_vdfra}
\bbE_\rvx\big[\TV(p_G(\cdot|\vx),p_{G_0}(\cdot|\vx))\big]\;\ge\; C\,\VDFRA(G,G_0).
\end{equation}
The proof follows the \emph{density decomposition} strategy in \citet{nguyen_demystifying_2023} but keeps all merged-moment block-sums that define \(\VDFRA\).
Let
\[
Q_N(\vx,y)=\Big[\textstyle\sum_{k=1}^{K_0}\exp\big((\vomega^0_{1k}+\vt_1)^\top \vx+\omega^0_{0k}+t_0\big)\Big]\cdot\big[p_G(y|\vx)-p_{G_0}(y|\vx)\big].
\]
A multi-index Taylor expansion (around \((\vomega_{1k}^0+\vt_1,\va_k^0,b_k^0,\sigma_k^0)\) within each cell \(\sA_k\)) up to order \(\bar r(|\sA_k|)\), together with the PDE identities
\(
\partial^2 u / \partial\vomega_1\partial b=\partial u/\partial\va\) and
\(\partial^2 u/\partial b^2 = 2\,\partial u/\partial\sigma
\),
rewrites \(Q_N\) as a linear combination of basis functions
\[
\vx^{\vell_1}\exp\!\big((\vomega_{1k}^0+\vt_1)^\top\vx\big)\,\frac{\partial^{\ell_2}}{\partial h_1^{\ell_2}}
\cN\!\big(y| \va_k^{0\top}\vx+b_k^0,\sigma_k^0\big),
\quad 1\le |\vell_1|+\ell_2\le2\bar r(|\sA_k|),
\]
with coefficients that are precisely the atomwise sums appearing in \(\VDFRA\) (up to constants).
If \cref{eq:inverse_bound_vdfra} failed, all these coefficients would have to vanish at a rate faster than \(\VDFRA(G,G_0)\), forcing a non-trivial solution to the polynomial system of \cref{eq_system_of_polynomial_recall}, in contradiction with the definition of \(\bar r(\cdot)\)
(\cref{lemma:value_r_unified}).
This yields \cref{eq:inverse_bound_vdfra}.

\paragraph{(B) Applying density rates.}
By Proposition~2 of \citet{nguyen_demystifying_2023},
\(
\bbE_\rvx[\hels(p_{\widehat G_N}(\cdot|\vx),p_{G_0}(\cdot|\vx))]
=\cO_\sP((\log N/N)^{1/2}).
\)
Using the inequality \(\TV\le \sqrt{2}\, \hels^{1/2}\) and \cref{eq:inverse_bound_vdfra} with \(G=\widehat G_N\), we obtain
\[
\VDFRA(\widehat G_N,G_0)=\cO_\sP\big((\log N/N)^{1/2}\big).
\]
Now apply Lemma~\ref{lem:monotone_path} along the aggregation path:
for every \(\kappa\in[K_0+1,K]\),
\[
\VDFRA(\widehat G_N^{(\kappa)},G_0)\;\lesssim\;\VDFRA(\widehat G_N,G_0)
=\cO_\sP\!\big((\log N/N)^{1/2}\big).
\]
For the exact-fit and under-fit levels \(\kappa'\le K_0\),
\(\VDFRA=\VDE\) by definition, which gives the second claim.

\subsection{Proof Sketch of Theorem~\ref{thm:heights}}

For \(\kappa\in[K_0+1,K]\), the height \(\height_N^{(\kappa)}\) is the minimum \(\divclus\)-distance between any two atoms of \(\widehat G_N^{(\kappa)}\).
Inside a multi-covered cell \(\sA_k(\widehat G_N)\), the Taylor/merged-moment analysis from the proof of \cref{eq:inverse_bound_vdfra} implies that
\[
\divclus\!\big(\exp(\widehat\omega_{0i})\delta_{\widehat\vtheta_i},
\exp(\widehat\omega_{0j})\delta_{\widehat\vtheta_j}\big)\;\lesssim\;
\big\|(\Delta_{\vt_1}\widehat\vomega_{1ik},\Delta \widehat b_{ik})\big\|^2
+\big\|(\Delta \widehat\va_{ik},\Delta \widehat\sigma_{ik})\big\|.
\]
The right-hand side is controlled by \(\VDFRA(\widehat G_N^{(\kappa)},G_0)\) with the exponents \(\bar r(|\sA_k|)\), hence
\(
\height_N^{(\kappa)}\lesssim (\log N/N)^{1/\bar r(\widehat G_N)}
\).
For \(\kappa'\le K_0\), heights converge at parametric rate because atoms are separated and \(\VDE(\widehat G_N^{(\kappa')},G_0^{(\kappa')})=\cO_\sP((\log N/N)^{1/2})\).

\subsection{Proof Sketch of Theorem~\ref{thm:likelihood_path}}

Let \(\bar\ell_N(p_G)=N^{-1}\sum_{n=1}^N\log p_G(y_n| \rvx_n)\) and \(\cL(p_G)=\E_{(\rvx,y)\sim P_{G_0}}[\log p_G(y| \vx)]\).
Under Condition~K, a local Lipschitz/curvature argument yields
\[
\big|\bar\ell_N(p_{G})-\cL(p_{G})\big|\;\lesssim\;
\E_{(\rvx,y)\sim P_{G_0}}\big[\TV(p_G(\cdot|\vx),p_{G_0}(\cdot|\vx))\big]
+\text{empirical fluctuation}.
\]
For \(\kappa\ge K_0\), combine the inverse bound
\(\E[\TV]\gtrsim \VDFRA\) with \(\VDFRA(\widehat G_N^{(\kappa)},G_0)=\cO_\sP((\log N/N)^{1/2})\) and standard empirical-process bounds (e.g., \citealp{van_de_geer_empirical_2000}) to obtain
\(
|\bar\ell_N(p_{\widehat G_N^{(\kappa)}})-\cL(p_{G_0})|\lesssim (\log N/N)^{1/(2\bar r(\widehat G_N))}
\).
For \(\kappa'\le K_0\), \(\widehat G_N^{(\kappa')}\) is exact/under-fit and converges at parametric rate, hence
\(\bar\ell_N(p_{\widehat G_N^{(\kappa')}})\to \cL(p_{G_0^{(\kappa')}})\) in probability.

\section{PROOF OF MAIN RESULTS}\label{sec:appendix:proofs-main}

Before proving the main results, we fix notation used throughout this appendix. For any natural number \(N\in\Ns\), write \([N]:=\{1,2,\ldots,N\}\).
Given two sequences of positive real numbers \(\{a_N\}_{N=1}^\infty\) and \(\{b_N\}_{N=1}^\infty\), we write \(a_N=\cO(b_N)\) (equivalently, \(a_N\lesssim b_N\)) to mean that there exists a constant \(C>0\) such that \(a_N\le C\,b_N\) for all \(N\in\Ns\).
For a vector \(\vv\in\R^D\) and any multi-index \(\vp \in \Ns^D\), set \(|\vp|:=p_1+\cdots+p_D\), \(\vv^{\vp} := v_1^{p_1} v_2^{p_2} \cdots v_D^{p_D}\), \(\vp! := p_1! p_2! \cdots p_D!\), and let \(\|\vv\|_{p}\) denote its \(p\)-norm; by default, \(\|\vv\|\) refers to the \(2\)-norm unless otherwise stated.
We also use \(\|\mA\|\) for the Frobenius norm of a matrix \(\mA\in\R^{D\times D}\).
For any set \(\sS\), \(|\sS|\) denotes its cardinality.
For two probability density functions \(p\) and \(q\) with respect to the Lebesgue measure \(\mu\), define \(\TV(p,q) := \tfrac 1 2 \int |p-q|\,\mathrm d\mu\) as their total variation distance, while \(\hels(p,q) := \tfrac 1 2 \int (\sqrt p - \sqrt q)^2\,\mathrm d\mu\) denotes the squared Hellinger distance. Moreover, for \(\vmu \in \R^D\), \(\valpha \in \Ns^D\), and a differentiable function \(f\) of \(\vmu\), we write the partial derivative of order \(|\valpha|\) as
\[
\frac{\partial^{|\valpha|}}{\partial^{\valpha} \vmu} f(\vmu) := \frac{\partial^{|\valpha|}}{\partial \mu_1^{\valpha_1} \partial \mu_2^{\valpha_2} \cdots \partial \mu_D^{\valpha_D}} f(\vmu).
\]

Let \(\vTheta\) be the parameter space.
Write \(\cE_K(\vTheta)\) for the collection of discrete probability measures on \(\vTheta\) with exactly \(K\) atoms, and \(\cO_K(\vTheta):=\bigcup_{K'\le K}\cE_{K'}(\vTheta)\) for those with at most \(K\) atoms.
For a mixing measure \(G=\sum_{k=1}^{K}\pi_k\,\delta_{\vtheta_k}\), we (slightly abusively) refer to each component \(\pi_k\,\delta_{\vtheta_k}\) as an “atom,” comprising both its weight \(\pi_k\) and parameter \(\vtheta_k\). Finally, the domain of parameters in the SGMoE is \(\vTheta\), where \(\veta_{k}^{0}:=(\omega_{0k}^{0}, \vomega_{1k}^{0}, \va_{k}^{0}, b_{k}^{0}, \sigma_{k}^{0})
\in \vTheta \subset \R \times \R^{D} \times \R^{D} \times \R \times \R_{>0}\). Furthermore, assume \(\vTheta\) is compact and \(\cX\subset\R^{D}\), the support of \(\rvx\), is bounded. When clear from context, we drop \(\vTheta\) and simply write \(\cE_K\) and \(\cO_K\).

\subsection{Proof of Lemma~\ref{lem:monotone_path}}

     We prove the inequality $\VDFRA (G^{(K)}, G_0) \gtrsim \VDFRA (G^{(K - 1)}, G_0)$, and the rest are similar.

     Assume that $G_N := G_N^{(K)} = \sum_{k = 1}^{K}  \exp(\omega_{0k}^{N}) \delta_{(\vomega_{1k}^{N}, \va^N_k, b^N_k, \sigma^N_k)} \in \cE_{K}$ varies so that $\VDFRA (G_N, G_0) \to 0$. We consider the Voronoi cells $\sA_k^N := \sA_k(G_N)$, for $k \in [K_0]$, of the mixing measure $G_N$ generated by the true components of $G_0$. Since the argument in this proof is asymptotic, we assume without loss of generality that those Voronoi cells are independent of $N$ for all $N \in \mathbb{N}$, i.e, $\sA_k = \sA^N_k$.
     
     Then, we have $(\va_{\ell}^N, b_{\ell}^N, \sigma_{\ell}^N) \to (\va_k^0, b_k^0, \sigma_k^0)$, and there exist $t_0 \in \mathbb{R}$ and $\vt_1 \in \mathbb{R}^D$ such that $\sum\limits_{\ell \in \sA_k^N} \exp{(\omega_{0 \ell}^N)} \to \exp{(\omega_{0k}^0 + t_0)}$ and $\vomega_{1\ell}^N \to \vomega_{1k}^0 + \vt_1$ for any $\ell \in \sA_k$ and $k \in [K_0]$ as $N$ approaches infinity. 
     
     We are going to show that the merging pair of indices $(\ell_1, \ell_2)$ must belong to a common $\sA_k$. 
Indeed, for every pair $(\ell_1, \ell_2)$ in a common $\sA_k$, 
since $(\va_{\ell_1}^N, b_{\ell_1}^N, \sigma_{\ell_1}^N) \to (\va_k^0, b_k^0, \sigma_k^0)$ and 
$\vomega_{1\ell_1}^N \to \vomega_{1k}^0 + \vt_1$, 
and $(\va_{\ell_2}^N, b_{\ell_2}^N, \sigma_{\ell_2}^N) \to (\va_k^0, b_k^0, \sigma_k^0)$ and 
$\vomega_{1\ell_2}^N \to \vomega_{1k}^0 + \vt_1$, 
we have 
$$
\divclus \left( 
   \exp(\omega_{0\ell_1}^{N}) \delta_{(\vomega_{1\ell_1}^{N}, \va_{\ell_1}^N, b_{\ell_1}^N, \sigma_{\ell_1}^N)}, \,
   \exp(\omega_{0\ell_2}^{N}) \delta_{(\vomega_{1\ell_2}^{N}, \va_{\ell_2}^N, b_{\ell_2}^N, \sigma_{\ell_2}^N)} 
\right) \to 0, 
\quad \text{as $N \to \infty$}.
$$

     On the other hand, for every pair $(\ell, \ell') \in \sA_k \times \sA_{k'}$, where $k \ne k'$, 
because $(\va_{\ell}^N, b_{\ell}^N, \sigma_{\ell}^N) \to (\va_k^0, b_k^0, \sigma_k^0)$ and 
$\vomega_{1\ell}^N \to \vomega_{1k}^0 + \vt_1$, 
and $(\va_{\ell'}^N, b_{\ell'}^N, \sigma_{\ell'}^N) \to (\va_{k'}^0, b_{k'}^0, \sigma_{k'}^0)$ and 
$\vomega_{1\ell'}^N \to \vomega_{1k'}^0 + \vt_1$, 
we have 
$$
\divclus \left( 
   \exp(\omega_{0\ell}^{N}) \delta_{(\vomega_{1\ell}^{N}, \va_{\ell}^N, b_{\ell}^N, \sigma_{\ell}^N)}, \,
   \exp(\omega_{0\ell'}^{N}) \delta_{(\vomega_{1\ell'}^{N}, \va_{\ell'}^N, b_{\ell'}^N, \sigma_{\ell'}^N)} 
\right) \gtrsim 
\norm{(\vomega^0_{1k}, b^0_{k}) - (\vomega^0_{1k'}, b^0_{k'})}^2 
+ \norm{(\va^0_{k}, \sigma^0_{k}) - (\va^0_{k'}, \sigma^0_{k'})},
$$
where the multiplicative constant is not dependent on $N$. 
Hence, the merging pair must belong to a common $\sA_k$.

Next, for any $(t_{0}, \vt_{1}) \in \mathbb{R} \times \mathbb{R}^{D}$ such that $\omega_{0k}^{0} + t_{0}$ and $\vomega_{1k}^{0} + \vt_{1}$ still lie inside the domain of the parameter space $\vTheta$, we define $ \cD (G_N, G_0, t_0, \vt_1)$ as
\begin{align}
    &\cD(G_N, G_{0}, t_0, \vt_1) : =  \sum_{k: |\sA_{k}| > 1} \sum_{\ell \in \sA_{k}} \exp(\omega_{0\ell}^N) 
    \Big(\|(\Delta_{\vt_{1}} \vomega_{1\ell k}^N,  \Delta b_{\ell k}^N)\|^{\bar{r}_k} 
    + \|(\Delta \va_{\ell k}^N,\Delta \sigma_{\ell k}^N)\|^{\bar{r}_k/2}\Big) \nonumber \\
    &\quad + \sum_{k: |\sA_{k}| = 1} \sum_{\ell \in \sA_{k}} \exp(\omega_{0\ell}^N) 
    \|(\Delta_{\vt_{1}} \vomega_{1\ell k}^N, \Delta \va_{\ell k}^N,  \Delta b_{\ell k}^N, \Delta \sigma_{\ell k}^N)\| + \sum_{k = 1}^{K_{0}} \Big|\sum_{\ell \in \sA_{k}} \exp(\omega_{0\ell}^N) - \exp(\omega_{0k}^{0} + t_{0})\Big| \nonumber\\
    &\quad + \sum_{k: |\sA_{k}| > 1} \bigg( 
    \norm{\sum_{\ell \in \sA_k} \exp{(\omega_{0\ell}^N)} (\Delta b_{\ell k}^N)} 
    + \norm{\sum_{\ell \in \sA_k} \exp{(\omega_{0\ell}^N)} (\Delta_{\vt_1} \vomega_{1\ell k}^N)} + \norm{\sum_{\ell \in \sA_k} \exp{(\omega_{0\ell}^N)}[(\Delta b_{\ell k}^N)^2 + (\Delta \sigma_{\ell k}^N)]} \nonumber \\
    &\qquad + \norm{\sum_{\ell \in \sA_k} \exp{(\omega_{0\ell}^N)}[(\Delta_{\vt_1} \vomega_{1\ell k}^N)(\Delta b_{\ell k}^N) + (\Delta \va_{\ell k}^N)]} \nonumber +  \norm{\sum_{\ell \in \sA_k} \exp{(\omega_{0\ell}^N)}(\Delta_{\vt_1} \vomega_{1\ell k}^N)(\Delta_{\vt_1} \vomega_{1\ell k}^N)^{\top}} 
    \bigg), \nonumber \label{eq_def_voronoi_D2'}
\end{align}
in which 
$\Delta_{\vt_{1}} \vomega_{1\ell k}^N : = \vomega_{1\ell}^N - \vomega_{1k}^{0} - \vt_{1}$, 
$\Delta \va_{\ell k}^N : = \va_{\ell}^N - \va_{k}^{0}$, 
$\Delta b_{\ell k}^N : = b_{\ell}^N - b_{k}^{0}$, 
$\Delta \sigma_{\ell k}^N : = \sigma_{\ell}^N - \sigma_{k}^{0}$, and $\bar{r}_k := \bar{r}\left(\sA_k (\widehat{G}_N)\right)$.

     We prove that $\cD (G^{(K-1)}_N, G_0, t_0, \vt_1) \lesssim \cD (G^{(K)}_N, G_0, t_0, \vt_1)$. Let the merging pair of indices $(\ell_1, \ell_2)$ in the Voronoi cell $\sA_k$, then $|\sA_k| > 1$ and the merged atom is $\exp{(\omega_{0*}^N)} \delta_{(\vomega_{1*}^N, \va_*^N, b_*^N, \sigma_*^N)}$, i.e,
\begin{align}
    \omega_{0*}^N &= \log \left( \exp{\omega_{0\ell_1}^N} + \exp{\omega_{0\ell_2}^N} \right), \nonumber \\
    \vomega_{1*}^N &= \exp{\left(\omega_{0\ell_1}^N - \omega_{0*}^N \right)} \vomega_{1\ell_1}^N 
    + \exp{\left(\omega_{0\ell_2}^N - \omega_{0*}^N\right)} \vomega_{1\ell_2}^N, \nonumber\\
    b_*^N &= \exp{\left(\omega_{0\ell_1}^N - \omega_{0*}^N \right)} b_{\ell_1}^N 
    + \exp{\left(\omega_{0\ell_2}^N - \omega_{0*}^N\right)} b_{\ell_2}^N, \nonumber\\
    \va_*^N &= \frac{\exp(\omega_{0\ell_1}^N)}{\exp(\omega_{0*}^N)} \Big[(\vomega_{1\ell_1}^N - \vomega_{1*}^N)(b_{\ell_1}^N - b_*^N) + \va_{\ell_1}^N\Big] 
    + \frac{\exp(\omega_{0\ell_2}^N)}{\exp(\omega_{0*}^N)} \Big[(\vomega_{1\ell_2}^N - \vomega_{1*}^N)(b_{\ell_2}^N - b_*^N) + \va_{\ell_2}^N\Big], \nonumber \\
    \sigma_*^N &= \frac{\exp(\omega_{0\ell_1}^N)}{\exp(\omega_{0*}^N)} \Big[(b_{\ell_1}^N - b_*^N)^2 + \sigma_{\ell_1}^N\Big] 
    + \frac{\exp(\omega_{0\ell_2}^N)}{\exp(\omega_{0*}^N)} \Big[(b_{\ell_2}^N - b_*^N)^2 + \sigma_{\ell_2}^N\Big]. \nonumber
\end{align}

Hence, we have that 
\begin{align*}
    \Big|\sum_{\ell \in \sA_{k}} \exp(\omega_{0\ell}^N) - \exp(\omega_{0k}^{0} + t_{0})\Big| 
    &= \Big|\sum_{\ell \in \sA_{k}, \ell \not \in \{\ell_1, \ell_2\}} \exp(\omega_{0\ell}^N) + \exp{(\omega_{0*}^N)} - \exp(\omega_{0k}^{0} + t_{0})\Big|, \\
    \exp(\omega_{0*}^N)\, \Delta b_{*k}^N 
&= \exp(\omega_{0*}^N) \big( b_*^N - b_k^0 \big) \\
&= \exp(\omega_{0*}^N) \left( 
   \exp(\omega_{0\ell_1}^N - \omega_{0*}^N)\, b_{\ell_1}^N
 + \exp(\omega_{0\ell_2}^N - \omega_{0*}^N)\, b_{\ell_2}^N
 - b_k^0 \right) \\
&= \exp(\omega_{0\ell_1}^N)\, b_{\ell_1}^N 
 + \exp(\omega_{0\ell_2}^N)\, b_{\ell_2}^N 
 - \exp(\omega_{0*}^N)\, b_k^0 \\
 &= \exp(\omega_{0\ell_1}^N)(b_{\ell_1}^N - b_k^0)
 +  \exp(\omega_{0\ell_2}^N)(b_{\ell_2}^N - b_k^0) \\
 &=\exp(\omega_{0\ell_1}^N)\, \Delta b_{\ell_1 k}^N 
+ \exp(\omega_{0\ell_2}^N)\, \Delta b_{\ell_2 k}^N.
\end{align*}

It follows that the term
\begin{align*}
    \norm{\sum_{\ell \in \sA_k} \exp{(\omega_{0\ell}^N)} (\Delta b_{\ell k}^N)} = \norm{\sum_{\ell \in \sA_k \setminus \{\ell_1, \ell_2 \}} \exp{(\omega_{0\ell}^N)} (\Delta b_{\ell k}^N) +  \exp{(\omega_{0*}^N)} (\Delta b_{* k}^N)}.
\end{align*}

Similarly, we can show that 
\begin{align*}
    \norm{\sum_{\ell \in \sA_k} \exp{(\omega_{0\ell}^N)} (\Delta_{\vt_1} \vomega_{1\ell k}^N)} &= \norm{\sum_{\ell \in \sA_k \setminus \{\ell_1, \ell_2 \}} \exp{(\omega_{0\ell}^N)} (\Delta_{\vt_1} \vomega_{1\ell k}^N) +  \exp{(\omega_{0*}^N)} (\Delta_{\vt_1} \vomega_{1* k}^N)}, \\
    \norm{\sum_{\ell \in \sA_k} \exp{(\omega_{0\ell}^N)}[(\Delta b_{\ell k}^N)^2 + (\Delta \sigma_{\ell k}^N)]} &= \Big\|\sum_{\ell \in \sA_k \setminus \{\ell_1, \ell_2 \}} \exp{(\omega_{0\ell}^N)}[(\Delta b_{\ell k}^N)^2 + (\Delta \sigma_{\ell k}^N)] \\
    &\qquad+  \exp{(\omega_{0*}^N)}[(\Delta b_{* k}^N)^2 + (\Delta \sigma_{* k}^N)]\Big\|,\\
    \norm{\sum_{\ell \in \sA_k} \exp{(\omega_{0\ell}^N)}[(\Delta_{\vt_1} \vomega_{1\ell k}^N)(\Delta b_{\ell k}^N) + (\Delta \va_{\ell k}^N)]} &= \Big\|\sum_{\ell \in \sA_k \setminus \{\ell_1, \ell_2 \}} \exp{(\omega_{0\ell}^N)}[(\Delta_{\vt_1} \vomega_{1\ell k}^N)(\Delta b_{\ell k}^N) + (\Delta \va_{\ell k}^N)] \\
    &\qquad+  \exp{(\omega_{0*}^N)}[(\Delta_{\vt_1} \vomega_{1* k}^N)(\Delta b_{* k}^N) + (\Delta \va_{* k}^N)]\Big\|,\\
    \norm{\sum_{\ell \in \sA_k} \exp{(\omega_{0\ell}^N)}(\Delta_{\vt_1} \vomega_{1\ell k}^N)(\Delta_{\vt_1} \vomega_{1\ell k}^N)^{\top}} &= \Big\|\sum_{\ell \in \sA_k \setminus \{\ell_1, \ell_2 \}} \exp{(\omega_{0\ell}^N)}(\Delta_{\vt_1} \vomega_{1\ell k}^N)(\Delta_{\vt_1} \vomega_{1\ell k}^N)^{\top} \\
    &\qquad+  \exp{(\omega_{0*}^N)}(\Delta_{\vt_1} \vomega_{1* k}^N)(\Delta_{\vt_1} \vomega_{1* k}^N)^{\top}\Big\|.
\end{align*}

To this end, we show the key convexity step in detail. 
Firstly, we define $\alpha_i$ and $\vx_i$ ($i = 1, 2$) as follow:
\begin{align*}
    \alpha_i &:= \exp(\omega_{0\ell_i}^N-\omega_{0*}^N),\; i=1,2,\\
    \vx_i &:= \big(\vomega_{1\ell_i}^N-\vomega_{1k}^0-\vt_1,\; b_{\ell_i}^N-b_k^0\big)
= \big(\Delta_{\vt_1}\vomega_{1\ell_i k}^N,\; \Delta b_{\ell_i k}^N\big),\qquad i=1,2.
\end{align*}

Note that \(\valpha_1,\valpha_2\in(0,1)\) and \(\valpha_1+\valpha_2=1\) and by the definition of the merged atom \cref{eq_merging_SGMoE_models}, we have the convex combination identity $\big(\Delta_{\vt_1}\vomega_{1*k}^N,\; \Delta b_{*k}^N\big)
= \valpha_1 \vx_1 + \valpha_2 \vx_2$.

By the fact that $\bar{r} \geq 4$ (since $|\sA_k| > 1$), so we can use Jensen's inequality (convexity of the map \(z\mapsto\|z\|^{m}\) with \(m=\bar r_k\)):
\[
\big\|\valpha_1 \vx_1 + \valpha_2 \vx_2\big\|^{\bar r_k}
\le \valpha_1 \|\vx_1\|^{\bar r_k} + \valpha_2 \|\vx_2\|^{\bar r_k}.
\]
Multiply both sides by \(\exp(\omega_{0*}^N)\) and substitute \(\alpha_i=\exp(\omega_{0\ell_i}^N-\omega_{0*}^N)\):
\begin{align*}
\big(\exp(\omega_{0\ell_1}^N)+\exp(\omega_{0\ell_2}^N)\big)
\big\|\Delta_{\vt_1}\vomega_{1*k}^N,\Delta b_{*k}^N\big\|^{\bar r_k}
&= \exp(\omega_{0*}^N)\big\|\valpha_1 \vx_1+\valpha_2 \vx_2\big\|^{\bar r_k} \\
&\le \exp(\omega_{0*}^N)\big(\valpha_1\|\vx_1\|^{\bar r_k}+\valpha_2\|\vx_2\|^{\bar r_k}\big)\\
&= \exp(\omega_{0\ell_1}^N)\|\vx_1\|^{\bar r_k}+\exp(\omega_{0\ell_2}^N)\|\vx_2\|^{\bar r_k}\\
&= \exp(\omega_{0\ell_1}^N)\big\|(\Delta_{\vt_1}\vomega_{1\ell_1 k}^N,\Delta b_{\ell_1 k}^N)\big\|^{\bar r_k} +\exp(\omega_{0\ell_2}^N)\big\|(\Delta_{\vt_1}\vomega_{1\ell_2 k}^N,\Delta b_{\ell_2 k}^N)\big\|^{\bar r_k}.
\end{align*}

Analogously, we can show that 
\begin{align}
\exp(\omega_{0\ell_1}^N)\,\big\|(\Delta \va_{\ell_1 k}^N,\;\Delta \sigma_{\ell_1 k}^N)\big\|^{\bar r_k/2}
 + \exp(\omega_{0\ell_2}^N)\,\big\|(\Delta \va_{\ell_2 k}^N,\;\Delta \sigma_{\ell_2 k}^N)\big\|^{\bar r_k/2} \gtrsim \big(\exp(\omega_{0\ell_1}^N)+\exp(\omega_{0\ell_2}^N)\big)
        \big\|(\Delta \va_{* k}^N,\;\Delta \sigma_{* k}^N)\big\|^{\bar r_k/2}. \nonumber
\end{align}

Combining the two inequalities above gives the claimed comparison between the contribution of the merged atom and the contributions of the two original atoms:
\begin{align*}
\exp(\omega_{0\ell_1}^N)\Big(\|(\Delta_{\vt_1}\vomega_{1\ell_1 k}^N,\Delta b_{\ell_1 k}^N)\|^{\bar r_k}
 &+ \|(\Delta\va_{\ell_1 k}^N,\Delta\sigma_{\ell_1 k}^N)\|^{\bar r_k/2}\Big) +\exp(\omega_{0\ell_2}^N)\Big(\|(\Delta_{\vt_1}\vomega_{1\ell_2 k}^N,\Delta b_{\ell_2 k}^N)\|^{\bar r_k}
 + \|(\Delta\va_{\ell_2 k}^N,\Delta\sigma_{\ell_2 k}^N)\|^{\bar r_k/2}\Big) \\
&\quad \gtrsim
\big(\exp(\omega_{0\ell_1}^N)+\exp(\omega_{0\ell_2}^N)\big)
\Big(\|(\Delta_{\vt_1}\vomega_{1*k}^N,\Delta b_{*k}^N)\|^{\bar r_k}
 + \|(\Delta\va_{*k}^N,\Delta\sigma_{*k}^N)\|^{\bar r_k/2}\Big).
\end{align*}

Hence
\[
\cD(G_N^{(K)},G_0, t_0, \vt_1)\;\gtrsim\;\cD(G_N^{(K-1)},G_0, t_0, \vt_1),
\]
and therefore
\[
\VDFRA(G_N^{(K-1)},G_0)\;\lesssim\;\VDFRA(G_N^{(K)},G_0).
\]

\subsection{Proof of Theorem~\ref{thm:path_rates}}\label{sec:proof_of_thm1}

First of all, we study the convergence rate of the MLE $\widehat{G}_N \in \cE_K$ of the SGMoE; that is, we will show the inverse bound for SGMoE. We revisit the following result on the identifiability of the SGMoE models, which was previously studied in~\cite{nguyen_demystifying_2023,jiang_identifiability_1999}.

\begin{fact}[{\citealp[Proposition 1]{nguyen_demystifying_2023}}]
\label{prop_identify}
For any mixing measures $G = \sum_{k = 1}^{K} \exp(\omega_{0k}) \delta_{(\vomega_{1k}, \va_{k}, b_{k}, \sigma_{k})}$ and $G' = \sum_{k = 1}^{K'} \exp(\omega_{0k}') \delta_{(\vomega_{1k}', \va_{k}', b_{k}', \sigma_{k}')}$, if we have $p_{G}(y| \vx) = p_{G'}(y| \vx)$ for almost surely $(\vx, y)$, then it follows that $K = K'$ and $G \equiv G'_{t_{0},\vt_1}$ where $G'_{t_{0},\vt_1}:=\sum_{k=1}^{K'}\exp(\omega'_{0k}+t_{0})\delta_{(\vomega'_{1k}+\vt_1,\va_{k}', b_{k}',\sigma_{k}')}$ for some $t_{0}\in\mathbb{R}$ and $\vt_1\in\mathbb{R}^D$.

\end{fact}
The identifiability of the softmax gating Gaussian mixture of experts guarantees that the MLE $\widehat{G}_{N}$ converges to the true mixing measure $G_{0}$ (up to the translation of the parameters in the softmax gating).  

Given the consistency of the MLE, it is natural to ask about its convergence rate to the true parameters. Our next result establishes the convergence rate of conditional density estimation $p_{\widehat{G}_{N}}(y| \vx)$ to the true conditional density $p_{G_{0}}(y| \vx)$, which lays an important foundation for the study of MLE's convergence rate.

\begin{fact}[{\citealp[Proposition 2]{nguyen_demystifying_2023}}]
\label{prop_rate_density} 
The density estimation $p_{\widehat{G}_{N}}(y| \vx)$ converges to the true density $p_{G_0}(y| \vx)$ under the Hellinger distance $\hels(\cdot,\cdot)$ at the following rate:
\begin{align*}
    \bbE_\rvx[\hels(p_{\widehat{G}_{N}}(\cdot|\vx), p_{G_{0}}(\cdot|\vx))] = \mathcal{O}_P( \sqrt{\log(N)/N}).
\end{align*}
That is,
\begin{align*}
    \mathbb{P}(\bbE_\rvx[\hels(p_{\widehat{G}_{N}}(\cdot|\vx), p_{G_{0}}(\cdot|\vx))] > C (\log(N)/N)^{1/2}) \lesssim \exp(-c \log N),
\end{align*}
where $c$ and $C$ are universal constants. 
\end{fact}

The result of Fact~\ref{prop_rate_density} indicates that under either the exact-specified or over-specified cases of the SGMoE, the rate of the conditional density function $p_{\widehat{G}_{N}}(y| \vx)$ to the true one $p_{G_{0}}(y| \vx)$ under Hellinger distance is of order $\mathcal{O}(N^{-1/2})$ (up to some logarithmic factors), which is parametric on the sample size.

Now, we establish the convergence rate of the MLE under the over-specified case of the SGMoE via the Fast-Rate-Aware Voronoi Distance $\VDFRA$.

\begin{theorem}\label{theorem:convergence_rate_of_parameters}
    Under the over-specified case of the SGMoE, namely, when $K > K_0$, we obtain that
    \begin{align}
        \mathbb{E}_{\rvx} [\hels (p_{G} (\cdot | \vx), p_{G_0} (\cdot | \vx))] \geq C \cdot \VDFRA (G, G_0), \nonumber
    \end{align}
    for any $G \in \mathcal{O}_K$ where $C$ is some universal constant depending only on $G_0$ and $\vTheta$. Therefore, that lower bound leads to the following convergence rate of the MLE:
    \begin{align}
        \sP (\VDFRA (\widehat{G}_N, G_0) > C' (\log (N)/N)^{1/2}) \lesssim \exp (-c \log N),
    \end{align}\label{theorem_convergence_rate_of_D_A}
    where $C'$ and $c$ are some universal constants.
\end{theorem}

\begin{proof}[Proof of Theorem~\ref{theorem:convergence_rate_of_parameters}]
    We are going to prove that there exists a constant $C>0$ depending only on $G_0$ and $\vTheta$ such that, for any $G\in\cO_K$,
    \begin{align}
        \bbE_\rvx\big[\TV(p_G(\cdot|\vx),p_{G_0}(\cdot|\vx))\big]\;\gtrsim\;\VDFRA(G,G_0). \label{eq:inverse_bound_DFRA}
    \end{align}
Then, by the Fact~\ref{prop_rate_density}, we get the convergence rate of the MLE of SGMoE.

\textbf{Local version:} Firstly, we prove the local version of the \cref{eq:inverse_bound_DFRA}:
\begin{equation}
    \lim_{\varepsilon \to 0} \inf_{G \in \cO_K:\VDFRA (G, G_0) \leq \varepsilon} \bbE_\rvx\big[\TV(p_G(\cdot|\vx),p_{G_0}(\cdot|\vx))\big]/\VDFRA (G, G_0) > 0.\label{eq:local_version_ineq}
\end{equation}

Assume that the inequality in \cref{eq:local_version_ineq} does not hold true, there exists a sequence of mixing measures $G_N := \sum_{k = 1}^{K_N} \exp{(\omega_{0k}^N)} \delta_{(\vomega_{1k}^N, \va_k^N, b_k^N, \sigma_k^N)} \in \mathcal{O}_K$ such that
\begin{align}
    \mathbb{E}_\rvx [\TV(p_{G_N} (\cdot|x), p_{G_0} (\cdot|x)]/\VDFRA(G_N, G_0) &\to 0, \notag \\
    \VDFRA (G_N, G_0) &\to 0, \notag
\end{align}
when $N$ to infinity. Since the proof argument is asymptotic, we also assume that $K_N = K' \leq K$ for all $N \geq 1$. Next, we consider the Voronoi cells $\sA_k^N := \sA_k(G_N)$, for $k \in [K_0]$, of the mixing measure $G_N$ generated by the true components of $G_0$. And we can assume without loss of generality (WLOG) that those Voronoi cells are independent of $N$ for all $N \in \mathbb{N}$, i.e. $\sA_k = \sA_k^N$. Additionally, since $\VDFRA (G_N, G_0) \to 0$, we have $(\va_{\ell}^N, b_{\ell}^N, \sigma_{\ell}^N) \to (\va_{\ell}^0, b_{\ell}^0, \sigma_{\ell}^0)$ for any $\ell \in \sA_k$ as $N \to \infty$. Furthermore, there exist $t_0 \in \mathbb{R}$ and $\vt_1 \in \mathbb{R}^D$ such that $\sum_{\ell \in \sA_k} \exp{(\omega_{0\ell}^N)} \to \exp{(\omega_{0k}^0 + t_0)}$ and $\vomega_{1\ell}^N \to \vomega_{1k}^0 + \vt_1$ as $N$ approaches infinity for any $\ell \in \sA_k$ and $k \in [K_0]$. It suggests that we can upper bound $\VDFRA$ as $\VDFRA (G_N, G_0) \leq \VD (G_N, G_0)$, where 
\begin{align}
    &\VD (G_N, G_{0}) : =  \sum_{k: |\sA_{k}| > 1} \sum_{\ell \in \sA_{k}} \exp(\omega_{0\ell}^N) 
    \Big(\|(\Delta_{\vt_{1}} \vomega_{1\ell k}^N,  \Delta b_{\ell k}^N)\|^{\bar{r}(|\sA_k|)} 
    + \|(\Delta \va_{\ell k}^N,\Delta \sigma_{\ell k}^N)\|^{\bar{r}(|\sA_k|)/2}\Big) \nonumber \\
    &\quad + \sum_{k: |\sA_{k}| = 1} \sum_{\ell \in \sA_{k}} \exp(\omega_{0\ell}^N) 
    \|(\Delta_{\vt_{1}} \vomega_{1\ell k}^N, \Delta \va_{\ell k}^N,  \Delta b_{\ell k}^N, \Delta \sigma_{\ell k}^N)\| + \sum_{k = 1}^{K_{0}} \Big|\sum_{\ell \in \sA_{k}} \exp(\omega_{0\ell}^N) - \exp(\omega_{0k}^{0} + t_{0})\Big| \nonumber\\
    &\quad + \sum_{k: |\sA_{k}| > 1} \bigg( 
    \norm{\sum_{\ell \in \sA_k} \exp{(\omega_{0\ell}^N)} (\Delta b_{\ell k}^N)} 
    + \norm{\sum_{\ell \in \sA_k} \exp{(\omega_{0\ell}^N)} (\Delta_{\vt_1} \vomega_{1\ell k}^N)} + \norm{\sum_{\ell \in \sA_k} \exp{(\omega_{0\ell}^N)}[(\Delta b_{\ell k}^N)^2 + (\Delta \sigma_{\ell k}^N)]} \nonumber \\
    &\qquad + \norm{\sum_{\ell \in \sA_k} \exp{(\omega_{0\ell}^N)}[(\Delta_{\vt_1} \vomega_{1\ell k}^N)(\Delta b_{\ell k}^N) + (\Delta \va_{\ell k}^N)]} \nonumber +  \norm{\sum_{\ell \in \sA_k} \exp{(\omega_{0\ell}^N)}(\Delta_{\vt_1} \vomega_{1\ell k}^N)(\Delta_{\vt_1} \vomega_{1\ell k}^N)^{\top}} 
    \bigg), \nonumber 
\end{align}
in which 
$\Delta_{\vt_{1}} \vomega_{1\ell k}^N : = \vomega_{1\ell}^N - \vomega_{1k}^{0} - \vt_{1}$, 
$\Delta \va_{\ell k}^N : = \va_{\ell}^N - \va_{k}^{0}$, 
$\Delta b_{\ell k}^N : = b_{\ell}^N - b_{k}^{0}$, 
$\Delta \sigma_{\ell k}^N : = \sigma_{\ell}^N - \sigma_{k}^{0}$.

\textbf{Step 1: Density Decomposition} 

In this step, we try to find a density decomposition for the quatity $Q_{N}=\left[\sum_{k=1}^{K_{0}} \exp \left(\left(\vomega_{1 k}^{0}+\vt_1\right)^{\top} \vx+\omega_{0 k}^{0}+t_{0}\right)\right] \cdot\left[p_{G_{N}}(y | \vx)-p_{G_{0}}(y | \vx)\right]$:
\begin{align*}
Q_{N} 
&= \sum_{k=1}^{K_{0}} \sum_{\ell \in \sA_{k}} \exp \left(\omega_{0 \ell}^{N}\right)
\left[
u\left(y | \vx ; \vomega_{1 \ell}^{N}, \va_{\ell}^{N}, b_{\ell}^{N}, \sigma_{\ell}^{N}\right)
-
u\left(y | \vx ; \vomega_{1 k}^{0}+\vt_{1}, \va_{k}^{0}, b_{k}^{0}, \sigma_{k}^{0}\right)
\right] \\
&\quad - \sum_{k=1}^{K_{0}} \sum_{\ell \in \sA_{k}} \exp \left(\omega_{0 \ell}^{N}\right)
\left[
v\left(y | \vx ; \vomega_{1 \ell}^{N}\right)
-
v\left(y | \vx ; \vomega_{1 k}^{0}+\vt_{1}\right)
\right] \\
&\quad + \sum_{k=1}^{K_{0}}\left(\sum_{\ell \in \sA_{k}} \exp \left(\omega_{0 \ell}^{N}\right)-\exp \left(\omega_{0 k}^{0}+t_{0}\right)\right)
\left[
u\left(y | \vx ; \vomega_{1 k}^{0}+\vt_{1}, \va_{k}^{0}, b_{k}^{0}, \sigma_{k}^{0}\right)
-
v\left(y | \vx ; \vomega_{1 k}^{0}+\vt_{1}\right)
\right], \\
&:= A_{N} + B_{N} + E_{N},
\end{align*}
where we denote $u\left(y | \vx ; \vomega_{1}, \va, b, \sigma\right)
:= \exp \left(\vomega_{1}^{\top} \vx\right) \mathcal{N}\left(y | \va^{\top} \vx + b, \sigma\right)$ and $v\left(y | \vx ; \vomega_{1}\right)
:= \exp \left(\vomega_{1}^{\top} \vx\right) p_{G_N}(y | \vx)$.

Since each Voronoi cell $\sA_k$ possibly has more than one element, we continue to decompose $A_N$ as follows:
\begin{align*}
    A_N &= \sum_{k : |\sA_k| > 1} \sum_{\ell \in \sA_{k}} \exp \left(\omega_{0 \ell}^{N}\right)
\left[
u\left(y | \vx ; \vomega_{1 \ell}^{N}, \va_{\ell}^{N}, b_{\ell}^{N}, \sigma_{\ell}^{N}\right)
-
u\left(y | \vx ; \vomega_{1 k}^{0}+\vt_{1}, \va_{k}^{0}, b_{k}^{0}, \sigma_{k}^{0}\right)
\right] \\
    &+ \sum_{k : |\sA_k| = 1} \sum_{\ell \in \sA_{k}} \exp \left(\omega_{0 \ell}^{N}\right)
\left[
u\left(y | \vx ; \vomega_{1 \ell}^{N}, \va_{\ell}^{N}, b_{\ell}^{N}, \sigma_{\ell}^{N}\right)
-
u\left(y | \vx ; \vomega_{1 k}^{0}+\vt_{1}, \va_{k}^{0}, b_{k}^{0}, \sigma_{k}^{0}\right)
\right] \\
&:= A_{N, 1} + A_{N, 2}.
\end{align*}

Now, we perform Taylor expansion up to the $\bar{r} (|\sA_k|) - $th order, and then rewrite $A_{N, 1}$ with a note that $\valpha = (\valpha_1, \valpha_2, \alpha_3, \alpha_4) \in \mathbb{N}^{D} \times\mathbb{N}^{D} \times \mathbb{N} \times \mathbb{N}$ as follows:
\begin{align*}
    A_{N, 1} &= \sum_{k : |\sA_k| > 1} \sum_{\ell \in \sA_{k}} \sum_{|\valpha| = 1}^{\bar{r}(|\sA_k|)} \frac{\exp{(\omega_{0 \ell}^N)}}{\valpha!} (\Delta_{\vt_1} \vomega_{1\ell k}^N)^{\valpha_1} (\Delta \va_{\ell k}^N)^{\valpha_2} (\Delta b_{\ell k}^N)^{\alpha_3} (\Delta \sigma_{\ell k}^N)^{\alpha_4} \\
    &\times \frac{\partial^{|\valpha_1|+|\valpha_2| + \alpha_3 +\alpha_4}}{\partial \vomega_1^{\valpha_1} \partial \va^{\valpha_2} \partial b^{\alpha_3} \partial \sigma^{\alpha_4}} u (y | \vx; \vomega_{1k}^0 + \vt_1, \va_k^0, b_k^0, \sigma_k^0) + R^N_1 (\vx, y),
\end{align*}
where $R^N_1(\vx, y)$ is the remainder term such that $$R_1^N (\vx, y) = o \left( \sum_{k : |\sA_k| > 1} \sum_{\ell \in \sA_{k}} \exp(\omega_{0\ell}^N) (\|\Delta_{\vt_1} \vomega_{1\ell k}^N \|^{\bar{r}(|\sA_k|)} + \| \Delta \va_{\ell k}^N \|^{\bar{r}(|\sA_k|)} + \| \Delta b_{\ell k}^N \|^{\bar{r}(|\sA_k|)} + \| \Delta \sigma_{\ell k}^N \|^{\bar{r}(|\sA_k|)}) \right).$$

Next, for each $k \in [K_0]$ and $\ell \in \sA_k$, we denote $h_1 (\vx, \va, b) := (\va)^{\top} \vx + b$. By the partial differential equations $$\frac{\partial^2 u}{\partial \vomega_{1} \partial b} = \frac{\partial u}{\partial \va}; \quad \frac{\partial^2 u}{\partial b^2} = 2\frac{\partial u}{\partial \sigma},$$ we have 
\begin{align*}
    \frac{\partial^{|\valpha_2|}u}{\partial \va^{\valpha_2}} = \frac{\partial^{2|\valpha_2|} u}{\partial \vomega_1^{\valpha_2} \partial b^{|\valpha_2|}}; \quad \frac{\partial^{\alpha_4} u}{\partial \sigma^{\alpha_4}} = \frac{1}{2^{\alpha_4}} \cdot \frac{\partial^{2\alpha_4} u}{\partial b^{2\alpha_4}}.
\end{align*}
Hence $$\frac{\partial^{|\valpha_1|+|\valpha_2| + \alpha_3 +\alpha_4}u}{\partial \vomega_1^{\valpha_1} \partial \va^{\valpha_2} \partial b^{\alpha_3} \partial \sigma^{\alpha_4}} = \frac{1}{2^{\alpha_4}} \cdot \frac{\partial^{(|\valpha_1| + |\valpha_2|) + (|\valpha_2| + \alpha_3 + 2\alpha_4)} u}{\partial \vomega_1^{\valpha_1 + \valpha_2} \partial b^{|\valpha_2| + \alpha_3 + 2\alpha_4}}.$$
It follows that 
\begin{align*}
    A_{N, 1} &= \sum_{k : |\sA_k| > 1} \sum_{\ell \in \sA_{k}} \sum_{|\valpha| = 1}^{\bar{r}(|\sA_k|)} \frac{\exp{(\omega_{0 \ell}^N)}}{\valpha!} (\Delta_{\vt_1} \vomega_{1\ell k}^N)^{\valpha_1} (\Delta \va_{\ell k}^N)^{\valpha_2} (\Delta b_{\ell k}^N)^{\alpha_3} (\Delta \sigma_{\ell k}^N)^{\alpha_4} \\
    &\quad \times \frac{1}{2^{\alpha_4}} \cdot \frac{\partial^{(|\valpha_1| + |\valpha_2|) + (|\valpha_2| + \alpha_3 + 2\alpha_4)}}{\partial \vomega_1^{\valpha_1 + \valpha_2} \partial b^{|\valpha_2| + \alpha_3 + 2\alpha_4}} u (y | \vx; \vomega_{1k}^0 + \vt_1, \va_k^0, b_k^0, \sigma_k^0) + R^N_1 (\vx, y) \\
    &= \sum_{k : |\sA_k| > 1} \sum_{\ell \in \sA_{k}} \sum_{|\vell_1| + \ell_2 = 1}^{2\bar{r}(|\sA_k|)} \sum_{\valpha \in \sI_{\vell_1, \ell_2}} \frac{\exp{(\omega_{0 \ell}^N)}}{2^{\alpha_4} \valpha!} (\Delta_{\vt_1} \vomega_{1\ell k}^N)^{\valpha_1} (\Delta \va_{\ell k}^N)^{\valpha_2} (\Delta b_{\ell k}^N)^{\alpha_3} (\Delta \sigma_{\ell k}^N)^{\alpha_4} \\
    &\quad \times \vx^{\vell_1} \exp{\left( (\vomega_{1k}^0 + \vt_1)^{\top} \vx \right)} \cdot \frac{\partial^{\ell_2} \mathcal{N}}{\partial h_1^{\ell_2}} (y | (\va_k^0)^{\top} \vx + b_k^0, \sigma_k^0) + R^N_1 (\vx, y),
\end{align*}
where $\sI_{\vell_1, \ell_{2}} = \{\valpha = (\valpha_{1}, \valpha_{2}, \alpha_{3}, \alpha_{4}) \in \Ns^{D} \times \Ns^{D} \times \Ns \times \Ns: \ \valpha_{1} + \valpha_{2} = \vell_1, \ |\valpha_{2}| + \alpha_{3} + 2 \alpha_{4} = \ell_{2}\}$.

Similarly, we can decompose $A_{N, 2}$ by the first-order Taylor expansion as
\begin{align*}
    A_{N, 2} &= \sum_{k : |\sA_k| = 1} \sum_{\ell \in \sA_k} \sum_{\left|\vell_{1}\right|+\ell_{2}=1}^{2} 
\sum_{\valpha \in \sI_{\ell_1, \ell_2}} \frac{\exp{(\omega_{0 \ell}^N)}}{2^{\alpha_4} \valpha!} (\Delta_{\vt_1} \vomega_{1\ell k}^N)^{\valpha_1} (\Delta \va_{\ell k}^N)^{\valpha_2} (\Delta b_{\ell k}^N)^{\alpha_3} (\Delta \sigma_{\ell k}^N)^{\alpha_4} \\
    &\quad \times \vx^{\vell_1} \exp{\left( (\vomega_{1k}^0 + \vt_1)^{\top} \vx \right)} \cdot \frac{\partial^{|\ell_2|} \mathcal{N}}{\partial h_1^{|\ell_2|}} (y | (\va_k^0)^{\top} \vx + b_k^0, \sigma_k^0) + R^N_2 (\vx, y),
\end{align*}
where $$R_2^N (\vx, y) = o \left( \sum_{k : |\sA_k| = 1} \sum_{\ell \in \sA_{k}} \exp(\omega_{0\ell}^N) (\|\Delta_{\vt_1} \vomega_{1\ell k}^N \| + \| \Delta \va_{\ell k}^N \| + \| \Delta b_{\ell k}^N \| + \| \Delta \sigma_{\ell k}^N \|) \right).$$

Analogously, $B_N$ can be rewritten as 
\begin{align*}
    B_N &= B_{N, 1} + B_{N, 2} \\
    &= - \sum_{k : |\sA_k| > 1} \sum_{\ell \in \sA_k} \sum_{|\gamma|=1}^{2} 
\frac{\exp\left(\omega_{0 \ell}^{N}\right)}{\gamma!}
\left(\Delta_{\vt_1} \vomega_{1 i k}^{N}\right)^{\gamma} \cdot \vx^{\gamma} 
\exp \left(\left(\vomega_{1 k}^{0}+\vt_1\right)^{\top} \vx\right) 
p_{G_{N}}(y | \vx)
+ R^N_{3}(\vx, y) \\
&\quad - \sum_{k : |\sA_k| = 1} \sum_{\ell \in \sA_k} \sum_{|\gamma|=1} 
\frac{\exp\left(\omega_{0 \ell}^{N}\right)}{\gamma!}
\left(\Delta_{\vt_1} \vomega_{1 i k}^{N}\right)^{\gamma} \cdot \vx^{\gamma} 
\exp \left(\left(\vomega_{1 k}^{0}+\vt_1\right)^{\top} \vx\right) 
p_{G_{N}}(y | \vx)
+ R^N_{4}(\vx, y)
\end{align*}
where 
\begin{align*}
    R_3^N (\vx, y) &= o \left( \sum_{k : |\sA_k| > 1} \sum_{\ell \in \sA_{k}} \exp(\omega_{0\ell}^N) (\|\Delta_{\vt_1} \vomega_{1\ell k}^N \|^{2}) \right), \\
    R_4^N (\vx, y) &= o \left( \sum_{k : |\sA_k| = 1} \sum_{\ell \in \sA_{k}} \exp(\omega_{0\ell}^N) (\|\Delta_{\vt_1} \vomega_{1\ell k}^N \|) \right).
\end{align*}

Therefore, $Q_{N}$ can be represented as
\begin{align}
Q_{N} &= \sum_{k=1}^{K_{0}} \sum_{\left|\vell_{1}\right|+\ell_{2}=1}^{2 \bar{r}\left(\left|\sA_{k}\right|\right)} 
T_{\vell_{1}, \ell_{2}}^{N}(k) \cdot \vx^{\ell_{1}} 
\exp \left(\left(\vomega_{1 k}^{0}+\vt_1\right)^{\top} \vx\right) 
\frac{\partial^{\ell_{2}} \mathcal{N}}{\partial h_{1}^{\ell_{2}}}
\left(y | \va_{k}^{0\top} \vx + b_{k}^{0}, \sigma_{k}^{0}\right) \nonumber\\
&\quad + \sum_{k=1}^{K_{0}} \sum_{|\gamma|=1}^{1+\mathbf{1}_{\left\{|\sA_{k}|>1\right\}}} 
S_{\gamma}^{N}(k) \cdot \vx^{\gamma} 
\exp \left(\left(\vomega_{1 k}^{0}+\vt_1\right)^{\top} \vx\right) 
p_{G_{N}}(y | \vx)
+ \sum_{\rho=1}^{4} R_{\rho}^N(\vx, y) \nonumber\\
&\quad + \sum_{k=1}^{K_{0}}\left(\sum_{\ell \in \sA_{k}} \exp \left(\omega_{0 \ell}^{N}\right)-\exp \left(\omega_{0 k}^{0}+t_{0}\right)\right)
\left[
u\left(y | \vx ; \vomega_{1 k}^{0}+\vt_{1}, \va_{k}^{0}, b_{k}^{0}, \sigma_{k}^{0}\right)
-
v\left(y | \vx ; \vomega_{1 k}^{0}+\vt_{1}\right)
\right] \nonumber\\
&= \sum_{k=1}^{K_{0}} \sum_{\left|\vell_{1}\right|+\ell_{2}=0}^{2 \bar{r}\left(\left|\sA_{k}\right|\right)} 
T_{\vell_{1}, \ell_{2}}^{N}(k) \cdot \vx^{\ell_{1}} 
\exp \left(\left(\vomega_{1 k}^{0}+\vt_1\right)^{\top} \vx\right) 
\frac{\partial^{\ell_{2}} \mathcal{N}}{\partial h_{1}^{\ell_{2}}}
\left(y | \va_{k}^{0\top} \vx + b_{k}^{0}, \sigma_{k}^{0}\right) \nonumber\\
&\quad + \sum_{k=1}^{K_{0}} \sum_{|\gamma|=0}^{1+\mathbf{1}_{\left\{|\sA_{k}|>1\right\}}} 
S_{\gamma}^{N}(k) \cdot \vx^{\gamma} 
\exp \left(\left(\vomega_{1 k}^{0}+\vt_1\right)^{\top} \vx\right) 
p_{G_{N}}(y | \vx)
+ \sum_{\rho=1}^{4} R_{\rho}^N(\vx, y), \label{eq:decomposition}
\end{align}
with coefficients $T_{\vell_{1}, \ell_{2}}^{N}(k)$ and $S_{\gamma}^{N}(k)$ are defined for any $k \in [K_0]$, $0 \leq |\vell_1| + \ell_2 \leq 2 \bar{r}\left(\left|\sA_{k}\right|\right)$ and $0 \leq |\gamma| \leq 2$ as
\begin{align*}
    T_{\vell_{1}, \ell_{2}}^{N}(k) &=\begin{cases}
        \sum_{\ell \in \sA_{k}} \sum_{\valpha \in \sI_{\ell_{1}, \ell_{2}}} 
\frac{\exp\left(\omega_{0 \ell}^{N}\right)}{2^{\alpha_{4}} \valpha!}
\left(\Delta_{\vt_1} \vomega_{1 \ell k}^{N}\right)^{\valpha_{1}}
\left(\Delta \va_{\ell k}^{N}\right)^{\valpha_{2}}
\left(\Delta b_{\ell k}^{N}\right)^{\alpha_{3}}
\left(\Delta \sigma_{\ell k}^{N}\right)^{\alpha_{4}}, &(\vell_1, \ell_2) \ne (0_D, 0), \\
\sum_{\ell \in \sA_k} \exp(\omega_{0\ell}^N) - \exp(\omega_{0k}^0 + t_0), &(\vell_1, \ell_2) = (0_D, 0),
    \end{cases}\\
S_{\gamma}^{N}(k) &= \begin{cases}
    -\sum_{\ell \in \sA_{k}} \frac{\exp\left(\omega_{0 \ell}^{N}\right)}{\gamma!}
\left(\Delta_{\vt_1} \vomega_{1 \ell k}^{N}\right)^{\gamma}, &|\gamma| \ne 0,\\
-\sum_{\ell \in \sA_k} \exp(\omega_{0\ell}^N) + \exp(\omega_{0k}^0 + t_0),& |\gamma|=0.
\end{cases}
\end{align*}

\textbf{Step 2: Non-vanishing coefficients} 

Next, we will show that not all the quatities $T_{\vell_{1}, \ell_{2}}^{N}(k)/\VD (G_N, G_0)$ and $S_{\gamma}^{N}(k)/\VD (G_N, G_0)$ go to $0$ as $N \to \infty$. We assume that all of them go to $0$ as $N \to \infty$. Then, by assumption $T_{0_D, 0}^{N}(k)/\VD (G_N, G_0) \to 0$, we have 
\begin{equation}
    \frac{1}{\VD (G_N, G_0)} \sum_{k = 1}^{K_0} \left| \sum_{\ell \in \sA_k} \exp(\omega_{0\ell}^N) - \exp(\omega_{0k}^0 + t_0)\right| \to 0.
\end{equation}

For any $k$ such that $|\sA_k| = 1$, consider all $(|\vell_1|, \ell_2)$ implying $1 \leq |\vell_1| + \ell_2 \leq 2$, we have $T^N_{\vell_1, \ell_2} (k)/\VD(G_N, G_0) \to 0$ for all $k$ such that $|\sA_k| = 1$. Hence 
\begin{equation}
    \frac{1}{\VD (G_N, G_0)} \left( \sum_{k : |\sA_k| = 1} \sum_{\ell \in \sA_k} \exp(\omega_{0\ell}^N) \norm{(\Delta_{\vt_1} \vomega_{1\ell k}^N, \Delta\va_{\ell k}^N, \Delta b_{\ell k}^N, \Delta \sigma_{\ell k}^N)}  \right) \to 0.
\end{equation}

Next, we consider $k$ such that $|\sA_k| > 1$ and $(|\vell_1|, \ell_2)$ such that $1 \leq |\vell_1| + \ell_2 \leq 2$:
\begin{itemize}
    \item For $(|\vell_1|, \ell_2) = (0, 1)$, then 
    \begin{equation}
        \frac{1}{\VD (G_N, G_0)} \norm{\sum_{\ell \in \sA_k} \exp{(\omega_{0\ell}^N)} (\Delta b_{\ell k}^N)} \to 0.\nonumber
    \end{equation}

    \item For $(|\vell_1|, \ell_2) = (1, 0)$, then 
    \begin{equation}
        \frac{1}{\VD (G_N, G_0)} \norm{\sum_{\ell \in \sA_k} \exp{(\omega_{0\ell}^N)} (\Delta_{\vt_1} \vomega_{1\ell k}^N)} \to 0.\nonumber
    \end{equation}

    \item For $(|\vell_1|, \ell_2) = (1, 1)$, then 
    \begin{equation}
        \frac{1}{\VD (G_N, G_0)} \norm{\sum_{\ell \in \sA_k} \exp{(\omega_{0\ell}^N)}[(\Delta_{\vt_1} \vomega_{1\ell k}^N)(\Delta b_{\ell k}^N) + (\Delta \va_{\ell k}^N)]} \to 0.\nonumber
    \end{equation}

    \item For $(|\vell_1|, \ell_2) = (0, 2)$, then 
    \begin{equation}
        \frac{1}{\VD (G_N, G_0)} \norm{\sum_{\ell \in \sA_k} \exp{(\omega_{0\ell}^N)}[(\Delta b_{\ell k}^N)^2 + (\Delta \sigma_{\ell k}^N)]} \to 0.\nonumber
    \end{equation}

    \item For $(|\vell_1|, \ell_2) = (2, 0)$, then 
    \begin{equation}
        \frac{1}{\VD (G_N, G_0)} \norm{\sum_{\ell \in \sA_k} \exp{(\omega_{0\ell}^N)}(\Delta_{\vt_1} \vomega_{1\ell k}^N)(\Delta_{\vt_1} \vomega_{1\ell k}^N)^{\top}} \to 0.\nonumber
    \end{equation}
\end{itemize}

Combining the above limit and the formulation of $\VDFRA (G_N, G_0)$ together, it follows that
\[
\frac{1}{\VD (G_N, G_0)} \cdot \sum_{k:\left|\sA_k\right|>1} \sum_{\ell \in \sA_k} \exp \left(\omega_{0 \ell}\right)\left(\left\|\left(\Delta_{\vt_{1}} \vomega_{1 \ell k}^{N}, \Delta b_{\ell k}^{N}\right)\right\|^{\bar{r}\left(\left|\sA_k\right|\right)}+\left\|\left(\Delta \va_{\ell k}^{N}, \Delta \sigma_{\ell k}^{N}\right)\right\|^{\bar{r}\left(\left|\sA_k\right|\right) / 2}\right) \nrightarrow 0
\]

which implies that there exists some index $k^{*} \in\left[K_0\right]$ such that $\left|\sA_{k^*}\right|>1$ and

\[
\frac{1}{\VD (G_N, G_0)} \cdot \sum_{\ell \in \sA_{k^*}}\exp \left(\omega_{0 \ell}\right)\left(\left\|\left(\Delta_{vt_{1}} \vomega_{1 ell k^{*}}^{N}, \Delta b_{\ell k^{*}}^{N}\right)\right\|^{\bar{r}\left(\left|\sA_k\right|\right)}+\left\|\left(\Delta \va_{\ell k^{*}}^{N}, \Delta \sigma_{\ell k^{*}}^{N}\right)\right\|^{\bar{r}\left(\left|\sA_k\right|\right) / 2}\right) \nrightarrow 0
\]

for all $\vt_{1} \in \mathbb{R}^{D}$. WLOG, we assume that $k^{*}=1$. For $\left(\vell_{1}, \ell_{2}\right) \in \mathbb{N}^{D} \times \mathbb{N}$ such that $1 \leq\left|\vell_{1}\right|+\ell_{2} \leq \bar{r}\left(\left|\sA_{1}\right|\right)$, we have $T_{\vell_{1}, \ell_{2}}^{N}(1) / \VD (G_N, G_0) \to 0$ as $N \to \infty$. Thus, by dividing this ratio and the left hand side of the above equation and let $\vt_{1}=0$, we have
\begin{equation}
\frac{\sum_{\ell \in \sA_{1}} \sum_{\valpha \in \sI_{\vell_{1}, \ell_{2}}} \frac{\exp \left(\omega_{0 \ell}^{N}\right)}{2^{\alpha_{4}} \alpha!}\left(\Delta_{\vt_{1}} \vomega_{1 \ell 1}^{N}\right)^{\valpha_{1}}\left(\Delta \va_{\ell 1}^{N}\right)^{\valpha_{2}}\left(\Delta b_{\ell 1}^{N}\right)^{\alpha_{3}}\left(\Delta \sigma_{\ell 1}^{N}\right)^{\alpha_{4}}}{\sum_{\ell \in \sA_{1}} \exp \left(\omega_{0 \ell}^{N}\right)\left(\left\|\left(\Delta_{\vt_{1}} \vomega_{1 \ell 1}^{N}, \Delta b_{\ell 1}^{N}\right)\right\|^{\bar{r}\left(\left|\sA_{1}\right|\right)}+\left\|\left(\Delta \va_{\ell 1}^{N}, \Delta \sigma_{\ell 1}^{N}\right)\right\|^{\bar{r}\left(\left|\sA_{1}\right|\right) / 2}\right)} \to 0 \label{eq:pre_poly_eq}
\end{equation}
for all $\left(\vell_{1}, \ell_{2}\right)$ such that $1 \leq\left|\vell_{1}\right|+\ell_{2} \leq \bar{r}\left(\left|\sA_{1}\right|\right)$.\\
Let us define $\bar{M}_{N}:=\max \left\{\left\|\Delta_{\vt_{1}} \vomega_{1 \ell 1}^{N}\right\|,\left\|\Delta \va_{\ell 1}^{N}\right\|^{1 / 2},\left|\Delta b_{\ell 1}^{N}\right|,\left|\Delta \sigma_{\ell 1}^{N}\right|^{1 / 2}: \ell \in \sA_{1}\right\}$ and $\bar{\omega}_{N}:= \max_{\ell \in \sA_{1}} \exp \left(\omega_{0 \ell}^{N}\right)$. Since the sequence $\exp \left(\omega_{0 \ell}^{N}\right) / \bar{\omega}_{N}$ is bounded, we can replace it by its subsequence that has a positive limit $p_{5 \ell}^{2}:=\lim _{N \to \infty} \exp \left(\omega_{0 \ell}^{N}\right) / \bar{\omega}_{N}$. Hence, at least one among $p_{5 \ell}^{2}$, for $\ell \in \sA_{1}$, equals $1$.

Similarly, we also define
\begin{align}
\left(\Delta_{\vt_{1}} \vomega_{1 \ell 1}^{N}\right) / \bar{M}_{N} \to p_{1 \ell}, & \left(\Delta \va_{\ell 1}^{N}\right) / \bar{M}_{N} \to p_{2 \ell}, \nonumber\\
\left(\Delta b_{\ell 1}^{N}\right) / \bar{M}_{N} \to p_{3 \ell}, & \left(\Delta \sigma_{\ell 1}^{N}\right) /\left[2 \bar{M}_{N}\right] \to p_{4 \ell}. \nonumber
\end{align}
Here, at least one of $p_{1 \ell}, p_{2 \ell}, p_{3 \ell}$ and $p_{4 \ell}$ for $\ell \in \sA_{1}$ equals either $1$ or $-1$ . Next, we divide both the numerator and the denominator of the ratio in \cref{eq:pre_poly_eq} by $\bar{\omega}_{N} \bar{M}_{N}^{\vell_{1}+\ell_{2}}$, and then achieve the following system of polynomial equations:

\[
\sum_{\ell \in \sA_{1}} \sum_{\valpha \in \sI_{\vell_{1}, \ell_{2}}} \frac{1}{\alpha!} \cdot p_{5 \ell}^{2} p_{1 \ell}^{\valpha_{1}} p_{2 \ell}^{\valpha_{2}} p_{3 \ell}^{\alpha_{3}} p_{4 \ell}^{\alpha_{4}}=0
\]

for all $\left(\vell_{1}, \ell_{2}\right) \in \mathbb{N}^{D} \times \mathbb{N}$ such that $1 \leq\left|\vell_{1}\right|+\ell_{2} \leq \bar{r}\left(\left|\sA_{1}\right|\right)$. However, based on the definition of $\bar{r}\left(\left|\sA_{1}\right|\right)$, the above system has no non-trivial solutions, which is a contradiction. Thus, not all the quantities $T_{\vell_{1}, \ell_{2}}^{N}(k) / \VD (G_N, G_0)$ and $S_{\gamma}^{N}(k) / \VD (G_N, G_0)$ go to $0$ as $N \to \infty$.

\textbf{Step 3: Fatou's lemma involvement}

Following this, we define by $m_{N}$ be the maximum of the absolute values of those quantities. Based on the result in Step 2, we know that $1 / m_{N} \nrightarrow \infty$. Then, by applying the Fatou's lemma, we obtain that
\begin{equation}
    \lim_{N \to \infty} \frac{\bbE_\rvx\big[\TV(p_G(\cdot|\vx),p_{G_0}(\cdot|\vx))\big]}{m_N \cdot \VD(G_N, G_0)} \ge \int \liminf_{N \to \infty} \frac{|p_G(y|\vx),p_{G_0}(y|\vx)|}{2m_N \cdot \VD (G_N, G_0)}\;\mathrm{d}(\vx, y).\label{eq:fatou_lemma}
\end{equation}

By assumption, the left-hand side of \cref{eq:fatou_lemma} equals to $0$, so the integrand in the right-hand side also equals to $0$ for almost surely $(\vx, y)$. Hence, we get that $Q_{N} /\left[m_{N} \VD (G_N, G_0)\right] \to 0$ as $N \to \infty$ for almost surely $(\vx, y)$. It follows from the decomposition of $Q_{N}$ in \cref{eq:decomposition} that
\begin{align}
\sum_{k=1}^{K_0} \sum_{\left|\vell_{1}\right|+\ell_{2}=0}^{2 \bar{r}\left(\left|\sA_k\right|\right)} & \tau_{\vell_{1}, \ell_{2}}(k) \cdot \vx^{\vell_{1}} \exp \left(\left(\vomega_{1 k}^{0}+\vt_{1}\right)^{\top} \vx\right) \frac{\partial^{\ell_{2}} \cN}{\partial h_{1}^{\ell_{2}}}\left(y|\left(\va_{k}^{0}\right)^{\top} \vx+b_{k}^{0}, \sigma_{k}^{0}\right) \nonumber\\
& +\sum_{k=1}^{K_0} \sum_{|\gamma|=0}^{1+\mathbf{1}_{\left\{\left|\sA_k\right|>1\right\}}} \xi_{\gamma}(j) \cdot \vx^{\gamma} \exp \left(\left(\vomega_{1 k}^{0}+\vt_{1}\right)^{\top} \vx\right) p_{G_{0}}(y|\vx)=0, \nonumber
\end{align}
for almost surely $(\vx, y)$, where $\tau_{\vell_{1}, \ell_{2}}(k)$ and $\xi_{\gamma}(k)$ denote the limits of $T_{\vell_{1}, \ell_{2}}^{N}(k) /\left[m_{N} \VD (G_N, G_0)\right]$ and $S_{\gamma}^{N}(j) /\left[m_{N} \VD (G_N, G_0)\right]$ as $N \to \infty$, respectively, for all $k \in\left[K_0\right], 0 \leq 2\left|\vell_{1}\right|+\ell_{2} \leq 2 \bar{r}\left(\left|\sA_k\right|\right)$ and $0 \leq|\gamma| \leq 1+\mathbf{1}_{\left\{\left|\sA_k\right|>1\right\}}$. By definition, at least one among $\tau_{\ell_{1}, \ell_{2}}(k)$ and $\xi_{\gamma}(k)$ is different from zero. 

Furthermore, we denote the set $\cW$ as follows:
\begin{align}
\cW &:=\left\{\vx^{\vell_{1}} \exp \left(\left(\vomega_{1 k}^{0}+\vt_{1}\right)^{\top} \vx\right) \frac{\partial^{\ell_{2}} \cN}{\partial h_{1}^{\ell_{2}}}\left(y|\left(\va_{k}^{0}\right)^{\top} \vx+b_{k}^{0}, \sigma_{k}^{0}\right): k \in\left[K_0\right], 0 \leq \left|\vell_{1}\right|+\ell_{2} \leq 2 \bar{r}\left(\left|\sA_k\right|\right)\right\} \nonumber \\
\quad&\cup\left\{\vx^{\gamma} \exp \left(\left(\vomega_{1 k}^{0}+\vt_{1}\right)^{\top} \vx\right) p_{G_{0}}(y|\vx): k \in\left[K_0\right], 0 \leq|\gamma| \leq 1+\mathbf{1}_{\left\{\left|\sA_k\right|>1\right\}}\right\}. \nonumber
\end{align}

Similarly to the proof of Fact~\ref{lem:set_independent} in \citealp{nguyen_demystifying_2023}:
\begin{fact}[{\citealp[Lemma 2]{nguyen_demystifying_2023}}]\label{lem:set_independent}
    The set $\cW_1$ is linearly indeqendent w.r.t $\vx$ and $y$, where $\cW_1$ is denoted as follows:
    \begin{align}
\cW_1 &:=\left\{\vx^{\vell_{1}} \exp \left(\left(\vomega_{1 k}^{0}+\vt_{1}\right)^{\top} \vx\right) \frac{\partial^{\ell_{2}} \cN}{\partial h_{1}^{\ell_{2}}}\left(y|\left(\va_{k}^{0}\right)^{\top} \vx+b_{k}^{0}, \sigma_{k}^{0}\right): k \in\left[K_0\right], 0 \leq \left|\vell_{1}\right|+\ell_{2} \leq 2 \right\} \nonumber \\
\quad&\cup\left\{\vx^{\gamma} \exp \left(\left(\vomega_{1 k}^{0}+\vt_{1}\right)^{\top} \vx\right) p_{G_{0}}(y|\vx): k \in\left[K_0\right], 0 \leq|\gamma| \leq 1\right\}, \nonumber
\end{align}
\end{fact}

the set $\cW$ is linearly independent w.r.t $\vx$ and $y$, it follows that
\[
\tau_{\vell_{1}, \ell_{2}}(k)=\xi_{\gamma}(k)=0
\]
for all $k \in\left[K_0\right], 0 \leq 2\left|\ell_{1}\right|+\ell_{2} \leq 2 \bar{r}\left(\left|\sA_k\right|\right)$ and $0 \leq|\gamma| \leq 1+\mathbf{1}_{\left\{\left|\sA_k\right|>1\right\}}$, which is a contradiction. Hence, we achieve the \cref{eq:local_version_ineq}.

\textbf{Global version:} Hence, it is sufficient to prove its following global inequality:
\begin{equation}
    \inf_{G \in \cO_K:\VDFRA (G, G_0) > \varepsilon^{\prime}} \bbE_\rvx\big[\TV(p_G(\cdot|\vx),p_{G_0}(\cdot|\vx))\big]/\VDFRA (G, G_0) > 0.\label{eq:global_ineq}
\end{equation}

Assume by contrary that there exists a sequence $G_{N}^{\prime} \in \cO_K$ that satisfies
$$\begin{cases}
    \lim_{N \to \infty} \bbE_\rvx\big[\TV(p_{G_N^{\prime}}(\cdot|\vx),p_{G_0}(\cdot|\vx))\big]/\VDFRA (G_N^{\prime}, G_0) = 0, \\
    \VDFRA (G_N^{\prime}, G_0) > \varepsilon^{\prime}.
\end{cases}$$

Then, we get that $\bbE_\rvx\big[\TV(p_{G_N^{\prime}}(\cdot|\vx),p_{G_0}(\cdot|\vx))\big] \to 0$ as $N \to \infty$. Since the set $\vTheta$ is compact, we can replace the sequence $G_{N}^{\prime}$ by its subsequence which converges to some mixing measure $G^{\prime} \in \cO_K$ such that $\VDFRA\left(G^{\prime}, G_{0}\right)>\varepsilon^{\prime}$. Then, by the Fatou's lemma, we get
\begin{align}
    \lim_{N \to \infty} \bbE_\rvx\big[\TV(p_{G_N^{\prime}}(\cdot|\vx),p_{G_0}(\cdot|\vx))\big] \ge \frac{1}{2} \int \liminf_{N \to \infty} |p_{G_N^{\prime}}(y|\vx) - p_{G_0}(y|\vx)| \mathrm{d} (\vx, y). \nonumber
\end{align}
It follows that
\[
\int\left|p_{G^{\prime}}(y|\vx)-p_{G_{0}}(y|\vx)\right| \mathrm{d} (\vx, y)=0.
\]
Thus, we obtain that $p_{G^{\prime}}(y|\vx)=p_{G_{0}}(y|\vx)$ for almost surely $(\vx, y)$. By Fact~\ref{prop_identify}, the mixing measure $G^{\prime}$ admits the form $G^{\prime}=\sum_{k=1}^{K_{0}} \exp \left(\omega_{0 \nu(k)}^{0}+t_{0}\right) \delta_{\left(\vomega_{1 \nu(k)}^{0}+\vt_{1}, \va_{\nu(k)}^{0}, b_{\nu(k)}^{0}, \sigma_{\nu(k)}^{0}\right)}$ for some $(t_{0}, \vt_1) \in \R \times \R^D$, where $\nu$ is some permutation of the set $\{1,2, \ldots, K_0\}$. It follows that $\VDFRA\left(G^{\prime}, G_{0}\right)=0$, which contradicts the hypothesis $\VDFRA\left(G^{\prime}, G_{0}\right)>\varepsilon^{\prime}>0$. Hence, we obtain the inequality in \cref{eq:inverse_bound_DFRA}.
\end{proof}

Next, assume that $\widehat{G}_N \in \cE_K$ with $K > K_0$. From \cref{prop_rate_density}, there exists a constant $c(\vTheta, K)$ depending on $\vTheta$ and $K$ so that on an event, we call $A_N$, with probability at least $1 - CN^{-c}$, we have $$\bbE_\rvx[\TV(p_{\widehat{G}_{N}}(\cdot|\vx), p_{G_{0}}(\cdot|\vx))] \leq \sqrt{2}\bbE_\rvx[\hels(p_{\widehat{G}_{N}}(\cdot|\vx), p_{G_{0}}(\cdot|\vx))] \leq c(\vTheta, K) \cdot \left( \frac{\log N}{N} \right)^{1/2}.$$
Now, we prove Theorem~\ref{thm:path_rates}.

\begin{proof}[Proof of Theorem~\ref{thm:path_rates}]
    Firstly, we prove for the over-specified case. By \cref{lem:monotone_path} and \cref{theorem:convergence_rate_of_parameters}, we have the first statement.

    To prove the rest, we need to consider the exact-specified case. When $\kappa' = K_0$, by definition of $\VDFRA (\widehat{G}^{(\kappa')}_N, G_0)$, we obtain that $\VDFRA (\widehat{G}^{(K_0)}_N, G_0) = \VDE (\widehat{G}^{(K_0)}_N, G_0)$. Hence, by \cref{lem:monotone_path}, we get the convergence rate $$\VDE (\widehat{G}^{(K_0)}_N, G_0) \lesssim \left( \frac{\log N}{N} \right)^{1/2}.$$

    Assume that $\widehat{G}^{(K_0)}_N = \sum_{k = 1}^{K_0} \exp(\omega_{0k}^{N}) \delta_{(\vomega_{1k}^{N}, \va^N_k, b^N_k, \sigma^N_k)} \in \mathcal{E}_{K_0}$. Building on our previous work, there exist $t_0 \in \mathbb{R}$ and $\vt_1 \in \mathbb{R}^D$ such that for large $N$ enough, we get
    \begin{align}
        |\exp(\omega_{0k}^N) - \exp(\omega_{0k}^{0} + t_{0})| &\lesssim \left(\frac{\log N}{N}\right)^{1/2},\nonumber \\
        \|(\Delta_{\vt_{1}} \vomega_{1k}^N, \Delta \va_{k}^N,  \Delta b_{k}^N, \Delta \sigma_{k}^N)\| &\lesssim \left(\frac{\log N}{N}\right)^{1/2},\nonumber
    \end{align}
    for every $k \in [K_0]$, where $\Delta_{\vt_{1}}^N \vomega_{1k}^N := \vomega_{1k}^N - \vomega_{1k}^0 - \vt_1$, $\Delta \va_{k}^N := \va_{k}^N - \va_{k}^0$, $\Delta b_{k}^N := b_k^N - b_k^0$ and $\Delta \sigma_{k}^N := \sigma_k^N - \sigma_k^0$.

    This implies that for every $(i, j) \in [K_0]^2$, by the triangle inequality, we have  
    \begin{align*}
            \bigg| \|(\vomega_{1i}^N - \vomega_{1j}^N, b_i^N - b_j^N) \| &- \|(\vomega_{1i}^0 - \vomega_{1j}^0, b_i^0 - b_j^0) \| \bigg| \leq \|(\vomega_{1i}^N - \vomega_{1j}^N - \vomega_{1i}^0 + \vomega_{1j}^0, b_i^N - b_j^N - b_i^0 + b_j^0) \| \\
            &\leq \|(\vomega_{1i}^N - \vomega_{1i}^0 - \vt_1, b_i^N - b_i^0) \| + \|(\vomega_{1j}^N - \vomega_{1j}^0 - \vt_1, b_j^N - b_j^0) \| \\
            &\lesssim \left( \dfrac{\log N}{N} \right)^{1/2}.
        \end{align*}

    Similarly, we have $$\bigg| \|(\va^N_i - \va^N_j, \sigma^N_i - \sigma^N_j)\| - \|(\va^0_i - \va^0_j, \sigma^0_i - \sigma^0_j)\| \bigg| \lesssim \left( \dfrac{\log N}{N} \right)^{1/2}.$$

    Hence, we obtain that 
    \begin{align}
        \Bigg| &\dfrac{1}{\exp{( -\omega^N_{0i})} + \exp{(-\omega^N_{0j})}} \left( \| (\vomega^N_{1i} - \vomega^N_{1j},   b^N_{i} - b^N_j)\|^{2} + \|( \va^N_{i} - \va^N_j, \sigma^N_{i} - \sigma^N_j)\| \right) \nonumber\\
        &- \dfrac{1}{\exp{( -\omega^0_{0i} - t_0)} + \exp{(-\omega^0_{0j} - t_0)}} \left( \| (\vomega^0_{1i} - \vomega^0_{1j},   b^0_{i} - b^0_j)\|^{2} + \|( \va^0_{i} - \va^0_j, \sigma^0_{i} - \sigma^0_j)\| \right) \Bigg| \nonumber\\
        &\lesssim \left( \dfrac{\log N}{N} \right)^{1/2}, \quad \forall (i, j) \in [K_0]^2. \label{eq:under_specified_ineq}
    \end{align}

    Hence, on $A_N$, the optimal choice of indices $(\ell_1, \ell_2)$ to merge for $\widehat{G}^{(K_0)}_N$ will be the same as $G_0$ for every $N$ large enough. It follows that we have two merged atoms are $\exp{(\omega_{0*}^N)} \delta_{(\vomega_{1*}^N, \va_*^N, b_*^N, \sigma_*^N)}$ and $\exp{(\omega_{0*}^0)} \delta_{(\vomega_{1*}^0, \va_*^0, b_*^0, \sigma_*^0)}$ denoted as follows:
    \begin{align}
    \omega_{0*}^N &= \log \left( \exp{\omega_{0\ell_1}^N} + \exp{\omega_{0\ell_2}^N} \right), \nonumber \\
    \vomega_{1*}^N &= \exp{\left(\omega_{0\ell_1}^N - \omega_{0*}^N \right)} \vomega_{1\ell_1}^N 
    + \exp{\left(\omega_{0\ell_2}^N - \omega_{0*}^N\right)} \vomega_{1\ell_2}^N, \nonumber\\
    b_*^N &= \exp{\left(\omega_{0\ell_1}^N - \omega_{0*}^N \right)} b_{\ell_1}^N 
    + \exp{\left(\omega_{0\ell_2}^N - \omega_{0*}^N\right)} b_{\ell_2}^N, \nonumber\\
    \va_*^N &= \frac{\exp(\omega_{0\ell_1}^N)}{\exp(\omega_{0*}^N)} \Big[(\vomega_{1\ell_1}^N - \vomega_{1*}^N)(b_{\ell_1}^N - b_*^N) + \va_{\ell_1}^N\Big] 
    + \frac{\exp(\omega_{0\ell_2}^N)}{\exp(\omega_{0*}^N)} \Big[(\vomega_{1\ell_2}^N - \vomega_{1*}^N)(b_{\ell_2}^N - b_*^N) + \va_{\ell_2}^N\Big], \nonumber \\
    \sigma_*^N &= \frac{\exp(\omega_{0\ell_1}^N)}{\exp(\omega_{0*}^N)} \Big[(b_{\ell_1}^N - b_*^N)^2 + \sigma_{\ell_1}^N\Big] 
    + \frac{\exp(\omega_{0\ell_2}^N)}{\exp(\omega_{0*}^N)} \Big[(b_{\ell_2}^N - b_*^N)^2 + \sigma_{\ell_2}^N\Big], \nonumber
\end{align}
    and 
    \begin{align}
    \omega_{0*}^0 &= \log \left( \exp{\omega_{0\ell_1}^0} + \exp{\omega_{0\ell_2}^0} \right), \nonumber \\
    \vomega_{1*}^0 &= \exp{\left(\omega_{0\ell_1}^0 - \omega_{0*}^0 \right)} \vomega_{1\ell_1}^0 
    + \exp{\left(\omega_{0\ell_2}^0 - \omega_{0*}^0\right)} \vomega_{1\ell_2}^0, \nonumber\\
    b_*^0 &= \exp{\left(\omega_{0\ell_1}^0 - \omega_{0*}^0 \right)} b_{\ell_1}^0 
    + \exp{\left(\omega_{0\ell_2}^0 - \omega_{0*}^0\right)} b_{\ell_2}^0, \nonumber\\
    \va_*^0 &= \frac{\exp(\omega_{0\ell_1}^0)}{\exp(\omega_{0*}^0)} \Big[(\vomega_{1\ell_1}^0 - \vomega_{1*}^0)(b_{\ell_1}^0 - b_*^0) + \va_{\ell_1}^0\Big] 
    + \frac{\exp(\omega_{0\ell_2}^0)}{\exp(\omega_{0*}^0)} \Big[(\vomega_{1\ell_2}^0 - \vomega_{1*}^0)(b_{\ell_2}^0 - b_*^0) + \va_{\ell_2}^0\Big], \nonumber \\
    \sigma_*^0 &= \frac{\exp(\omega_{0\ell_1}^0)}{\exp(\omega_{0*}^0)} \Big[(b_{\ell_1}^0 - b_*^0)^2 + \sigma_{\ell_1}^0\Big] 
    + \frac{\exp(\omega_{0\ell_2}^0)}{\exp(\omega_{0*}^0)} \Big[(b_{\ell_2}^0 - b_*^0)^2 + \sigma_{\ell_2}^0\Big]. \nonumber
\end{align}

    After merging, we also have
    \begin{align}
        |\exp(\omega_{0*}^N) - \exp(\omega_{0*}^{0} + t_{0})| &= |\exp(\omega_{0\ell_1}^N) + \exp(\omega_{0\ell_2}^N)  - \exp(\omega_{0\ell_1}^{0} + t_{0}) - \exp(\omega_{0\ell_1}^{0} + t_{0})| \nonumber \\
        & \leq |\exp(\omega_{0\ell_1}^N) - \exp(\omega_{0\ell_1}^{0} + t_{0})| + |\exp(\omega_{0\ell_1}^N) - \exp(\omega_{0\ell_1}^{0} + t_{0})| \nonumber\\
        &\lesssim \left( \dfrac{\log N}{N} \right)^{1/2},
    \end{align}
    and 
    \begin{align}
        \exp(\omega_{0*}^N)&\|(\Delta_{\vt_{1}}^N \vomega_{1*}^N, \Delta \va_{*}^N,  \Delta b_{*}^N, \Delta \sigma_{*}^N)\| \nonumber\\
        &\leq \exp(\omega_{0*}^N) \times \exp{(\omega^N_{0 \ell_1} - \omega_{0*}^N)}\|(\Delta_{\vt_{1}}^N \vomega_{1\ell_1}^N, \Delta \va_{\ell_1}^N,  \Delta b_{\ell_1}^N, \Delta \sigma_{\ell_1}^N)\| \nonumber\\
        &\hspace{2 em} + \exp(\omega_{0*}^N) \times \exp{(\omega^N_{0 \ell_2} - \omega_{0*}^N)}\|(\Delta_{\vt_{1}}^N \vomega_{1\ell_2}^N, \Delta \va_{\ell_2}^N,  \Delta b_{\ell_2}^N, \Delta \sigma_{\ell_2}^N)\| \nonumber\\
        &\lesssim \left( \dfrac{\log N}{N} \right)^{1/2}. \nonumber
    \end{align}

    Hence, $\VDE(\widehat{G}_{N}^{(K_0 - 1)}, G_{0}^{(K_0 - 1)}) \lesssim \left(\frac{\log N}{N}\right)^{1/2}$. By the induction, we have the rest statement.
\end{proof}

\subsection{Proof of Theorem~\ref{thm:heights}}
    For the convergence rate of the height at all levels $\kappa \geq K_0 + 1$, from \cref{thm:path_rates}, we have $$\VDFRA (\widehat{G}_N^{(\kappa)}, G_0) \lesssim \left( \frac{\log N}{N} \right)^{1/2}.$$ Because $\kappa \geq K_0 + 1$, by the pigeonhole principle, there exists at least two $i, j \in [\kappa]$ such that two atoms $\exp(\omega_{0i}^{N}) \delta_{(\vomega_{1i}^{N}, \va^N_i, b^N_i, \sigma^N_i)}$ and $\exp(\omega_{0j}^{N}) \delta_{(\vomega_{1j}^{N}, \va^N_j, b^N_j, \sigma^N_j)}$ belongs to a common Voronoi cell of some $\vtheta_k^0$ (we suppress the dependence of $i, j$, and $\sA_k$ on $N$ for ease of notation). Hence, 
    \begin{align*}
        &\inf_{\vt_1}  \exp(\omega_{0i}) \left( \|(\Delta_{\vt_{1}} \vomega_{1 i k},  \Delta b_{i k})\|^{\bar{r}(|\sA_k|)} + \|(\Delta \va_{i k},\Delta \sigma_{i k})\|^{\bar{r}(|\sA_k|)/2} \right)  \\
        &\qquad+  \exp(\omega_{0j}) \left( \|(\Delta_{\vt_{1}} \vomega_{1 j k},  \Delta b_{j k})\|^{\bar{r}(|\sA_k|)} + \|(\Delta \va_{j k},\Delta \sigma_{j k})\|^{\bar{r}(|\sA_k|)/2} \right) \lesssim \left(\frac{\log N}{N} \right)^{1/2}.
    \end{align*}

    Using the fact that $\min \{ \exp{(\omega_{0i})}, \exp{(\omega_{0j})} \} \geq \dfrac{1}{\exp{( -\omega_{0i})} + \exp{(-\omega_{0j})}}$, $\bar{r} (\widehat{G}_N) \geq \bar{r} (|\sA_k|) \geq \bar{r} (2) = 4$, and using the H\"older's inequality, for every $\vt_1 \in \mathbb{R}^D$ we have 
    \begin{align*}
        &\exp(\omega_{0i}) \left( \|(\Delta_{\vt_{1}} \vomega_{1 i k},  \Delta b_{i k})\|^{\bar{r}(|\sA_k|)} + \|(\Delta \va_{i k},\Delta \sigma_{i k})\|^{\bar{r}(|\sA_k|)/2} \right)  \\ 
        &\hspace{1 em}+  \exp(\omega_{0j}) \left( \|(\Delta_{\vt_{1}} \vomega_{1 j k},  \Delta b_{j k})\|^{\bar{r}(|\sA_k|)} + \|(\Delta \va_{j k},\Delta \sigma_{j k})\|^{\bar{r}(|\sA_k|)/2} \right) \\
        &\geq \dfrac{1}{\exp{( -\omega_{0i})} + \exp{(-\omega_{0j})}} \bigg[ \left(\|(\Delta_{\vt_{1}} \vomega_{1 i k},  \Delta b_{i k})\|^{\bar{r}(|\sA_k|)} + \|(\Delta_{\vt_{1}} \vomega_{1 j k},  \Delta b_{j k})\|^{\bar{r}(|\sA_k|)} \right)  \\
        &\hspace{1 em}+ \left( \|(\Delta \va_{i k},\Delta \sigma_{i k})\|^{\bar{r}(|\sA_k|)/2} + \|(\Delta \va_{j k},\Delta \sigma_{j k})\|^{\bar{r}(|\sA_k|)/2} \right) \bigg] \\
        &\gtrsim \dfrac{1}{\exp{( -\omega_{0i})} + \exp{(-\omega_{0j})}} \left( \|   (\vomega_{1i} - \vomega_{1j},   b_{i} - b_j)\|^{\bar{r}(|\sA_k|)} + \|( \va_{i} - \va_j, \sigma_{i} - \sigma_j)\|^{\bar{r}(|\sA_k|)/2} \right) \\
        &\gtrsim \left( \dfrac{1}{\exp{( -\omega_{0i})} + \exp{(-\omega_{0j})}} \left( \| (\vomega_{1i} - \vomega_{1j},   b_{i} - b_j)\|^{2} + \|( \va_{i} - \va_j, \sigma_{i} - \sigma_j)\| \right) \right)^{\bar{r}(\widehat{G}_N)/2}.
    \end{align*}

    Since the height of the dendrogram is the minimum of $\divclus$ over all pairs $(i, j)$, we obtain that $$\height_N^{(\kappa)} \lesssim \dfrac{1}{\exp{( -\omega_{0i})} + \exp{(-\omega_{0j})}} \left( \| (\vomega_{1i} - \vomega_{1j},   b_{i} - b_j)\|^{2} + \|( \va_{i} - \va_j, \sigma_{i} - \sigma_j)\| \right) \lesssim \left( \frac{\log N}{N} \right)^{1/\bar{r}(\widehat{G}_N)},$$ for all $\kappa \geq K_0 +1$.

    When $\kappa \leq K_0$, the conclusion follows from inequality in \cref{eq:under_specified_ineq} in the proof of \cref{thm:path_rates}.

\subsection{Proof of Theorem~\ref{thm:likelihood_path}}

Before we prove \cref{thm:likelihood_path}, we revisit preliminary on empirical process theory and connection between the Hellinger distance and the Wasserstein metric.

\paragraph{Preliminary on Empirical Process Theory.}

Suppose $\rvx_1, \ldots, \rvx_N \sim P_{G_0}$. Denote $P_N := \frac{1}{N} \sum_{n = 1}^N \delta_{\rvx_n}$ is the empirical measure. Denote the empirical process for $G$: $$\nu_N (G) := \sqrt{N} (P_N - P_{G_0}) \log{\frac{\bar{p}_{G}}{p_{G_0}}}.$$ The following results is important in proof below.

\begin{fact}[{\citealp[Theorem 5.11]{van_de_geer_empirical_2000}}]\label{thm:empirical_process}
    Let positive numbers $R, C, C_1, a$ satisfy: $$a \leq C_1 \sqrt{N} R^2 \land 8 \sqrt{N} R,$$ and $$a \geq \sqrt{C^2 (C_1 + 1)} \left( \int_{a/(2^6 \sqrt{N})}^R H^{1/2}_B \left( \frac{u}{\sqrt{2}}, \{ p_G : G \in \mathcal{O}_K, \hels (p_G, p_{G_0}) \leq R\}, \nu \right) \mathrm{d} u \lor R \right),$$ then $$\sP_{G_0} \left( \sup_{G \in \mathcal{O}_K, \hels (p_G, p_{G_0}) \leq R} |\nu_N (G)| \geq a \right) \leq C \exp{ \left(- \frac{a^2}{C^2(C_1 + 1) R^2} \right)}.$$
\end{fact}

\paragraph{Connection between the
Hellinger distance and the Wasserstein metric.}
We introduce the Wasserstein distances to measure the difference between two measures. For two mixing measure $G=\sum_{k=1}^{K} p_{k} \delta_{\theta_{k}}$ and $G^{\prime}=\sum_{\ell=1}^{K^{\prime}} p_{\ell}^{\prime} \delta_{\theta_{\ell}^{\prime}}$, the Wasserstein-$r$ distance (for $r \geq 1$ ) between $G$ and $G^{\prime}$ is defined as
\begin{align}
    W_{r}\left(G, G^{\prime}\right):=\left(\inf _{\vq \in \Pi\left(\vp, \vp^{\prime}\right)} \sum_{k, \ell=1}^{K, K^{\prime}} q_{k \ell}\left\|\theta_{k}-\theta_{\ell}^{\prime}\right\|^{r}\right)^{1 / r},\label{eq:wasserstein_metric}
\end{align}

where $\Pi\left(\vp, \vp^{\prime}\right)$ is the set of all couplings between $\vp=\left(p_{1}, \ldots, p_{K}\right)$ and $\vp^{\prime}=\left(p_{1}^{\prime}, \ldots, p_{K^{\prime}}^{\prime}\right)$, i.e, $\Pi\left(\vp, \vp^{\prime}\right)=\left\{\vq \in \mathbb{R}_{+}^{K \times K^{\prime}}: \sum_{k=1}^{K} q_{k \ell}=p_{\ell}^{\prime}, \sum_{l=1}^{K^{\prime}} q_{k \ell}=p_{k}, \forall k \in[K], \ell \in\left[K^{\prime}\right]\right\}$. Fix $G_{0}=\sum_{k=1}^{K_{0}} \pi_{k}^{0} \delta_{\theta_{k}^{0}} \in \mathcal{E}_{K_{0}}$, and consider $G=\sum_{\ell=1}^{K} \pi_{\ell} \delta_{\theta_{\ell}}$ such that $W_{r}\left(G, G_{0}\right) \rightarrow 0$, we obtain that
\[
W_{r}^{r}\left(G, G_{0}\right) \asymp \sum_{k=1}^{K_{0}}\left(\left|\sum_{\ell \in \sA_{k} (G)} \pi_{\ell}-\pi_{k}^{0}\right|+\sum_{\ell \in \sA_{k} (G)} \pi_{\ell}\left\|\theta_{\ell}-\theta_{k}^{0}\right\|^{r}\right).
\]

Now, we remind Lemma 1 in \citet{ho_convergence_2016}.
\begin{fact}[{\citealp[Lemma 1]{ho_convergence_2016}}]\label{lem:compare_hellinger}
    Let $G = \sum_{i = 1}^k p_i \delta_{\theta_i}$ denote a discrete probability measure and $p_G (x) = \sum_{i = 1}^k p_i f(x | \theta_i)$ be the mixture density. According to the Lemma 1: Let $G, G' \in \mathcal{O}_k (\Theta)$ such that both $\rho_{\phi} (p_G, p_{G'})$ and $d_{\rho_{\phi}} (G, G')$ are finite for some convex function $\phi$. Then, $\rho_{\phi} (p_G, p_{G'}) \leq d_{\rho_{\phi}} (G, G')$.
\end{fact}

By \cref{lem:compare_hellinger}, we can compare the expectation of Hellinger distance between $p_{G} (y | \vx)$ and $p_{G'} (y | \vx)$ with the Wasserstein metric between $G$ and $G'$ following: $$\bbE_{\rvx} (\hels(p_{G} (\cdot | \vx), p_{G'} (\cdot | \vx)) \lesssim W_2 (G, G').$$

Now, we are going to prove \cref{thm:likelihood_path}. 

\begin{proof}[Proof of \cref{thm:likelihood_path}]
Firstly, we recall the empirical average log-likelihood and population average log-likelihood as follows:
\begin{align}
    \bar\ell_N(p_G) &= \frac{1}{N} \sum_{n=1}^N \log p_G(y_n | \rvx_n) =: P_N \log p_G, \nonumber\\
    \cL(p_G) &= \mathbb{E}_{(\rvx,y) \sim P_{G_0}} \left[ \log p_G(y | \vx) \right] = \int \log p_G(y | \vx) \, dP_{G_0}(\vx, y) =: P_{G_0} \log p_G, \nonumber
\end{align}
where $P_N := \frac{1}{N} \sum_{n=1}^N \delta_{(\rvx_n, y_n)}$ is the empirical measure from data, and the joint distribution $P_{G_0}$ over $(\vx, y)$ is then constructed by first sampling $\vx \sim P_{\rvx}$ and then $y | \vx \sim p_{G_0}(y | \vx)$.

    We divide into three cases.
    
    \textbf{Case 1:} $\kappa \geq K_0$. For any $G$, we denote $P_G$ by the distribution of $p_G$. By the concavity of log function, we have $$\frac{1}{2} \log{\frac{p_G}{p_{G_0}}} \leq \log{\frac{p_G + p_{G_0}}{2p_{G_0}}} = \log{\frac{\bar{p}_G}{p_{G_0}}}, \quad \forall G\in \mathcal{O}_K.$$

    Therefore, for all $\kappa > K_0$ we have 
    \begin{align*}
        \frac{1}{2} P_N \log{\frac{p_{\widehat G^{(\kappa)}_N}}{p_{G_0}}} &\leq P_N \log{\frac{\bar{p}_{\widehat G^{(\kappa)}_N}}{p_{G_0}}} \\
        &= (P_N - P_{G_0}) \log{\frac{\bar{p}_{\widehat G^{(\kappa)}_N}}{p_{G_0}}} - \mathrm{KL} (p_{G_0} \| \bar{p}_{\widehat G^{(\kappa)}_N}) \\
        &\leq (P_N - P_{G_0}) \log{\frac{\bar{p}_{\widehat G^{(\kappa)}_N}}{p_{G_0}}}.
    \end{align*}

    Hence,
    \begin{align}
        P_N \log{p_{\widehat G^{(\kappa)}_N}} - P_{G_0} \log p_{G_0} &= P_N \log{\frac{p_{\widehat G^{(\kappa)}_N}}{p_{G_0}}} + (P_N - P_{G_0}) \log{p_{G_0}} \nonumber\\
        &\leq 2(P_N - P_{G_0}) \log{\frac{\bar{p}_{\widehat G^{(\kappa)}_N}}{p_{G_0}}} + (P_N - P_{G_0}) \log{p_{G_0}}. \nonumber
    \end{align}

    By \cref{thm:path_rates}, we obtain that $\VDFRA (\widehat{G}_N^{(\kappa)}, G_0) \lesssim (\log{N}/N)^{1/2}$, and obviously we have $$\inf_{t_0, \vt_1} W_{\bar{r}(\widehat{G}_N)}^{\bar{r}(\widehat{G}_N)} (\widehat{G}_N^{(\kappa)}, G_{0, t_0, \vt_1}) \leq \VDFRA (\widehat{G}_N^{(\kappa)}, G_0),$$ where $G_{0, t_0, \vt_1} = \sum_{k = 1}^{K_0} \exp{(\omega_{0k}^0 + t_0) \delta_{(\vomega_{1k}^0 + \vt_1, \va_k^0, b_k^0, \sigma_k^0)}}$, so there exists a constant $D$ such that $$\mathbb{P}_{G_0} \left( \inf_{t_0, \vt_1} W_{\bar{r}(\widehat{G}_N)} (\widehat{G}_N^{(\kappa)}, G_{0, t_0, \vt_1}) \leq D \left( \frac{\log N}{N} \right)^{1/2\bar{r}(\widehat{G}_N)} \right) \geq 1 - c_1 N^{-c_2}, \quad \kappa \in [K_0, K].$$

    Now, we compare Wasserstein metrics $W_2$ and $W_{\bar{r}(\widehat{G}_N)}$. Since $2/\bar{r}(\widehat{G}_N) \leq 2/4 < 1$, with a note that for a probability $q_{i, j}$, we have $q_{i, j} \leq q_{i, j}^{2/ \bar{r}(\widehat{G}_N)}$. Combining with all norms on finite space is equivalent, we obtain that $$\left( \sum_{i, j} q_{i, j} \norm{\theta_k - \theta'_{\ell}}^2 \right)^{1/2} \leq \left( \sum_{i, j} q_{i, j}^{2/ \bar{r}(\widehat{G}_N)} \norm{\theta_k - \theta'_{\ell}}^2 \right)^{1/2} \lesssim \left( \sum_{i, j} q_{i, j} \norm{\theta_k - \theta'_{\ell}}^{\bar{r}(\widehat{G}_N)} \right)^{1/\bar{r}(\widehat{G}_N)}.$$

    Then, we get $W_2 \lesssim W_{\bar{r}(\widehat{G}_N)}$. Using the fact that $\bbE_{\rvx} (\hels(p_{G} (\cdot | \vx), p_{G'} (\cdot | \vx)) \lesssim W_2 (G, G')$ and $p_{G_0} = p_{G_{0, t_0, \vt_1}}, \forall (t_0, \vt_1) \in \mathbb{R} \times \mathbb{R}^D$, we also have 
    \begin{align*}
        \mathbb{P}_{G_0} &\left( \mathbb{E} (\hels(p_{\widehat{G}_N^{(\kappa)}} (\cdot | \vx), p_{G_{0}} (\cdot | \vx)) \leq D \left( \frac{\log N}{N} \right)^{1/2\bar{r}(\widehat{G}_N)} \right) \\
        &= \mathbb{P}_{G_0} \left( \inf_{t_0, \vt_1} \bbE_{\rvx} (\hels(p_{\widehat{G}_N^{(\kappa)}} (\cdot | \vx), p_{G_{0, t_0, \vt_1}} (\cdot | \vx)) \leq D \left( \frac{\log N}{N} \right)^{1/2\bar{r}(\widehat{G}_N)} \right) \\
        &\geq \mathbb{P}_{G_0} \left( \inf_{t_0, \vt_1} W_{\bar{r}(\widehat{G}_N)} (\widehat{G}_N^{(\kappa)}, G_{0, t_0, \vt_1}) \leq D \left( \frac{\log N}{N} \right)^{1/2\bar{r}(\widehat{G}_N)} \right) \geq 1 - c_1 N^{-c_2}, \quad \kappa \in [K_0, K].
    \end{align*}

    Let $\mathcal{P}_K (\Theta) := \{ p_G (y | \vx) : G \in \mathcal{O}_K (\Theta) \}$ and $H_B (\varepsilon, \mathcal{P}_K (\Theta), h)$ denotes the bracketing entropy of $\mathcal{P}_K (\Theta)$ under the Hellinger distance. By the Lemma 3 in \cite{nguyen_demystifying_2023}, there is a constant $C > 0$ such that $H_B (\varepsilon, \mathcal{P}_K (\Theta), h) \lesssim \log(1/\varepsilon)$ for any $0 \leq \varepsilon \leq 1/2$.

    Define, $\alpha := 1/2\bar{r}(\widehat{G}_N) \leq 1/4$, substitute $R = D \left( \frac{\log N}{N} \right)^{\alpha}$, $a = D \frac{\log^{\alpha + 1/2} N}{N^{\alpha}}$, then for any positive number $\varepsilon < R$, we have $0 \leq \varepsilon \leq 1/e < 1/2$ and $\log(1/\varepsilon) > 1$ for large $N$ enough. Therefore, for large $N$ enough, we obtain that $a \leq \sqrt{N} R^2 \leq \sqrt{N} R$ and 
    \begin{align*}
        a \geq R \left( \log \left( \frac{2^6 \sqrt{N}}{a} \right) \right) &\geq \int_{a/(2^6 \sqrt{N})}^R \log \frac{1}{\varepsilon} d \varepsilon \\
        &\geq \int_{a/(2^6 \sqrt{N})}^R \log^{1/2} \frac{1}{\varepsilon} d \varepsilon \\
        &\geq \int_{a/(2^6 \sqrt{N})}^R H^{1/2}_B (\varepsilon, \mathcal{P}_K (\Theta), h) d \varepsilon \\
        &\geq \int_{a/(2^6 \sqrt{N})}^R H^{1/2}_B (\varepsilon, \{ p_G : G \in \mathcal{O}_{K} (\Theta), \bbE_{\rvx}(\hels(p_G(\cdot | \vx), p_{G_0}(\cdot | \vx)) \leq R \}, \nu) d \varepsilon.
    \end{align*}

    By \cref{thm:empirical_process}, we get $$\mathbb{P}_{G_0} \left( \sup_{\bbE_{\rvx}(\hels(p_G(\cdot | \vx), p_{G_0}(\cdot | \vx)) \leq D(\log N/N)^{\alpha}} \left| \sqrt{N} (P_N - P_{G_0}) \log \frac{\bar{p}_G}{p_{G_0}} \right| \geq D \frac{\log^{\alpha + 1/2} N}{N^{\alpha}} \right) \leq N^{-c_2}.$$

    Combining with the bound on Hellinger distance, we have 
    \begin{align*}
&\mathbb{P}_{G_0}\!\left(\left|(P_N - P_{G_0}) \log \frac{\bar{p}_{\widehat{G}_N^{(\kappa)}}}{p_{G_0}}\right|
   \geq D \frac{\log^{\alpha + 1/2} N}{N^{\alpha + 1/2}}\right) \\
&\leq \mathbb{P}_{G_0}\!\left( \bbE_{\rvx}(\hels(p_{\widehat{G}_N^{(\kappa)}}(\cdot | \vx), p_{G_0}(\cdot | \vx)) \geq D(\log N/N)^{\alpha} \right) \\
&\quad + \mathbb{P}_{G_0}\!\left(\left|(P_N - P_{G_0}) \log \frac{\bar{p}_{\widehat{G}_N^{(\kappa)}}}{p_{G_0}}\right|
   \geq D \frac{\log^{\alpha + 1/2} N}{N^{\alpha + 1/2}},\;
   \bbE_{\rvx}(\hels(p_{\widehat{G}_N^{(\kappa)}}(\cdot | \vx), p_{G_0}(\cdot | \vx)) \leq D(\log N/N)^{\alpha} \right) \\
& \leq c_1 N^{-c_2} 
   + \mathbb{P}_{G_0}\!\left(\sup_{\bbE_{\rvx}(\hels(p_G(\cdot | \vx), p_{G_0}(\cdot | \vx)) \leq D(\log N/N)^{\alpha}} 
   \left|\sqrt{N}(P_N - P_{G_0}) \log \frac{\bar{p}_G}{p_{G_0}}\right|
   \geq D \frac{\log^{\alpha + 1/2} N}{N^{\alpha}}\right) \\
& \leq c_1' N^{-c_2}.
\end{align*}

   For the second term, by the Chebyshev inequality, we have 
    \begin{align}
        \mathbb{P}_{G_0} (|(P_N - P_{G_0}) \log{p_{G_0}}| \geq t) \leq \frac{\Var (\log p_{G_0})}{N t^2}. \label{inequality of Case 1}
    \end{align}

    Choose $t = (\log N/ N)^{\alpha}$, we have $$\mathbb{P}_{G_0} (|(P_N - P_{G_0}) \log{p_{G_0}}| \leq (\log N/ N)^{\alpha}) \geq 1 - c_1 N^{-c_2}.$$

    Hence, we conclude that $$\mathbb{P}_{G_0} \left(\bar{\ell}_N (\widehat{G}^{(\kappa)}_N) - \cL (p_{G_0}) \leq \left( \dfrac{\log N}{N} \right)^{1/2 \bar{r}(\widehat{G}_N)}\right) \geq 1 - c_1 N^{-c_2}. $$

    \textbf{Case 2:} $\kappa = K_0$. By the \cref{thm:path_rates}, we have $$\VDE (\widehat{G}^{(K_0)}_N, G_0) \lesssim \left( \frac{\log N}{N} \right)^{1/2}.$$

    Assume that $\widehat{G}^{(K_0)}_N = \sum_{k = 1}^{K_0} \exp(\omega_{0k}^{N}) \delta_{(\vomega_{1k}^{N}, \va^N_k, b^N_k, \sigma^N_k)}$, since $\VDE (\widehat{G}^{(K_0)}_N, G_0) \to 0$ as $N \to \infty$, the Voronoi cell $\sA_k$ has only one element for any $k \in [K_0]$. WLOG, we suppose that $\sA_k = \{ k \}$ for all $k \in [K_0]$. Moreover, there exist $t_0 \in \mathbb{R}$ and $\vt_1 \in \mathbb{R}^D$ independent of $N$ such that $\exp(\omega_{0k}^N) \to \exp(\omega_{0k}^0 + t_0)$ and $\vomega_{1k}^N \to \vomega_{1k}^0 + \vt_1$ as $N \to \infty$ for all $k \in [K_0]$. By the definition of $\VDE (\widehat{G}^{(K_0)}_N, G_0)$, we get for large $N$ enough
    \begin{align}
        |\exp(\omega_{0k}^N) - \exp(\omega_{0k}^{0} + t_{0})| \lesssim \left(\frac{\log N}{N}\right)^{1/2}, \quad
        \|(\Delta_{\vt_{1}} \vomega_{1k}^N, \Delta \va_{k}^N,  \Delta b_{k}^N, \Delta \sigma_{k}^N)\| \lesssim \left(\frac{\log N}{N}\right)^{1/2}, \nonumber
    \end{align}
    for every $k \in [K_0]$, where $\Delta_{\vt_{1}}^N \vomega_{1k}^N := \vomega_{1k}^N - \vomega_{1k}^0 - \vt_1$, $\Delta \va_{k}^N := \va_{k}^N - \va_{k}^0$, $\Delta b_{k}^N := b_k^N - b_k^0$ and $\Delta \sigma_{k}^N := \sigma_k^N - \sigma_k^0$.

    Because the function $f(\vx, y | \vtheta) = u (y | \vx; \vomega_{1}, \va, b, \sigma)$ satisfies Condition K (see \cref{lem:checking_condition_K}), let $\epsilon_N = (\log N/N)^{1/2} \to 0$, from condition K, there exist $c_{\alpha}$ and $c_{\omega}$ such that $$u (y | \vx; \vomega_{1k}^N, \va_k^N, b_k^N, \sigma_k^N) \geq (u (y | \vx; \vomega_{1k}^0 + \vt_1, \va_k^0, b_k^0, \sigma_k^0))^{(1 + c_{\omega} \epsilon_N)} e^{-c_{\alpha} \epsilon_N}, \quad \forall k \in [K_0].$$

    Besides, we can find constant $c_q > 0$ and $c_p > 0$ such that
    \begin{align*}
        \exp(\omega_{0k}^N) &\geq (1 - c_p \epsilon_N) \exp{(\omega_{0k}^0 + t_0)}, &\forall k \in [K_0],\\
        \sum_{k = 1}^{K_0} \exp((\vomega_{1k}^0 + \vt_1)^{\top} \vx + \omega_{0k}^0 + t_0) &\geq (1 - c_q \epsilon_N) \sum_{k = 1}^{K_0} \exp((\vomega_{1k}^N)^{\top} \vx + \omega_{0k}^N), &\forall k \in [K_0].
    \end{align*}
    Hence, we have 
    \begin{align*}
        &\left[\sum_{k = 1}^{K_0} \exp((\vomega_{1k}^0 + \vt_1)^{\top} \vx + \omega_{0k}^0 + t_0) \right] \cdot p_{\widehat{G}^{(K_0)}_N} (y | \vx) \\
        &\quad = \frac{\sum_{k = 1}^{K_0} \exp((\vomega_{1k}^0 + \vt_1)^{\top} \vx + \omega_{0k}^0 + t_0)}{\sum_{k = 1}^{K_0} \exp((\vomega_{1k}^N)^{\top} \vx + \omega_{0k}^N)} \cdot \sum_{k = 1}^{K_0} \exp{(\omega_{0k}^N)} u (y | \vx; \vomega_{1k}^N, \va_k^N, b_k^N, \sigma_k^N) \\
        &\quad \geq (1 - c_q \epsilon) \sum_{k = 1}^{K_0} (1 - c_p \epsilon) \exp{(\omega_{0k}^0 + t_0)} (u (y | \vx; \vomega_{1k}^0 + \vt_1, \va_k^0, b_k^0, \sigma_k^0))^{(1 + c_{\omega} \epsilon_N)} e^{-c_{\alpha} \epsilon_N}.
    \end{align*}

    With the fact that $g(t) = t^{1 + c_{\omega} \epsilon_N}$ is a convex function, we get
    \begin{align*}
        p_{\widehat{G}^{(K_0)}_N} (y | \vx) &\geq (1 - c_q \epsilon) (1 - c_p \epsilon) \frac{1}{\sum_{k = 1}^{K_0} \exp((\vomega_{1k}^0 + \vt_1)^{\top} \vx + \omega_{0k}^0 + t_0)} \\
        &\quad \times \sum_{k = 1}^{K_0} \exp{(\omega_{0k}^0 + t_0)} (u (y | \vx; \vomega_{1k}^0 + \vt_1, \va_k^0, b_k^0, \sigma_k^0))^{(1 + c_{\omega} \epsilon_N)} e^{-c_{\alpha} \epsilon_N} \\
        &\geq (1 - c_q \epsilon) (1 - c_p \epsilon) e^{-c_{\alpha} \epsilon_N} \sum_{k = 1}^{K_0} \frac{\exp{((\omega_{1k}^0)^{\top} \vx + \omega_{0k}^0)}}{\sum_{j = 1}^{K_0} \exp{((\omega_{1j}^0)^{\top} \vx + \omega_{0j}^0)}} \cdot \mathcal{N} (y | \va_k^0 \vx + b_k^0, \sigma_k^0)^{(1 + c_{\omega} \epsilon_N)} \\
        &\geq (1 - c_q \epsilon) (1 - c_p \epsilon) e^{-c_{\alpha} \epsilon_N} p_{G_0} (y | \vx)^{(1 + c_{\omega} \epsilon_N)}.
    \end{align*}

    Therefore, we have $$\frac{1}{N} \sum_{i = 1}^N \log \frac{p_{\widehat{G}^{(K_0)}_N}}{p_{G_0}} (y_i | \vx_i) \geq \log ((1 - c_q \epsilon) (1 - c_p \epsilon)) - (c_{\alpha} \epsilon_N) + (c_{\omega} \epsilon_N) \frac{1}{N} \sum_{i = 1}^N \log p_{G_0} (y_i | \vx_i).$$

    Hence
    \begin{align}
        \bar{\ell} (p_{\widehat{G}^{(K_0)}_N}) - \cL (p_{G_0}) &\geq \log ((1 - c_q \epsilon) (1 - c_p \epsilon)) - (c_{\alpha} \epsilon_N) + (c_{\omega} \epsilon_N) P_{G_0} \log p_{G_0} + (1 + c_{\omega} \epsilon_N) (P_N - P_{G_0}) \log p_{G_0}. \label{inequality of Case 2}
    \end{align}

    Now, we will bound the right-hand side of above equation, from Chebyshev inequality from \cref{inequality of Case 1}, choose $t = (\log N/N)^{1/2}$, we get that 
    \begin{align}
        \mathbb{P}_{G_0} \left(|(P_N - P_{G_0}) \log{p_{G_0}}| \geq \left( \frac{\log N}{N} \right)^{1/2}\right) \leq \frac{\Var (\log p_{G_0})}{\log N}. \nonumber
    \end{align}

    Obviously, the terms $|\log ((1 - c_q \epsilon) (1 - c_p \epsilon)) - (c_{\alpha} \epsilon_N) + (c_{\omega} \epsilon_N) P_{G_0} \log p_{G_0}| \lesssim \epsilon_N = (\log N/N)^{1/2}$, thus there exist a constant $C > 0$ such that $\log ((1 - c_q \epsilon) (1 - c_p \epsilon)) - (c_{\alpha} \epsilon_N) + (c_{\omega} \epsilon_N) P_{G_0} \log p_{G_0} \geq -C(\log N/N)^{1/2}$. Then, for some constant $C_e > 0$, we have $$\mathbb{P}_{G_0} \left( \text{RHS of \cref{inequality of Case 2}} \geq - C_e \left( \frac{\log N}{N} \right)^{1/2}  \right) \geq 1 - \frac{\Var (\log p_{G_0})}{\log N}.$$

    Call the event under above case is $B$, then we obtain that 
    \begin{align*}
        \mathbb{P}_{G_0} \left( \bar{\ell} (p_{\widehat{G}^{(K_0)}_N}) - \cL (p_{G_0}) \geq - C_e \left( \frac{\log N}{N} \right)^{1/2}  \right) &\geq \mathbb{P}_{G_0} (A_N \cap B)  = \mathbb{P}_{G_0} (B) - \mathbb{P}_{G_0} (B \cap A_N^c) \\
        &\geq \mathbb{P}_{G_0} (B) - \mathbb{P}_{G_0} (A_N^c) = 1 - \frac{\Var (\log p_{G_0})}{\log N} - c_1 N^{-c_2}
    \end{align*}
    approach $1$ when $N \to \infty$, where $A_N$ is defined in \cref{sec:proof_of_thm1}. Therefore, combine both results, we can conclude that $$|\bar{\ell} (p_{\widehat{G}^{(K_0)}_N}) - \cL (p_{G_0})| \lesssim \left( \dfrac{\log N}{N} \right)^{1/2 \bar{r}(\widehat{G}_N)}.$$

    \textbf{Case 3:} $\kappa < K_0$. Since $|\log p_G(y | \vx)| \leq m(y | \vx)$ for a measurable function $m$ for all $G \in \mathcal{O}_\kappa$, we can use uniform law of large number to get that
\[
\sup_{G \in \mathcal{O}_\kappa} \big| \bar{\ell}_N(G) - P_{G_0} \log p_G \big| 
\overset{\mathbb{P}}{\longrightarrow} 0,
\]
where $\overset{\mathbb{P}}{\longrightarrow}$ means convergence in probability. Therefore,
\[
\big| \bar{\ell}_N(\widehat{G}_N^{(\kappa)}) - P_{G_0} \log p_{\widehat{G}_N^{(\kappa)}} \big|
\overset{\mathbb{P}}{\longrightarrow} 0.
\]

We know that $\log p_{\widehat{G}_N^{(\kappa)}} \to \log p_{G_0^{(\kappa)}}$ in probability, by application of Dominated Convergence theorem, we obtain
\[
P_{G_0} \log p_{\widehat{G}_N^{(\kappa)}} 
\overset{\mathbb{P}}{\longrightarrow} 
P_{G_0} \log p_{G_0^{(\kappa)}}.
\]

Combining the above results together, we get
\[
\bar{\ell}_N(\widehat{G}_N^{(\kappa)}) 
\overset{\mathbb{P}}{\longrightarrow} 
P_{G_0} \log p_{G_0^{(\kappa)}} 
= \cL(\log P_{G_0^{(\kappa)}}).
\]
\end{proof}

\paragraph{Checking condition K.} 
Finally, we check condition K for the function $f(\vx, y | \vtheta)
:= \exp \left(\vomega_{1}^{\top} \vx\right) \mathcal{N}\left(y | \va^{\top} \vx + b, \sigma\right)$.

\begin{lemma}\label{lem:checking_condition_K}
    The condition K is satisfied for $f(\vx, y | \vtheta)
:= \exp \left(\vomega_{1}^{\top} \vx\right) \mathcal{N}\left(y | \va^{\top} \vx + b, \sigma\right)$, where $\vtheta = (\vomega_{1}, \va, b, \sigma) \in \R^D \times \R^D \times \R \times \R$ and $\cX$ are bounded as from the initial setup, and the eigenvalues of $\sigma$ are bounded below and above by the positive constants $\sigma_{\min }$ and $\sigma_{\max }$.
\end{lemma}

\begin{proof}[Proof of \cref{lem:checking_condition_K}]
    When $\|\vtheta -\vtheta^{0}\| \leq \epsilon$ with $\vtheta^0 = (\vomega_1^0, \va^0, b^0, \sigma^0)$, by the equivalence of the norm, we can consider the cases where $\left\|\vomega_{1}-\vomega_{1}^{0}\right\|,\left\|\va-\va^{0}\right\|,\left\|b-b^{0}\right\|,\left\|\sigma-\sigma^{0}\right\| \leq \epsilon$. We aim to show that for sufficiently small $\epsilon$, there exist $c_{\alpha}, c_{\beta}>0$ such that

\[
\log \left(\exp (\vomega_1^{\top} \vx) \cN(y | \va^{\top} \vx+b, \sigma)\right) \geq\left(1+c_{\beta} \epsilon\right) \log \left(\exp ((\vomega_1^0)^{\top} \vx) \cN(y | (\va^0)^{\top} \vx+b^0, \sigma^0)\right)-c_{\alpha} \epsilon .
\]
which is equivalent to
\begin{align*}
    & \Big[\left(1+c_{\beta} \epsilon\right) \left(\vomega_{1}^{0}\right)^{\top}\vx- \left(\vomega_{1}\right)^{\top}\vx\Big] +  \Big[\left(1+c_{\beta} \epsilon\right) \log \left(\left|\sigma^{0}\right|\right)-\log (|\sigma|) \Big] \\
+ & \Big[\left(1+c_{\beta} \epsilon\right)\left(y-(\va^{0})^{\top} \vx-b^{0}\right)^{\top}\left(\sigma^{0}\right)^{-1}\left(y-(\va^{0})^{\top} \vx-b^{0}\right)-(y-\va^{\top} \vx-b)^{\top}(\sigma)^{-1}(y-\va^{\top} \vx-b)\Big]+c_{\alpha} \epsilon \geq 0.
\end{align*}

Firstly, since $\cX$ is bounded, we can omit the term $\Big[\left(1+c_{\beta} \epsilon\right) \left(\vomega_{1}^{0}\right)^{\top}\vx- \left(\vomega_{1}\right)^{\top}\vx\Big]$. Next, we note that
\[
\frac{\mathrm{d} \log (|\sigma|)}{\mathrm{d} \sigma}=\sigma^{-1}
\]

and if $\|\sigma\|$ is bounded above and below far from 0 (which satisfies because $\sigma$ is positive definite), then the map $\sigma \mapsto \log (|\sigma|)$ is Lipschitz; that is, there exists a constant $c_{\sigma}$ such that
\[
\left|\log \left(\left|\sigma^{0}\right|\right)-\log (|\sigma|)\right| \leq c_{\sigma}\left\|\sigma^{0}-\sigma\right\| .
\]

Furthermore, we have $|\sigma| \geq \sigma_{\min}$. Hence, for all $c_{\beta}>\frac{c_{\sigma}}{\log \left(\sigma_{\min}\right)}$, then we have

$$
c_{\beta} \epsilon \log \left(|\sigma^{0}|\right) \geq c_{\sigma} \epsilon \geq c_{\sigma}\left\|\sigma-\sigma^{0}\right\| \geq\left|\log (|\sigma|)-\log \left(|\sigma^{0}|\right)\right| .
$$

So that

$$
\left(1+c_{\beta} \epsilon\right) \log \left(|\sigma^{0}|\right) \geq \log (|\sigma|).
$$

We want to choose $c_{\alpha}>0$ such that

$$
\Big[\left(1+c_{\beta} \epsilon\right)\left(y-(\va^{0})^{\top} \vx-b^{0}\right)^{\top}\left(\sigma^{0}\right)^{-1}\left(y-(\va^{0})^{\top} \vx-b^{0}\right)-(y-\va^{\top} \vx-b)^{\top}(\sigma)^{-1}(y-\va^{\top} \vx-b)\Big]+c_{\alpha} \epsilon \geq 0.
$$

Let $u:=y-(\va^{0})^{\top} \vx-b^{0}, \Delta u:= (\va^{0})^{\top} \vx+b^{0} - [\va^{\top} \vx+b]$, using the boundedness of $\sigma$, there exist $c_{\sigma}$ such that

$$(\sigma^{0})^{-1} \geq c_{\sigma} \sigma^{-1}.$$

Hence, we only need to prove

$$
\left(1+c_{\beta} \epsilon\right) c_{\sigma} u^{\top} \sigma^{-1} u-(u+\Delta u)^{\top} \sigma^{-1}(u+\Delta u)+c_{\alpha} \epsilon \geq 0,
$$

which is equivalent to

\begin{align*}
&c_{\beta} \epsilon c_{\sigma} u^{\top} \sigma^{-1} u-u^{\top} \sigma^{-1} \Delta u-(\Delta u)^{\top} \sigma^{-1} u-(\Delta u)^{\top} \sigma^{-1} \Delta u+c_{\alpha} \epsilon \geq 0 \\
\Leftrightarrow \quad&\epsilon c_{\beta} c_{\sigma}\left(u-\frac{\Delta u}{\epsilon c_{\beta} c_{\sigma}}\right)^{\top} \sigma^{-1}\left(u-\frac{\Delta u}{\epsilon c_{\beta} c_{\sigma}}\right)+c_{\alpha} \epsilon \geq\left(1+\frac{1}{\epsilon c_{\beta} c_{\sigma}}\right)(\Delta u)^{\top} \sigma^{-1}(\Delta u) .
\end{align*}

We can bound the right-hand side of above equation as follow
$$
\left(1+\frac{1}{\epsilon c_{\beta} c_{\sigma}}\right)(\Delta u)^{\top} \sigma^{-1}(\Delta u) \leq\left(1+\frac{1}{\epsilon c_{\beta} c_{\sigma}}\right) \frac{\|\Delta u\|^{2}}{\sigma_{\min }} \leq\left(1+\frac{1}{\epsilon c_{\beta} c_{\sigma}}\right) \frac{\epsilon^{2}}{\sigma_{\min }} .
$$

Hence, it is sufficient to choose $c_{\alpha}$ such that

$$
c_{\alpha} \geq\left(1+\frac{1}{\epsilon c_{\beta} c_{\sigma}}\right) \frac{\epsilon}{\sigma_{\min }}=\frac{\epsilon}{\sigma_{\min }}+\frac{1}{c_{\beta} c_{\sigma} \sigma_{\min }}.
$$
Then
$\left(1+c_{\beta} \epsilon\right)\left(y-(\va^{0})^{\top} \vx-b^{0}\right)^{\top}\left(\sigma^{0}\right)^{-1}\left(y-(\va^{0})^{\top} \vx-b^{0}\right)-(y-\va^{\top} \vx-b)^{\top}(\sigma)^{-1}(y-\va^{\top} \vx-b)+c_{\alpha} \epsilon \geq 0$.

Therefore, we complete the proof.
\end{proof}

\subsection{Proof of Theorem~\ref{thm_order_consistency}}

Define \(\mathrm{DSC}_N^{(\kappa)}=-\big(\height_N^{(\kappa)}+\epsilon_N\,\bar\ell_N(p_{\widehat G_N^{(\kappa)}})\big)\) with \(1\ll \epsilon_N \ll (N/\log N)^{1/(2\bar r(\widehat G_N))}\) (e.g., \(\epsilon_N=\log N\)).
For \(\kappa>K_0\), \(\height_N^{(\kappa)}\) shrinks at order \((\log N/N)^{1/\bar r(\widehat G_N)}\) while the likelihood term cannot compensate at that scale given the chosen \(\epsilon_N\), so \(\mathrm{DSC}_N^{(\kappa)}\) is suboptimal.
For \(\kappa<K_0\), the (under-fit) likelihood gap dominates and \(\mathrm{DSC}_N^{(\kappa)}\) is worse than at \(\kappa=K_0\).
Hence \(\widehat K_N=\argmin_\kappa \mathrm{DSC}_N^{(\kappa)}\to K_0\) in probability. We will give a more detailed proof below.

\begin{proof}[Proof of Theorem~\ref{thm_order_consistency}]
    Note that entropy $H\left(p_{G_{0}}\right)=-\mathcal{L}\left(p_{G_{0}}\right)$. We have

$$
\height_{N}^{(\kappa)}= \begin{cases}O\left(\left(\dfrac{\log N}{N}\right)^{1 / \bar{r}\left(\widehat{G}_{N}\right)}\right), & \text { if } \kappa>K_{0} \\ \height_{0}^{(\kappa)}+O\left(\left(\dfrac{\log N}{N}\right)^{1 / 2}\right), & \text { if } \kappa \leq K_{0}\end{cases}
$$

and in the proof of \cref{thm:likelihood_path}, we get

$$
\begin{cases}\bar{\ell}_{N}^{(\kappa)} \leq-H\left(p_{G_{0}}\right)+O\left(\left(\dfrac{\log N}{N}\right)^{1 / 2 \bar{r}\left(\widehat{G}_{N}\right)}\right), & \text { if } \kappa>K_{0} \\ \bar{\ell}_{N}^{(\kappa)}=-H\left(p_{G_{0}}\right)+O\left(\left(\dfrac{\log N}{N}\right)^{1 / 2 \bar{r}\left(\widehat{G}_{N}\right)}\right), & \text { if } \kappa=K_{0} \\ \bar{\ell}_{N}^{(\kappa)}=-H\left(p_{G_{0}}\right)-\mathrm{KL}\left(p_{G_{0}} \| p_{G_{0}}^{(\kappa)}\right)+o(1), & \text { if } \kappa<K_{0}\end{cases}
$$

Then we have

$$
\begin{cases}\mathrm{DSC}_N^{(\kappa)} \geq \epsilon_N H\left(p_{G_{0}}\right)+O\left(\epsilon_N\left(\frac{\log N}{N}\right)^{1 / 2 \bar{r}\left(\widehat{G}_{N}\right)}\right), & \text { if } \kappa>K_{0} \\ \mathrm{DSC}_N^{(\kappa)}=\epsilon_N H\left(p_{G_{0}}\right)-\height_{0}^{(\kappa)}+O\left(\epsilon_N\left(\frac{\log N}{N}\right)^{1 / 2 \bar{r}\left(\widehat{G}_{N}\right)}\right), & \text { if } \kappa=K_{0} \\ \mathrm{DSC}_N^{(\kappa)}=\epsilon_N H\left(p_{G_{0}}\right)+\epsilon_N \mathrm{KL}\left(p_{G_{0}} \| p_{G_{0}}^{(\kappa)}\right)-\height_{0}^{(\kappa)}+o\left(\epsilon_N\right), & \text { if } \kappa<K_{0}\end{cases}
$$

Since $\epsilon_N \rightarrow \infty, \epsilon_N(\log N / N)^{1 / 2} \bar{r}\left(\widehat{G}_{N}\right) \rightarrow 0$ and $\mathrm{KL}\left(p_{G_{0}} \| p_{G_{0}}^{(\kappa)}\right)>0$, then as $N \rightarrow \infty, \operatorname{DSC}_{N}^{K_{0}}$ is the smallest number. Hence, $\mathbb{P}_{p_{G_{0}}}\left(\widehat{K}_{N}=K_{0}\right) \geq \mathbb{P}_{p_{G_{0}}}\left(A_{N}\right) \rightarrow 1$ as $N \rightarrow \infty$, or $\widehat{K}_{N} \rightarrow K_{0}$ in probability.
\end{proof}

\end{document}